\documentclass[11pt,twoside]{article}

\usepackage{fancyhdr}
\usepackage{hyperref}
\hypersetup{colorlinks,citecolor=blue,urlcolor=blue,linkcolor=blue,bookmarks=false}
\usepackage{amsfonts,epsfig,graphicx}
\usepackage{afterpage}
\usepackage{amsmath,amssymb,amsthm} 
\usepackage{fullpage}
\usepackage[T1]{fontenc} 
\usepackage{epsf} 
\usepackage{graphics} 
\usepackage{amsfonts,amsmath}
\usepackage[sort,numbers]{natbib} 
\usepackage{psfrag,xspace}
\usepackage{color,etoolbox}
\usepackage{tabularx}
\usepackage{slantsc}

\usepackage{enumitem}
\usepackage{algpseudocode}
\usepackage{algorithm}
\usepackage{array, makecell}
\clearpage{}

\newcommand{\vect}[1]{\ensuremath{\mathbf{#1}}}

\newcommand{\argmin}{\mathop{\rm argmin}}
\newcommand{\argmax}{\mathop{\rm argmax}}

\newcommand{\iprod}[2]{\langle #1, #2 \rangle}

\newcommand{\oracle}[2]{\mathcal{O}_{#1}^{\text{FTPL}}\left(#2\right)}
\newcommand{\maxoracle}[2]{\mathcal{O}_{#1}^{\text{max}}\left(#2\right)}
\newcommand{\minoracle}[2]{\mathcal{O}_{#1}^{\text{min}}\left(#2\right)}

\newcommand{\E}[1]{\mathbb{E}\left[#1\right]}
\newcommand{\Eover}[2]{\mathbb{E}_{#1}\left[#2\right]}

\newcommand{\x}{\vect{x}}

\clearpage{}
\usepackage{csquotes}
\usepackage{mathtools}

\newtheorem{theorem}{Theorem}[section]
\newtheorem{corollary}{Corollary}[section]
\newtheorem{lemma}{Lemma}[section]
\newtheorem{proposition}{Proposition}[section]
\newtheorem{assumption}{Assumption}[section]
\theoremstyle{definition}
\newtheorem{defn}{Definition}[section]

\newtheorem{exmp}{Example}[section]
\theoremstyle{remark}
\newtheorem{remark}{Remark}[section]

\def\RND{\textsl{\textsc{rnd}}}
\def\AVG{\textsl{\textsc{avg}}}

\makeatletter
\newcommand{\printfnsymbol}[1]{\textsuperscript{\@fnsymbol{#1}}}
\makeatother

\begin{document}

\begin{center} 
{\LARGE{\bf{Learning Minimax Estimators via Online Learning}}}

\vspace*{.3in}

{\large{
\begin{tabular}{ccccc}
\begin{tabular}{@{}c@{}}Kartik Gupta$^{*\dagger}$\\ \texttt{kartikg1@andrew.cmu.edu }\end{tabular}~~~\begin{tabular}{@{}c@{}}Arun Sai Suggala$^{*\dagger}$\\ \texttt{asuggala@andrew.cmu.edu}\end{tabular}\vspace{0.1in}\\ 
\begin{tabular}{@{}c@{}}Adarsh Prasad$^\dagger$\\\texttt{adarshp@andrew.cmu.edu}\end{tabular}~~~\begin{tabular}{@{}c@{}}Praneeth Netrapalli$^\ddagger$\\\texttt{praneeth@microsoft.com}\end{tabular}~~~ 
\begin{tabular}{@{}c@{}}Pradeep Ravikumar$^\dagger$\\\texttt{pradeepr@cs.cmu.edu}\end{tabular} \\
\end{tabular}

\vspace*{.1in}

\begin{tabular}{ccc}
$^{\dagger}$Machine Learning Department, Carnegie Mellon University. \\
\end{tabular}

\begin{tabular}{c}
$^{\ddagger}$Microsoft Research, India. \\
\end{tabular}

\vspace*{.2in}

}}
\begin{abstract}
 We consider the problem of designing minimax estimators for estimating the parameters of a probability distribution. Unlike classical approaches such as the MLE and minimum distance estimators, we consider an algorithmic approach for constructing such estimators. 
 We view the problem of designing minimax estimators as finding a mixed strategy Nash equilibrium of a zero-sum game. By leveraging recent results in online learning with non-convex losses, we provide a general algorithm for finding a  mixed-strategy Nash equilibrium of general non-convex non-concave zero-sum games. Our algorithm requires access to two subroutines: (a) one which outputs a Bayes estimator corresponding to a given prior probability distribution, and (b) one which computes the worst-case risk of any given estimator. Given access to these two subroutines, we show that our algorithm outputs both a minimax estimator and a least favorable prior. To demonstrate the power of this approach, we use it to construct provably minimax estimators for classical problems such as estimation in the finite Gaussian sequence model, and linear regression.{\let\thefootnote\relax\footnote{{*  Equal contribution.}}}
\end{abstract}
\end{center}

\section{Introduction}
\label{sec:intro}
Estimating the properties of a probability distribution is a fundamental problem in machine learning and statistics. In this problem, we are given observations generated from an unknown probability distribution $P$ belonging to a class of distributions $\mathcal{P}$. Knowing $\mathcal{P}$, we are required to estimate certain properties of the unknown distribution $P$, based on the observations.
Designing good and ``optimal'' estimators for such problems has been a fundamental subject of research in statistics. Over the years, statisticians have considered various notions of optimality to compare the performance of estimators and to aid their search of good estimators. Some popular notions of optimality include admissibility, minimax optimality, Bayesian optimality, asymptotic efficiency~\citep{ferguson2014mathematical, lehmann2006theory}. Of these, minimax optimality is the most popular notion and has received wide attention in frequentist statistics. This notion of optimality has led to the minimax estimation principle, where the goal is to design estimators with the minimum worst-case risk. Let $R(\hat{\theta}, \theta(P))$ be the risk of an estimator $\hat{\theta}$ for estimating the property $\theta(P)$ of a distribution $P$, where an estimator is a function which maps observations to the set of possible values of the property. 
Then the worst-case risk of $\hat{\theta}$ is defined as $\sup_{P \in \mathcal{P}} R(\hat{\theta}, \theta(P))$. The goal in minimax estimation principle is to design estimators with worst-case risk close to the best worst-case risk, which is defined as $R^* = \inf_{\hat{\theta}}\sup_{P \in \mathcal{P}} R(\hat{\theta}, \theta(P))$, where the infimum is computed over the set of all estimators. Such estimators are often referred to as minimax estimators~\citep{Tsybakov}. 

\paragraph{Classical Estimators.} A rich body of work in statistics has focused on studying the minimax optimality properties of classical estimators such as the maximum likelihood estimator (MLE), Bayes estimators, and minimum contrast estimators (MCEs)~\citep{Ibragimov81book, le2012asymptotic, vaart_1998, birge1983approximation, birge1993rates, yang1999information}. Early works in this line have considered parametric estimation problems and focused on the asymptotic setting, where the number of observations approaches infinity, for a fixed problem dimension. In a series of influential works, H{\'a}jek and Le Cam showed that under certain regularity conditions on the parametric estimation problem, MLE is asymptotically minimax whenever the risk is measured with respect to a convex loss function~\citep{le2012asymptotic, Ibragimov81book}. Later works in this line have considered both parametric and non-parametric estimation problems in the non-asymptotic setting and studied the minimax rates of estimation. In a series of works, Birg{\'e}~\citep{birge1983approximation, birge1993rates} showed that under certain regularity conditions on the model class $\mathcal{P}$ and the estimation problem, MLE and MCEs are approximately minimax w.r.t Hellinger distance.  

While these results paint a compelling picture of classical estimators, we highlight two key problem settings where they tend to be rate inefficient (that is, achieve sub-optimal worst-case risk)~\citep{wellner,birge1993rates}. The first is the so-called high dimensional sampling setting, where the number of observations is comparable to the problem dimension, and under which, classical estimators can be highly sub-optimal. In some recent work, \citet{jiao2015minimax} considered the problem of entropy estimation in discrete distributions and showed that the MLE (plug-in rule) is sub-optimal in the high dimensional regime. Similarly, \citet{cai2011testing} considered the problem of estimation of non-smooth functional $\frac{1}{d}\sum_{i=1}^d|\theta_i|$ from an observation $Y \sim \mathcal{N}(\theta, I_{d})$  and showed that the MLE is sub-optimal. The second key setting where classical estimators tend to be sub-optimal is when the risk $R(\hat{\theta}, \theta(P))$ is measured w.r.t ``non-standard'' losses that have a very different behavior compared to standard losses such as Kullback-Leibler (KL) divergence. For example, consider the MLE, which can be viewed as a KL projection of the empirical distribution of observations onto the class of distributions $\mathcal{P}$. By its design, we expect it to be minimax when the risk is measured w.r.t KL divergence and other related metrics such as Hellinger distance~\citep{birge1993rates}. However, for loss metrics which are not aligned with KL, one can design estimators with better performance than MLE, by taking the loss into consideration. This phenomenon is better illustrated  with the following toy example. Suppose $\mathcal{P}$ is the set of multivariate normal distributions in $\mathbb{R}^d$ with identity covariance, and suppose our goal is to estimate the mean of a distribution $P\in \mathcal{P}$, given $n$ observations drawn from it. If the risk of estimating $\theta$ as $\Tilde{\theta}$ is measured w.r.t the following loss $\|\Tilde{\theta}-\theta-c\|_2^2$, for some constant $c$, then it is easy to see that MLE has a worst-case risk greater than $\|c\|_2^2$. Whereas, the minimax risk $R^*$ is equal to $d/n$, which is achieved by an estimator obtained by shifting the MLE by $c$. While the above loss is unnatural, such a  phenomenon can be observed with natural losses such as $\ell_q$ norms for $q\in (0,1)$ and asymmetric losses.

\vspace{-0.15in}
\paragraph{Bespoke Minimax Estimators.}
For problems where classical estimators are not optimal, designing a minimax estimator can be challenging. Numerous works in the  literature have attempted to design minimax estimators in such cases. However the focus of these works is on specific problems~\citep{cai2011testing,valiant, jiao2015minimax, butucea2018variable}, and there is no single estimator which is known to be optimal for a wide range of estimation problems. For example, \citet{jiao2015minimax, wu2016minimax} considered the problem of entropy estimation for discrete distributions and provided a minimax estimator in the high-dimensional setting. \citet{cai2011testing} considered the problem of estimating a non-smooth functional in high dimensions and provided a minimax estimator.  
While these results are impressive, the techniques used in these works are tailored towards specific problems and do not extend to other problems. So, a natural question that arises in this context is, how should one go about constructing minimax estimators for problems where none of the classical estimators are optimal? Unfortunately, our current understanding of minimax estimators does not provide any concrete guidelines on designing such estimators.  

\vspace{-0.15in}
\paragraph{Minimax Estimation via Solving Statistical Games.} 
In this work, we attempt to tackle the problem of designing minimax estimators from a game-theoretic perspective. Instead of the usual two-step approach of first designing an estimator and then certifying its minimax optimality, we take a more direct approach and attempt to directly solve the following min-max statistical game: $\inf_{\hat{\theta}}\sup_{P \in \mathcal{P}} R(\hat{\theta}, \theta(P))$. Since the resulting estimators are solutions to the min-max game, they are optimal by construction.
Such a direct approach for construction of minimax estimators has certain advantages over the classical estimators. First, the technique itself is very general and can \emph{theoretically} be used to construct minimax estimators for any estimation problem. Second, a direct approach often results in \emph{exact} minimax estimators with $R^*+o(1)$ worst-case risk. In contrast, classical estimators typically achieve $O(1)R^*$ worst-case risk, which is constant factors worse than the direct approach. Finally, a direct approach can make effective use of any available side information about the problem, to construct estimators with better worst-case risk than classical estimators. For example, consider the problem of mean estimation given samples drawn from an unknown Gaussian distribution. If it is known a priori that the true mean lies in a bounded set, then a direct approach for solving the min-max statistical game results in estimators with better performance than classical estimators.
Several past works have attempted to directly solve the min-max game associated with the estimation problem~\citep[see][and references therein]{berger1985statistical}. We discuss these further in Section~\ref{sec:background} after providing some background, but in gist, existing approaches either focus on specific problems or are applicable only to simple estimation problems.

\vspace{-0.15in}
\paragraph{This Work.} In this work, we rely on recent advances in online learning and game theory to directly solve the min-max statistical game. Recently, online learning techniques have been widely used for solving min-max games. For example, \citet{freund1996game} relied on these techniques to find equilibria in min-max games that arise in the context of boosting. Similar techniques have been explored for robust optimization by~\citet{chen2017robust,feige2015learning}.
In this work, we take a similar approach and provide an algorithm for solving statistical games. 
A critical distinction of statistical games, in contrast to the typical min-max games studied in the learning and games literature, is that the domain of all possible measurable estimators is extremely large, the set of possible parameters need not be convex, and the loss function need not be convex-concave. We show that it is nonetheless possible to finesse these technical caveats and solve the statistical game, provided we are given access to two subroutines: a Bayes estimator subroutine which outputs a Bayes estimator corresponding to any given prior, and a subroutine which computes the worst-case risk of any given estimator.
Given access to these two subroutines, we show that our algorithm outputs both a minimax estimator and a least favorable prior. The minimax estimator output by our algorithm is a randomized estimator which is an ensemble of multiple Bayes estimators. When the loss function is convex - which is the case for a number of commonly used loss functions - the randomized estimator can be transformed into a deterministic minimax estimator. For problems where the two subroutines are efficiently implementable, our algorithm provides an efficient technique to construct minimax estimators. While implementing the subroutines can be computationally hard in general, we show that the computational complexity can be significantly reduced for a wide range of problems satisfying certain invariance properties. 

 To demonstrate the power of this technique, we use it to construct provably minimax estimators for the classical problems of finite dimensional Gaussian sequence model and linear regression.
 In the problem of Gaussian sequence model, we are given a single sample drawn from a normal distribution with mean $\theta$ and identity covariance, where $\theta \in \mathbb{R}^d, \|\theta\|_2 \leq B$. Our goal is to estimate $\theta$ well under squared-error loss. This problem has received much attention in statistics because of its simplicity and connections to non-parametric regression~\citep{johnstone2002function}.  Surprisingly, however, the exact minimax estimator is unknown for the case when $B \geq 1.16\sqrt{d}$ ~\citep{bickel1981minimax, berry1990minimax, marchand2002minimax}.  In this work, we show that our technique can be used to construct provably minimax estimators for this problem, for general $B$. To further demonstrate that our technique is widely applicable, we present empirical evidence showing that our algorithm can be used to construct estimators for covariance and entropy estimation which match the performance of existing minimax estimators.

\vspace{-0.15in}
\paragraph{Outline.}
We conclude this section with a brief outline of the rest of the paper. In Section~\ref{sec:background}, we provide necessary background on online learning and minimax estimation. In Section~\ref{sec:games}, we introduce our algorithm for solving statistical games. In Sections~\ref{sec:invariance},~\ref{sec:mean_estimation},~\ref{sec:regression} we utilize our algorithm to construct provably minimax estimators for finite dimensional Gaussian sequence model and linear regression. In Section~\ref{sec:exps} we study the empirical performance of our algorithm on a variety of statistical estimation problems. We defer technical details to the Appendix. Finally, we conclude in Section~\ref{sec:conclusion} with a discussion of future directions and some open problems.

\vspace{-0.1in}
\section{Background and Problem Setup}
\label{sec:background}
In this section, we formally introduce the problem of minimax statistical estimation and review the necessary background on online learning.
\vspace{-0.1in}
\subsection{Minimax Estimation and Statistical Games}
 Let $\mathcal{P} = \{P_{\theta}: \theta \in \Theta \subseteq \mathbb{R}^d \}$ be a parametric family of distributions. In this work, we assume $\Theta$ is a compact set. Let \mbox{$\mathbb{X}^n = \{X_1,\dots X_n\} \in \mathcal{X}^n$} be $n$ independent samples drawn from some unknown distribution $P_{\theta} \in \mathcal{P}$. Given $\mathbb{X}^n$, our goal is to estimate the unknown parameter $\theta$. 
 A deterministic estimator $\hat{\theta}$ of $\theta$ is any measurable function from $\mathcal{X}^n$ to $\Theta$. We denote the set of deterministic estimators by $\mathcal{D}$. A randomized estimator is given by a probability measure on the set of deterministic estimators. Given $\mathbb{X}^n$, the unknown parameter $\theta$ is estimated by first sampling a deterministic estimator according to this probability measure and using the sampled estimator to predict $\theta$. 
Since any randomized estimator can be identified by a probability measure on $\mathcal{D}$, we denote the set of randomized estimators by $\mathcal{M}_{\mathcal{D}},$ the set of all probability measures on $\mathcal{D}$. 
 Let $M:\Theta\times \Theta \to \mathbb{R}$ be a measurable loss function such that $M(\theta',\theta)$ measures the cost of an estimate $\theta'$ when the true parameter is $\theta$.  Define the risk of an estimator $\hat{\theta}$ for estimating $\theta$ as $
R(\hat{\theta}, \theta) \stackrel{\text{def}}{=} \mathbb{E}\left[M(\hat{\theta}(\mathbb{X}^n), \theta)\right],$ where the expectation is taken with respect to randomness from $\mathbb{X}^n$ and the estimator $\hat{\theta}$.
The worst-case risk of an estimator $\hat{\theta}$ is defined as 
$\sup_{\theta \in \Theta}  R(\hat{\theta},\theta)$ and
the minimax risk is defined as the best worst-case risk that can be achieved by any estimator
\begin{equation}
\label{eqn:minimax_objective}
    R^* \stackrel{\text{def}}{=} \inf_{\hat{\theta} \in \mathcal{M}_\mathcal{D}} \sup_{\theta \in \Theta}  R(\hat{\theta}, \theta).
\end{equation}
Any estimator whose worst case risk is equal to the minimax risk is called a minimax estimator. We refer to the above min-max problem as a \emph{statistical game}. Often, we are also interested in deterministic minimax estimators, which are defined as estimators with worst case risk equal to
\begin{equation}
\label{eqn:minimax_objective_det}
    \inf_{\hat{\theta} \in \mathcal{D}} \sup_{\theta \in \Theta}  R(\hat{\theta}, \theta).
\end{equation}

From the perspective of game theory, the optimality notion in Equation~\eqref{eqn:minimax_objective} is referred to as the \emph{minmax} value of the game. This is to be contrasted with the \emph{maxmin} value of the game $ \sup_{\theta \in \Theta}\inf_{\hat{\theta} \in \mathcal{M}_\mathcal{D}}  R(\hat{\theta}, \theta)$. 
In general, these two quantities are \textbf{not} equal, but the following relationship always holds:
\begin{align}
\sup_{\theta \in \Theta}\inf_{\hat{\theta} \in \mathcal{M}_\mathcal{D}}  R(\hat{\theta}, \theta)  \leq \inf_{\hat{\theta} \in \mathcal{M}_\mathcal{D}} \sup_{\theta \in \Theta}  R(\hat{\theta}, \theta).
\label{eqn:min-max-ineq}
\end{align}
In statistical games, for typical choices of loss functions, \mbox{$\sup_{\theta \in \Theta} \inf_{\hat{\theta}\in \mathcal{M}_\mathcal{D}} R(\hat{\theta},\theta) = 0$,} whereas $ \inf_{\hat{\theta}\in \mathcal{M}_\mathcal{D}} \sup_{\theta \in \Theta} R(\hat{\theta},\theta) > 0$; that is, the minmax value is strictly greater than maxmin value of the game.  So we cannot in general reduce computing the minmax value to computing the maxmin value. 

\vspace{-0.15in}
\paragraph{Linearized Statistical Games.}
Without any additional structure such as convexity, computing the values of min-max games is difficult in general.
So it is common in game theory to consider a \emph{linearized game} in the space of probability measures, which is in general better-behaved. To set up some notation,
for any probability distribution $P$, define $R(\hat{\theta}, P)$  as 
\mbox{$\Eover{\theta \sim P}{R(\hat{\theta}, \theta)}.$} In the context of statistical games, a linearized game has the following form:
\begin{equation}
\label{eqn:minimax_objective_mixed}
\inf_{\hat{\theta} \in \mathcal{M}_\mathcal{D}} \sup_{P \in \mathcal{M}_{\Theta}}  R(\hat{\theta}, P),
\end{equation}
where $\mathcal{M}_{\Theta}$ is the set of all probability measures on $\Theta$.
The minmax and maxmin values of the linearized game and the original game in Equation~\eqref{eqn:minimax_objective} are related as follows
\begin{equation*}
    \sup_{\theta \in \Theta}\inf_{\hat{\theta} \in \mathcal{M}_\mathcal{D}}  R(\hat{\theta}, \theta)  \leq  \sup_{P \in \mathcal{M}_{\Theta}} \inf_{\hat{\theta} \in \mathcal{M}_\mathcal{D}} R(\hat{\theta}, P) \leq \inf_{\hat{\theta} \in \mathcal{M}_\mathcal{D}} \sup_{P \in \mathcal{M}_{\Theta}}  R(\hat{\theta}, P) \stackrel{(a)}{=} \inf_{\hat{\theta} \in \mathcal{M}_\mathcal{D}} \sup_{\theta \in \Theta}  R(\hat{\theta}, \theta),
\end{equation*}
where $(a)$  holds because for any estimator $\hat{\theta},$ $\sup_{P \in \mathcal{M}_{\Theta}}  R(\hat{\theta}, P)$ is equal to $\sup_{\theta \in \Theta}  R(\hat{\theta}, \theta).$ Thus, the minmax values of the original and linearized statistical games are equal. Any estimator whose worst-case risk is equal to the minmax value of the linearized game is a minimax estimator. The maxmin values of the original and linearized statistical games are however in general different. In particular, as discussed above, the maxmin value of the original statistical game is usually equal to zero. The maxmin value of the \emph{linearized game} however has a deep connection to Bayesian estimation.

Note that $R(\hat{\theta}, P)$ is simply the integrated risk of the estimator $\hat{\theta}$ under prior $P \in \mathcal{M}_\Theta$. Any estimator which minimizes $R(\hat{\theta}, P)$ is called the Bayes estimator for $P$, and the corresponding minimum value is called Bayes risk. Though the set of all possible measurable estimators is in general vast, in what might be surprising from an optimization or game-theoretic viewpoint, the Bayes estimator can be characterized simply as follows. Letting $P(\cdot|\mathbb{X}^n)$ be the posterior distribution of $\theta$ given the data $\mathbb{X}^n$, a Bayes estimator of $P$ can be found by minimizing the posterior risk
\begin{equation}
\label{eqn:bayes_estimator}
\hat{\theta}_{P}(\mathbb{X}^n) \in \argmin_{\tilde{\theta} \in \Theta} \Eover{\theta \sim P(\cdot|\mathbb{X}^n)}{M(\tilde{\theta}, \theta)}.
\end{equation}
Certain mild technical conditions need to hold for $\hat{\theta}_{P}$ to be measurable and for it to be a Bayes estimator~\citep{berger1985statistical}. We detail these conditions in Appendix~\ref{sec:measurable_bayes_estimator}, which incidentally are all satisfied for the problems considered in this work.
A least favourable prior is defined as any prior which maximizes the Bayes risk; that is, $\tilde{P}$ is LFP if  $\inf_{\hat{\theta} \in \mathcal{M}_\mathcal{D}} R(\hat{\theta},\tilde{P}) = \sup_{P \in \mathcal{M}_\Theta} \inf_{\hat{\theta} \in \mathcal{M}_\mathcal{D}} R(\hat{\theta}, P).$ Thus, LFPs solve for the maxmin value of the linearized statistical game. Any prior whose Bayes risk is equal to the maxmin value of the linearized game is an LFP.

\vspace{-0.15in}
\paragraph{Nash Equilibrium.}

Directly solving for the minmax or maxmin values of the (linearized) min-max games is in general computationally hard, in large part because: (a) these values need not be equal, which limits the set of possible optimization algorithms, and (b) the optimal solutions need not be stable, which makes it difficult for simple optimization problems. It is thus preferable that the two values are equal\footnote{John Von Neumann, a founder of game theory, has said he could not foresee there even being a theory of games without a theorem that equates these two values}, and the solutions be stable, which is formalized by the game-theoretic notion of a \emph{Nash equilibrium} (NE).

For the original statistical game in Equation~\eqref{eqn:minimax_objective}, a pair $(\hat{\theta}^*,\theta^*) \in \mathcal{M}_{\mathcal{D}}\times \Theta$ is called a pure strategy NE, if the following holds
\begin{align*}
    \sup_{\theta \in \Theta} R(\hat{\theta}^*, \theta) \leq R(\hat{\theta}^*,\theta^*) \leq \inf_{\hat{\theta} \in \mathcal{M}_\mathcal{D}} R(\hat{\theta}, \theta^*) = \inf_{\hat{\theta} \in \mathcal{D}} R(\hat{\theta}, \theta^*),
\end{align*}
where the equality follows since the optimum of a linear program over a convex hull can always be attained at an extreme point. Intuitively, this says that there is no incentive for any player to change their strategy while the other player keeps hers unchanged.  
Note that whenever a pure strategy NE exists, the minmax and maxmin values of the game are equal to each other:
\begin{align*}
     \inf_{\hat{\theta} \in \mathcal{M}_\mathcal{D}} \sup_{\theta \in \Theta}  R(\hat{\theta}, \theta) \leq \sup_{\theta \in \Theta} R(\hat{\theta}^*, \theta) \leq R(\hat{\theta}^*,\theta^*) \leq \inf_{\hat{\theta} \in \mathcal{M}_\mathcal{D}} R(\hat{\theta}, \theta^*) \leq \sup_{\theta \in \Theta}\inf_{\hat{\theta} \in \mathcal{M}_\mathcal{D}}R(\hat{\theta}, \theta).
\end{align*}
Since the RHS is always upper bounded by the LHS from \eqref{eqn:min-max-ineq}, the inequalities above are all equalities.

As we discussed above, the maxmin and minmax values of the statistical game in Equation~\eqref{eqn:minimax_objective} are in general not equal to each other, so that a pure strategy NE will typically not exist for the statistical game~\eqref{eqn:minimax_objective}.
Instead what often exists is a mixed strategy NE, which is precisely a pure strategy NE of the linearized game. That is, $(\hat{\theta}^*, P^*)\in \mathcal{M}_{\mathcal{D}}\times \mathcal{M}_\Theta $ is called a mixed strategy NE of statistical game~\eqref{eqn:minimax_objective}, if 
\begin{align*}
    \sup_{\theta \in \Theta} R(\hat{\theta}^*, \theta) = \sup_{P \in \mathcal{M}_\Theta} R(\hat{\theta}^*, \theta) \leq R(\hat{\theta}^*, P^*) &\leq  \inf_{\hat{\theta} \in \mathcal{M}_\mathcal{D}} R(\hat{\theta}, P^*) = \inf_{\hat{\theta} \in \mathcal{D}} R(\hat{\theta}, P^*).
\end{align*} 
 As with the original game, if $(\hat{\theta}^*, P^*)$ is a pure strategy NE of the linearized game of~\eqref{eqn:minimax_objective}, aka, a mixed strategy NE of~\eqref{eqn:minimax_objective}, then the minmax and maxmin values of the linearized game are equal to each other, and, moreover $\hat{\theta}^*$ is a minimax estimator and  $P^*$ is an LFP. Conversely, if $\hat{\theta}^*$ is a minimax estimator, and $P^*$ is an LFP, and the minmax and maxmin values of~\eqref{eqn:minimax_objective_mixed} are equal to each other, then $(\hat{\theta}^*, P^*)$ is a mixed strategy NE of~\eqref{eqn:minimax_objective}.  These just follow from similar sandwich arguments as with the original game, which we add for completeness in Appendix~\ref{sec:minimax_lfp_ne}.

In gist, it might be computationally easier to recover the mixed strategy NE of the statistical game, assuming they exist, and doing so, would recover minimax estimators and LFPs. In this work, we are thus interested in imposing mild conditions on the statistical game so that a mixed strategy NE exists, and under this setting, develop tractable algorithms to estimate the mixed strategy NE.

\vspace{-0.15in}
\paragraph{Existence of NE.} We now briefly discuss sufficient conditions for the existence of NE.  
As discussed earlier, a pure strategy NE does not exist for statistical games in general. So, here we focus on existence of mixed strategy NE.  
In a seminal work,~\citet{wald1949statistical} studied the conditions for existence of a mixed strategy NE,
and showed that a broad class of statistical games have mixed strategy NE. Suppose every distribution in the model class $\mathcal{P}$ is absolutely continuous,  $\Theta$ is compact, and the loss $M$ is a bounded, non-negative function. Then minmax and maxmin values of the linearized game are equal. Moreover, a minimax estimator with worst-case risk equal to $R^*$ exists.
Under the additional condition of compactness of $\mathcal{P}$, \citep{wald1949statistical} showed that an LFP exists as well. Thus, based on our previous discussion, this implies the game has a mixed strategy NE. In this work, we consider a  different and simpler set of conditions on the statistical game.  We assume that $\Theta$ is compact and the risk $R(\hat{\theta},\theta)$ is Lipschitz in its second argument. Under these assumptions, we show that the minmax and maxmin values of the linearized game in Equation~\eqref{eqn:minimax_objective_mixed} are equal to each other. Such results are known as minimax theorems and have been studied in the past~\citep{von2007theory, yanovskaya1974infinite, wald1949statistical}. However, unlike past works that rely on fixed point theorems, we rely on a constructive learning-style proof to prove the minimax theorem, where we present an algorithm which outputs an approximate NE of the statistical game. Under the additional condition that the risk $R(\hat{\theta}, \theta)$ is bounded, we show that the statistical game has a minimax estimator and an LFP. 

\vspace{-0.15in}
\paragraph{Computation of NE.} Next, we discuss previous numerical optimization techniques for computing a mixed strategy NE of the statistical game. Note that this is a difficult computational problem: minimizing over the domain of all possible estimators, and maximizing over the set of all probability measures on $\Theta$. Nonetheless, several works in statistics have attempted to tackle this problem~\citep{berger1985statistical}. One class of techniques involves reducing the set of estimators $\mathcal{D}$ via admissibility considerations to a small enough set. Given this restricted set of estimators, they can then directly calculate a minimax test for some testing problems; see for instance ~\citet{hald1971size}. A drawback of these approaches is that they are restricted to simple estimation problems for which the set of admissible estimators are easy to construct. Another class of techniques for constructing minimax estimators relies on the properties of LFPs~\citep{clarke1994jeffreys, johnstone2002function}. When the parameter set $\Theta$ is a compact subset of $\mathbb{R}$, and when certain regularity conditions hold, it is well known that LFPs are supported on a finite set of points~\citep{ghosh1964uniform, berger1985statistical}. Based on this result,~\citet{nelson1966minimax, kempthorne1987numerical} propose numerical approaches to determine the support points of LFPs and the probability mass that needs to be placed on these points. However, these approaches are restricted to 1-dimensional estimation problems and are not broadly applicable. In a recent work, \citet{Luedtkeeaaw2140} propose heuristic approaches for solving statistical games using deep learning techniques. In particular, they use neural networks to parameterize the statistical game and solve the resulting game using local search techniques such as alternating gradient descent. However, these approaches are not guaranteed to find minimax estimators and LFPs and can lead to undesirable equilibrium points. They moreover parameterize estimators via neural networks whose inputs are a simple concatenation of all the samples, which is not feasible for large $n$.

In our work, we develop numerical optimization techniques that rely on online learning algorithms (see Section~\ref{sec:bg_online_learning}). Though the domains as well as the setting of the statistical game are far more challenging than typically considered in learning and games literature, we reduce the problem of designing minimax estimators to a purely  computational problem of efficient implementation of certain optimization subroutines. For the wide range of problems where these subroutines can be efficiently implemented, our algorithm provides an efficient and scalable technique for constructing minimax estimators.
 \vspace{-0.1in}
\subsection{Online Learning}
\label{sec:bg_online_learning}
The online learning framework can be seen as a repeated game between a learner/decision-maker and an adversary. 
In this framework, in each round $t$, the learner makes a prediction \mbox{$\x_t \in \mathcal{X}$}, where $\mathcal{X} \subseteq\mathbb{R}^d$, and the adversary  chooses a loss function \mbox{$f_t:\mathcal{X} \rightarrow \mathbb{R}$} and observe each others actions. The goal of the learner is to choose a sequence of actions $\{\x_t\}_{t=1}^T$ so that the cumulative loss $\sum_{t=1}^T f_t(\x_t)$ is minimized. The benchmark with which the cumulative loss will be compared is called the best fixed policy in hindsight, which is given by $\inf_{\x \in \mathcal{X}}\sum_{t=1}^Tf_t(\x)$. This results in the following notion of regret, which the learner aims to minimize $$\sum_{t = 1}^T f_t(\x_t) - \inf_{\x \in \mathcal{X}}\sum_{t=1}^Tf_t(\x).$$ 
 \vspace{-0.15in}
\paragraph{When the domain $\mathcal{X}$ is compact, and convex, and the loss functions $f_t$ are convex:} Under this simple setting, a number of efficient algorithms for regret minimization have been studied. Some of these include Follow the Regularized Leader (FTRL)~\citep{hazan2016introduction, mcmahan2017survey}, Follow the Perturbed Leader (FTPL)~\citep{kalai2005efficient}. In FTRL, one predicts $\x_t$ as $\argmin_{\x \in \mathcal{X}} \sum_{i=1}^{t-1}f_i(\x) + r(\x)$, where $r$ is a strongly convex regularizer. In FTPL, one predicts $\x_t$ as $\Eover{\sigma}{\argmin_{\x \in \mathcal{X}} \sum_{i=1}^{t-1}f_i(\x) - \iprod{\sigma}{\x}}$, where $\sigma$ is a random perturbation drawn from some appropriate probability distribution such as exponential distribution. These algorithms are known to achieve the optimal $O(\sqrt{T})$ regret in the convex setting~\citep{mcmahan2017survey,ftpl_nonconvex}.

\vspace{-0.15in}
\paragraph{When $\mathcal{X}$ is compact, but either the domain or the loss functions $f_t$ are non-convex:} Under this setting, no deterministic algorithm can achieve sub-linear regret (i.e., regret which grows slower than $T$)~\citep{cesa2006prediction,ftpl_nonconvex}. In such cases one has to rely on randomized algorithms to achieve sub-linear regret. In randomized algorithms, in each round $t$, the learner samples the prediction $\x_t$ from a distribution $P_t \in \mathcal{M}_{\mathcal{X}}$, where $\mathcal{M}_{\mathcal{X}}$ is the set of all probability distributions supported on $\mathcal{X}$. The goal of the learner is to choose a sequence of distributions $\{P_t\}_{t=1}^T$ to minimize the expected regret
$
\sum_{t = 1}^T \Eover{\x\sim P_t}{f_t(\x)} - \inf_{\x \in \mathcal{X}}\sum_{t=1}^Tf_t(\x).
$
An alternative perspective of such randomized algorithms is as deterministic algorithms solving a \emph{linearized problem} in the space of probability distributions, with loss functions $\Tilde{f}_t(P) = \Eover{\x\sim P}{f_t(\x)}$, and rely on algorithms for online convex learning. For example, by relying of FTRL, one predicts $P_t$ as
$
\argmin_{P \in \mathcal{M}_{\mathcal{X}}} \sum_{i=1}^{t-1} \Tilde{f}_i(P) + r(P),
$
for some strongly convex regularizer $r(P)$. 
When $r(P)$ is the negative entropy of $P$, \citet{krichene2015hedge} show that the resulting algorithm achieves $O(\sqrt{dT\log{T}})$ expected regret. 
Another technique to solve the linearized problem is via the FTPL algorithm~\citep{gonen2018learning, ftpl_nonconvex}. In this algorithm,  $P_t$ is given by the distribution of the random variable $\x_t(\sigma)$, which is a minimizer of $\sum_{i=1}^{t-1}f_i(\x) - \iprod{\sigma}{\x}$. Here, $\sigma$ is a random perturbation drawn from some appropriate probability distribution. In  recent work, \citet{ftpl_nonconvex} show that this algorithm achieves $O(\sqrt{d^3T})$ expected regret. 

\vspace{-0.15in}
\paragraph{Without any assumptions on $\mathcal{X}$ or the loss functions $f_t$.}
A key caveat with statistical games is that the domain of all possible measurable estimators is not bounded and is an infinite-dimensional space. Thus, results as discussed above from the learning and games literature are not applicable to such a setting. In particular, regret bounds of FTRL and FTPL scale with the dimensionality of the domain, which is infinite in this case. But there is a very simple strategy that is applicable without making any assumptions on the domain whatsoever, but under the provision that $f_t$ \emph{was known} to the learner ahead of round $t$. Then, an optimal strategy for the learner is to predict $\x_t$ as simply a minimizer of $f_t(\x).$ It is easy to see that this algorithm, known as Best Response (BR), has $0$ regret. While this is an impractical algorithm in the framework of online learning, it can be used to solve min-max games, as we will see in Section~\ref{sec:games}.

\vspace{-0.15in}
\paragraph{FTPL.} We will be making use of the FTPL algorithm in the sequel, so we now describe this in a bit more detail. In this algorithm, the learner predicts $\x_t$ as a minimizer of \mbox{$\sum_{i = 1}^{t-1}f_i(\x) - \iprod{\sigma}{\x}$,} where $\sigma \in \mathbb{R}^d$ is a random perturbation such that $\{\sigma_{j}\}_{j = 1}^d \stackrel{i.i.d}{\sim} \text{Exp}(\eta)$ and $\text{Exp}(\eta)$ is the exponential distribution with parameter $\eta$\footnote{Recall, $X$ is an exponential random variable with parameter $\eta$ if  $P(X\geq s) =\exp(-\eta s)$}. When the domain $\mathcal{X}$ is bounded and loss functions $\{f_t\}_{t=1}^T$ are Lipschitz (not necessarily convex), FTPL achieves $O(\sqrt{d^3T})$ expected regret, for appropriate choice of $\eta$~\citep{ftpl_nonconvex}. A similar regret bound holds even when $\x_t$ is an approximate minimizer of \mbox{$\sum_{i = 1}^{t-1}f_i(\x) - \iprod{\sigma}{\x}$.} Suppose for any $t\in\mathbb{N}$, $\x_t$ is such that 
\[
\sum_{i = 1}^{t-1}f_i(\x_t) - \iprod{\sigma}{\x_t} \leq \inf_{\x \in \mathcal{X}} \sum_{i = 1}^{t-1}f_i(\x) - \iprod{\sigma}{\x} + \left(\alpha + \beta \|\sigma\|_1\right),
\]
where $\alpha, \beta$ are positive constants. Then FTPL achieves $O(T^{1/2} + \alpha T + \beta T^{3/2})$ expected regret for appropriate choice of $\eta$ (see Appendix~\ref{sec:ftpl} for more details).

\vspace{-0.1in}
\section{Minimax Estimation via Online Learning}
\label{sec:games}
In this section, we present our algorithm for computing a mixed strategy NE of the statistical game in Equation~\eqref{eqn:minimax_objective} (equivalently a pure strategy NE of the linearized game in Equation~\eqref{eqn:minimax_objective_mixed}). 
A popular and widely used approach for solving min-max games is to rely on online learning algorithms~\citep{hazan2016introduction, cesa2006prediction}. In this approach, the minimization player and the maximization player play a repeated game against each other. Both the players rely on online learning algorithms to choose their actions in each round of the game, with the objective of minimizing their respective regret. The following proposition shows that this repeated game play converges to a NE.
\begin{proposition}
\label{prop:no_regret_to_games}
Consider a repeated game between the minimization and maximization players in Equation~\eqref{eqn:minimax_objective_mixed}. 
Let $(\hat{\theta}_t, P_t)$ be the actions chosen by the players in iteration $t$. Suppose the actions are such that the regret of each player satisfies
\begin{align*}
    \sum_{t=1}^TR(\hat{\theta}_t, P_t)-\inf_{\hat{\theta} \in \mathcal{D}} \sum_{t=1}^TR(\hat{\theta}, P_t) &\leq \epsilon_1(T),\\
    \sup_{\theta\in\Theta}\sum_{t=1}^TR(\hat{\theta}_t, \theta) -\sum_{t=1}^TR(\hat{\theta}_t, P_t)&\leq \epsilon_2(T).
\end{align*}
Let $\hat{\theta}_{\RND}$ denote the randomized estimator obtained by uniformly sampling an estimator from the iterates  $\{\hat{\theta}_t\}_{t=1}^T$. Define the mixture distribution $P_{\AVG}$ as $\frac{1}{T}\sum_{i=1}^T P_i$. Then $(\hat{\theta}_{\RND}, P_{\AVG})$ is an approximate mixed strategy NE of Equation~\eqref{eqn:minimax_objective}
\begin{align*}
R(\hat{\theta}_{\RND}, P_{\AVG}) \leq \inf_{\hat{\theta} \in \mathcal{D}} R(\hat{\theta}, P_{\AVG}) + \frac{\epsilon_1(T) + \epsilon_2(T)}{T},\\
R(\hat{\theta}_{\RND}, P_{\AVG}) \geq \sup_{\theta \in \Theta} R(\hat{\theta}_{\RND}, \theta) - \frac{\epsilon_1(T) + \epsilon_2(T)}{T}.
\end{align*}
\end{proposition}
Note that the above proposition doesn't specify an algorithm to generate the iterates $(\hat{\theta}_t, P_t)$.  All it shows is that as long as both the players rely on algorithms which guarantee sub-linear regret, the iterates converge to a NE. As discussed in Section~\ref{sec:background}, there exist several algorithms such as FTRL, FTPL, Best Response (BR), which guarantee sub-linear regret. It is important to choose these algorithms appropriately as our choices impact the rate of convergence to a NE and also the computational complexity of the resulting algorithm. First, consider the minimization player, whose domain $\mathcal{M}_{\mathcal{D}}$ is the set of all probability measures over  $\mathcal{D}$. Note that $\mathcal{D}$, the set of all deterministic estimators, is an infinite dimensional space. So, algorithms such as FTRL, FTPL, whose regret bounds depend on the dimension of the domain, can not guarantee sub-linear regret. So the minimization player is forced to rely on BR, which has $0$ regret. Recall, in order to use BR, the minimization player requires the knowledge of the future action of the opponent. This can be made possible in the context of min-max games by letting the minimization player choose her action after the maximization player reveals her action. Next, consider the maximization player. Since the minimization player is relying on BR, the maximization player has to rely on either FTRL or FTPL to choose her action\footnote{If both the players use BR, then both will wait for the other player to pick an action first. As a result, the algorithm will never proceed.}.  In this work we choose the FTPL algorithm studied by~\cite{ftpl_nonconvex}. Our choice is mainly driven by the computational aspects of the algorithm. Each iteration of the FTRL algorithm of~\citet{krichene2015hedge} involves sampling from a general probability distribution. Whereas, each iteration of the FTPL algorithm requires minimization of a non-convex objective. While both sampling and optimization are computationally hard in general, the folklore is that optimization is relatively easier than sampling in many practical applications.

We now describe our algorithm for computing a pure strategy NE of Equation~\eqref{eqn:minimax_objective_mixed}. In iteration $t$, the maximization player chooses distribution $P_t$ using FTPL. $P_t$ is given by the distribution of the random variable $\theta_t(\sigma)$, which is generated by first sampling a random vector $\sigma \in \mathbb{R}^d$ from exponential distribution and then computing an optimizer of 
\begin{equation}
\label{eqn:max_oracle}
\sup_{\theta \in \Theta} \sum_{i = 1}^{t-1} R(\hat{\theta}_i, \theta) + \iprod{\sigma}{\theta}.
\end{equation}
The minimization player chooses $\hat{\theta}_t$ using BR, which involves computing a minimizer of the integrated risk under prior $P_t$
\begin{equation}
\label{eqn:min_oracle}
\inf_{\hat{\theta} \in \mathcal{D}} R(\hat{\theta}, P_t).    
\end{equation}
Very often, computing exact optimizers of the above problems is infeasible. Instead, one can only compute approximate optimizers. To capture the error from this approximation, we introduce the notion of approximate optimization oracles/subroutines.
\begin{defn}[Maximization Oracle]
A function $\maxoracle{\alpha,\beta}{\cdot}$ is called $(\alpha,\beta)$-approximate maximization oracle, if for any set of estimators $\{\hat{\theta}_i\}_{i=1}^T$ and perturbation $\sigma$, it returns $\theta' \in \Theta$ which satisfies the following inequality
\[
\sum_{i = 1}^{T} R\left(\theta', \theta\right) + \iprod{\sigma}{\theta'} \geq \sup_{\theta \in \Theta} \sum_{i = 1}^{T} R(\hat{\theta}_i, \theta) + \iprod{\sigma}{\theta} - \left(\alpha + \beta \|\sigma\|_1\right).
\]
We denote the output $\theta'$ by $\maxoracle{\alpha,\beta}{\{\hat{\theta}_i\}_{i=1}^T, \sigma}$.
\end{defn}
\begin{defn}[Minimization Oracle]
A function $\minoracle{\alpha}{\cdot}$ is called $\alpha$-approximate minimization oracle, if for any probability measure $P$, it returns an approximate Bayes estimator $\hat{\theta}'$ which satisfies the following inequality
\[
R(\hat{\theta}', P) \leq \inf_{\hat{\theta} \in \mathcal{D}} R(\hat{\theta}, P) + \alpha.
\]
We denote the output $\hat{\theta}'$ by $\minoracle{\alpha}{P}$.
\end{defn}
Given access to subroutines $\maxoracle{\alpha,\beta}{\cdot}, \minoracle{\alpha'}{\cdot}$ for approximately solving the optimization problems in Equations~\eqref{eqn:max_oracle},~\eqref{eqn:min_oracle}, the algorithm alternates between the maximization and minimization players who choose $P_t$ and $\hat{\theta}_t$ in each iteration. 
We summarize the overall algorithm in Algorithm~\ref{alg:ftpl_stat_games}. The following theorem shows that Algorithm~\ref{alg:ftpl_stat_games} converges to an approximate NE of the statistical game.

\begin{algorithm}[t]
\caption{FTPL for statistical games}
\label{alg:ftpl_stat_games}
\begin{algorithmic}[1]
  \small
  \State \textbf{Input:}   Parameter of exponential distribution $\eta$, approximate optimization oracles $\maxoracle{\alpha, \beta}{\cdot}, \minoracle{\alpha'}{\cdot}$ for  problems~\eqref{eqn:max_oracle},~\eqref{eqn:min_oracle} respectively
  \For{$t = 1 \dots T$}
  \State \hspace{-0.15in} Let $P_t$ be the distribution of random variable $\theta_t(\sigma)$, which is generated as follows:
  \begin{enumerate}[label=(\roman*)]
     \item Generate a random vector $\sigma$ such that $\{\sigma_{j}\}_{j = 1}^d \stackrel{i.i.d}{\sim} \text{Exp}(\eta)$
     \item Compute $\theta_t(\sigma)$ as
  \[
  \theta_t(\sigma) = \maxoracle{\alpha, \beta}{\{\hat{\theta}_i\}_{i=1}^{t-1}, \sigma}.\]
  \end{enumerate}
  \State \hspace{-0.15in} Compute $\hat{\theta}_t$ as
  \[
  \hat{\theta}_t = \minoracle{\alpha'}{P_t}.
  \]
  \EndFor
  \State \textbf{Output:} $\{\hat{\theta}_1, \dots \hat{\theta}_T\}, \{P_1, \dots P_T\}$.
\end{algorithmic}
\end{algorithm}

\begin{theorem}[Approximate NE]
\label{thm:ftpl_games}
Consider the statistical game in Equation~\eqref{eqn:minimax_objective}. Suppose $\Theta\subseteq\mathbb{R}^d$ is compact with $\ell_{\infty}$ diameter $D$, i.e., \mbox{$ D =  \sup_{\theta_1,\theta_2 \in \Theta}\|\theta_1-\theta_2\|_{\infty}$.} Suppose $R(\hat{\theta},\theta)$ is $L$-Lipschitz in its second argument w.r.t $\ell_1$ norm:
\[
\forall \hat{\theta}, \theta_1, \theta_2 \quad |R(\hat{\theta},\theta_1) - R(\hat{\theta},\theta_2)| \leq L\|\theta_1-\theta_2\|_1.
\]
Suppose Algorithm~\ref{alg:ftpl_stat_games} is run for $T$ iterations with approximate optimization subroutines $\maxoracle{\alpha, \beta}{\cdot}$, $\minoracle{\alpha'}{\cdot}$ for solving the maximization and minimization problems. 
Let $\hat{\theta}_{\RND}$ be the randomized estimator obtained by uniformly sampling an estimator from the iterates $\{\hat{\theta}_t\}_{t=1}^T$. Define the mixture distribution $P_{\AVG}$ as $\frac{1}{T}\sum_{i=1}^T P_i$.  Then $(\hat{\theta}_{\RND}, P_{\AVG})$ is an approximate mixed strategy NE of the statistical game in Equation~\eqref{eqn:minimax_objective}
\begin{align*}
\sup_{\theta \in \Theta} R(\hat{\theta}_{\RND}, \theta) - \epsilon \leq R(\hat{\theta}_{\RND}, P_{\AVG}) \leq \inf_{\hat{\theta} \in \mathcal{D}} R(\hat{\theta}, P_{\AVG}) + \epsilon,
\end{align*}
where $\epsilon = O\left(\eta d^2  + \frac{d(\beta T + D)}{\eta T} + \alpha+\alpha' \right).$
\end{theorem}
As an immediate consequence of Theorem~\ref{thm:ftpl_games}, we show that the minmax and maxmin values of the statistical game in Equation~\eqref{eqn:minimax_objective_mixed} are equal to each other. Moreover, when the risk is bounded, we show that the statistical game~\eqref{eqn:minimax_objective} has minimax estimators and LFPs.
\begin{corollary}[Minimax Theorem]
\label{cor:minimax_theorem}
Consider the setting of Theorem~\ref{thm:ftpl_games}. Then
\[
\inf_{\hat{\theta} \in \mathcal{M}_\mathcal{D}} \sup_{P \in \mathcal{M}_{\Theta}}  R(\hat{\theta}, P) = \sup_{P \in \mathcal{M}_{\Theta}} \inf_{\hat{\theta} \in \mathcal{M}_\mathcal{D}} R(\hat{\theta}, P) =: R^*.
\]
Furthermore, suppose the risk $R(\hat{\theta},\theta)$ is a bounded function and $\Theta$ is compact w.r.t the following metric: $\Delta_M(\theta_1,\theta_2) = \sup_{\theta \in \Theta}|M(\theta_1, \theta) - M(\theta_2,\theta)|.$
 Then there exists a minimax estimator $\hat{\theta}^*\in\mathcal{M}_{\mathcal{D}}$ whose worst-case risk satisfies
\[
\sup_{\theta \in \Theta}R(\hat{\theta}^*, \theta) = R^*,
\]
and there exists a least favorable prior \mbox{$P^* \in \mathcal{M}_{\Theta}$} whose Bayes risk satisfies
\[
\inf_{\hat{\theta} \in \mathcal{D}} R(\hat{\theta}, P^*) = R^*.
\]
\end{corollary}
We note that the assumption on compactness of $\Theta$ w.r.t $\Delta_M$ is mild and holds whenever $\Theta$ is compact w.r.t $\ell_2$ norm and $M$ is a continuous function. As another consequence of Theorem~\ref{thm:ftpl_games}, we show that Algorithm~\ref{alg:ftpl_stat_games} outputs approximate minimax estimators and LFPs.
\begin{corollary}
\label{cor:ftpl_minimax_estimators}
Consider the setting of Theorem~\ref{thm:ftpl_games}. Suppose Algorithm~\ref{alg:ftpl_stat_games} is run with \mbox{$\eta = \sqrt{\frac{1}{dL^2T}}.$}
 Then the worst-case risk of $\hat{\theta}_{\RND}$ satisfies
\[
\sup_{\theta \in \Theta} R(\hat{\theta}_{\RND},\theta) \leq  R^* + O(d^{\frac{3}{2}}LT^{-\frac{1}{2}} + \alpha+\alpha' + \beta d^{\frac{3}{2}}LT^{\frac{1}{2}}).
\]
Moreover, $P_{\AVG}$ is approximately least favorable with the associated Bayes risk satisfying
\begin{align*}
\inf_{\hat{\theta} \in \mathcal{D}} R(\hat{\theta}, P_{\AVG}) \geq  R^* - O(d^{\frac{3}{2}}LT^{-\frac{1}{2}} + \alpha+\alpha' + \beta d^{\frac{3}{2}}LT^{\frac{1}{2}}).
\end{align*}
In addition, suppose the loss $M$ used in the computation of risk is convex in its first argument. Let $\hat{\theta}_{\AVG}$ be the deterministic estimator which is equal to the mean of the probability distribution associated with $\hat\theta_{\RND}$. Then the worst-case risk of $\hat{\theta}_{\AVG}$ satisfies
\[
\sup_{\theta \in \Theta} R(\hat{\theta}_{\AVG},\theta) \leq R^* + O(d^{\frac{3}{2}}LT^{-\frac{1}{2}} + \alpha+\alpha' + \beta d^{\frac{3}{2}}LT^{\frac{1}{2}}),
\]
and $\hat{\theta}_{\AVG}$ is an approximate Bayes estimator for prior $P_{\AVG}$.
\end{corollary}
\begin{remark}[Near Optimal Estimator]
Corollary~\ref{cor:ftpl_minimax_estimators} shows that when the approximation error of the optimization oracles is sufficiently small and when $T$ is large enough, Algorithm~\ref{alg:ftpl_stat_games} outputs a minimax estimator with worst-case risk $(1+o(1))R^*$.
This improves upon the approximate minimax estimators that are usually designed in statistics, which have a worst-case risk of $O(1)R^*$.
\end{remark}
\begin{remark}[Deterministic Minimax Estimators]
For general non-convex loss functions, Algorithm~\ref{alg:ftpl_stat_games} only provides a randomized minimax estimator. 
Given this, a natural question that arises is whether there exist efficient algorithms for finding a deterministic minimax estimator. Unfortunately, even with access to the optimization subroutines used by Algorithm~\ref{alg:ftpl_stat_games}, finding a deterministic minimax estimator can be NP-hard~\citep[see Theorem 9 of][]{chen2017robust}
\end{remark}
\begin{remark}[Implementation Details]
 Note that the estimators $\{\hat{\theta}_i\}_{i=1}^T$ and distributions $\{P_i\}_{i=1}^T$ output by Algorithm~\ref{alg:ftpl_stat_games} are infinite dimensional objects and can not in general be stored using finitely many bits. However, in practice, we use independent samples generated from $P_i$ as its proxy and only work with these samples. Since $\hat{\theta}_i$ is a Bayes estimator for prior $P_i$, it can be approximately computed using samples from $P_i$. This process of approximating $P_i$ with its samples introduces some approximation error and the number of samples used in this approximation need to be large enough to ensure Algorithm~\ref{alg:ftpl_stat_games} returns a minimax estimator. For the problems of finite Gaussian sequence model and linear regression studied in Sections~\ref{sec:mean_estimation},~\ref{sec:regression}, we show that $\text{poly}(d)$ samples suffice to ensure a minimax estimator.
 \end{remark}

\begin{remark}[Computation of the Oracles]
We now consider the computational aspects involved in the implementation of optimization oracles used by Algorithm~\ref{alg:ftpl_stat_games}. Recall that the maximization oracle, given any estimator, computes its worst-case risk with some linear perturbation. Since this objective could potentially be non-concave, maximizing it can take exponential time in the worst-case. And recall that the minimization oracle computes the Bayes estimator given some prior distribution. Implementation of this minimization oracle can also be computationally expensive in the worst case. While the worst case complexities are prohibitive, for a number of problems, one can make use of \emph{the problem structure} to efficiently implement these oracles in polynomial time.

In particular, we leverage symmetry and invariance properties of the statistical games to reduce the complexity of optimization oracles, while controlling their approximation errors; see Section~\ref{sec:invariance}. We further consider the case where there is no structure in the problem, other than the existence of finite-dimensional sufficient statistics for the statistical model. This allows one to reduce the computational complexity of the minimization oracle by replacing the optimization over $\mathcal{D}$ in Equation~\eqref{eqn:min_oracle} with universal function approximators such as neural networks. Moreover, one can use existing global search techniques to implement the maximization oracle. While such a heuristic approach can reduce the computational complexity of the oracles, bounding their approximation errors can be hard (recall, the worst-case risk of our estimator depends on the approximation error of the optimization oracles). Nevertheless, in later sections, we empirically demonstrate that the estimators from this approach have superior performance over many existing estimators which are known to be approximately minimax.

We briefly discuss some classical work that can be leveraged for efficient implementation of optimization oracles, albeit for specific models or settings.  For several problems, it can be shown that there exists an approximate minimax estimator in some restricted space of estimators such as linear or polynomial functions of the data~\citep{donoho1990minimax, cai2011testing,polyanskiy2019dualizing}. Such results can be used to reduce the space of estimators in the statistical game~\eqref{eqn:minimax_objective}.  By replacing $\mathcal{M}_\mathcal{D}$ in Equation~\eqref{eqn:minimax_objective} with the restricted estimator space, one can greatly reduce the computational complexity of the optimization oracles. Another class of results relies on analyses of convergence of posterior distributions. As a key instance, when the number of samples $n$ is much larger than the dimension $d$, it is well known that the posterior distribution behaves like a normal distribution, whenever the prior has sufficient mass around the true parameter~\citep{hartigan1983asymptotic}. Such a property can be used to efficiently implement the minimization oracle. 
\end{remark} 
\vspace{-0.1in}
\section{Invariance of Minimax Estimators and LFPs}
\label{sec:invariance}
In this section, we show that whenever the statistical game satisfies certain invariance properties, the computational complexity of the optimization oracles required by Algorithm~\ref{alg:ftpl_stat_games} can be greatly reduced.
We first present a classical result from statistics about the invariance properties of minimax estimators.When the statistical game in Equation~\eqref{eqn:minimax_objective_det} is invariant to group transformations, the  \emph{invariance theorem} says that there exist minimax estimators which are also invariant to these group transformations~\citep{kiefer1957invariance, berger1985statistical}. Later, we utilize this result to reduce the computational complexity of the oracles required by Algorithm~\ref{alg:ftpl_stat_games}.

We first introduce the necessary notation and terminology to formally state the invariance theorem. We note that the  theorem stated here is tailored for our setting and more general versions of the theorem can be found in~\citet{kiefer1957invariance}.
Let $G$ be a compact group of transformations on $\mathcal{X} \times \Theta$ which acts component wise; that is, for each $g \in G$, $g(X, \theta)$ can be written as $(g_1X, g_2\theta)$, where $g_1,g_2$ are transformations on $\mathcal{X}, \Theta$. With a slight abuse of notation we write $gX, g\theta$ in place of $g_1X, g_2\theta$. We assume that the group action is continuous, so that the functions $(g,X) \rightarrow gX$ and $(g,\theta) \rightarrow g\theta$ are continuous. Finally, let $\mu$ be the unique left Haar measure on $G$ with $\mu(G)=1$.  We now formally define ``invariant statistical games'', ``invariant estimators'' and ``invariant probability measures''.

\begin{defn}[Invariant Game]
A statistical game is invariant to group transformations $G$, if the following two conditions hold for each $g \in G$
\begin{itemize}
    \item for all $\theta \in \Theta$, $g\theta\in\Theta$. Moreover, the probability distribution of $gX$ is $P_{g\theta}$, whenever the distribution of $X$ is $P_{\theta}$.
    \item $M(g\theta_1, g\theta_2) = M(\theta_1, \theta_2)$, for all $\theta_1,\theta_2 \in \Theta$.
\end{itemize}
\end{defn}
\begin{defn}[Invariant Estimator]
A deterministic estimator $\hat{\theta}$ is invariant if for each $g \in G$, 
$\hat{\theta}(g\mathbb{X}^n) = g\hat{\theta}(\mathbb{X}^n),$ where $g\mathbb{X}^n = \{gX_1, \dots gX_n\}$. 
\end{defn}
\begin{defn}[Invariant Measure]
Let $\mathcal{B}(\Theta)$  be the Borel $\sigma$-algebra corresponding to the parameter space $\Theta$. A measure $\nu$ on $(\Theta, \mathcal{B}(\Theta))$ is invariant if for all $ g\in G$  and any measurable set \mbox{$A \in \mathcal{B}(\Theta)$,} $\nu(gA) = \nu(A)$.
\end{defn}
\begin{exmp}
Consider the problem of estimating the mean of a Gaussian distribution. Given $n$ samples $X_1, \dots X_n$ drawn from $\mathcal{N}(\theta, I_{d\times d})$, our goal is to estimate the unknown parameter $\theta$. Suppose the parameter space is given by $\Theta =  \{\theta':\|\theta'\|_2 \leq B\}$ and the risk of any estimator is measured w.r.t squared $L_2$  loss. Then it is easy to verify that the problem is invariant to transformations of the orthogonal group $\mathbb{O}(d) = \{U: UU^T = U^TU = I\}$.
\end{exmp}
We now present the main result concerning the existence of invariant minimax estimators.A more general version of the result can be found in ~\citep{kiefer1957invariance}.
\begin{theorem}[Invariance] Consider the statistical game in Equation~\eqref{eqn:minimax_objective}. Suppose the game is invariant to group transformations $G$. Suppose the loss metric $M$ is convex in its first argument. Then for any deterministic estimator $\hat{\theta}$, there exists an estimator $\hat{\theta}_G$ which is invariant to group transformations $G$, with worst-case risk  no larger than the worst-case risk of $\hat{\theta}$
\[
\sup_{\theta \in \Theta} R(\hat{\theta}_G, \theta) \leq \sup_{\theta \in \Theta} R(\hat{\theta}, \theta).
\]
\label{thm:invariance-est-lfp}
\end{theorem}
This shows that there exists a minimax estimator which is invariant to group transformations. 
We now utilize this invariance property to reduce the complexity of the optimization oracles. 
Let $\Theta = \bigcup_{\beta} \Theta_{\beta}$ be the partitioning of $\Theta$ into equivalence classes under the equivalence $\theta_1 \sim \theta_2,$ if $\theta_1 = g\theta_2$ for some $g \in G$.
The quotient space of $\Theta$ is defined as the set of equivalence classes of the elements of $\Theta$ under the above defined equivalence and is given by $\Theta/G = \{\Theta_{\beta}\}_{\beta}$. For an invariant estimator $\hat{\theta}$, we define $R_G(\hat{\theta}, \Theta_{\beta})$ as  $R(\hat{\theta}, \theta_{\beta})$ for any $\theta_{\beta} \in \Theta_{\beta}$. Note that this is well defined because for invariant estimators \mbox{$R(\hat{\theta}, \theta_1) = R(\hat{\theta}, \theta_2)$} whenever $\theta_1 \sim \theta_2$ (see Lemma~\ref{lem:invariant_risk}). 
Our main result shows that Equation~\eqref{eqn:minimax_objective} can be reduced to the following simpler objective
\begin{equation}
\label{eqn:minimax_objective_simplified}
    \inf_{\hat{\theta} \in \mathcal{M}_{\mathcal{D},G}} \sup_{\Theta_{\beta} \in \Theta/G}  R_G(\hat{\theta}, \Theta_{\beta}),
\end{equation}
where $\mathcal{M}_{\mathcal{D},G}$ represents the set of randomized estimators which are invariant to group transformations $G$. 
This shows that the outer minimization  over the set of all estimators in Equation~\eqref{eqn:minimax_objective} can be replaced with a minimization over just the invariant estimators. Moreover, the inner maximization over the entire parameter space $\Theta$ can be replaced with a maximization over the smaller quotient space $\Theta/G$ , which in many examples we study here is a one or two-dimensional space, irrespective of the dimension of $\Theta$. \begin{theorem}
\label{thm:ftpl_minimax_invariance}
Suppose the statistical game in Equation~\eqref{eqn:minimax_objective} is invariant to group transformations $G$. Moreover, suppose the loss metric $M$ is convex in its first argument. 
Then, 
\[
\inf_{\hat{\theta} \in \mathcal{M}_\mathcal{D}} \sup_{\theta \in \Theta}  R(\hat{\theta}, \theta) = \inf_{\hat{\theta} \in \mathcal{M}_{\mathcal{D},G}} \sup_{\Theta_{\beta} \in \Theta/G}  R_G(\hat{\theta}, \Theta_{\beta}).
\]
Moreover, given any $\epsilon$-approximate mixed strategy NE of the reduced statistical game~\eqref{eqn:minimax_objective_simplified}, one can reconstruct an $\epsilon$-approximate mixed strategy NE of the original statistical game~\eqref{eqn:minimax_objective}.
\end{theorem}
We now demonstrate how Theorem \ref{thm:ftpl_minimax_invariance} can be used on a variety of fundamental statistical estimation problems.
\subsection{Finite Gaussian Sequence Model}
\label{sec:invariance_sequence_model}
In the finite Gaussian sequence model, we are given a single sample $X \in \mathbb{R}^d$ sampled from a Gaussian distribution $\mathcal{N}(\theta, I)$. We assume the parameter $\theta$ has a bounded $L_2$ norm and satisfies $\|\theta\|_2 \leq B$. 
Our goal is to design an estimator for $\theta$ which is minimax with respect to squared-error loss. This results in the following min-max problem
\begin{equation}
\label{eqn:sequence_model}
    \inf_{\hat{\theta} \in \mathcal{M}_{\mathcal{D}}} \sup_{\|\theta\|_2 \leq B} R(\hat{\theta}, \theta) \equiv \Eover{X \sim \mathcal{N}(\theta, I)}{\|\hat{\theta}(X)-\theta\|_2^2}.
\end{equation}
\begin{theorem} 
\label{thm:sequence_model_invariance}
Let $\mathbb{O}(d)=\{U: UU^T = U^TU = I\}$ be the group of $d\times d$ orthogonal matrices with matrix multiplication as the group operation. The statistical game in Equation~\eqref{eqn:sequence_model} is invariant under the action of $\mathbb{O}(d)$, where  the action of $g \in \mathbb{O}(d)$ on $(X, \theta)$ is defined as $g(X, \theta) = (gX, g\theta)$. Moreover, the quotient space $\Theta/\mathbb{O}(d)$ is homeomorphic to the real interval $[0,B]$ and the reduced statistical game is given by
\begin{equation}
\label{eqn:sequence_model_reduced}
    \inf_{\hat{\theta} \in \mathcal{M}_{\mathcal{D},G}} \sup_{b\in[0,B]}  R(\hat{\theta}, b\mathbf{e}_1),
\end{equation}
where $\mathbf{e}_{1}$ is the first standard basis vector in $\mathbb{R}^{d}$ and $\mathcal{M}_{\mathcal{D},G}$ represents the set of randomized estimators which are invariant to the actions of orthogonal group. 
\end{theorem}
The theorem shows that the supremum in the reduced statistical game~\eqref{eqn:minimax_objective_simplified} is over a bounded interval on the real line.  So the maximization oracle in this case can be efficiently implemented using  grid search over the interval $[0, B]$. In Section~\ref{sec:mean_estimation} we use this result to obtain estimators for  Gaussian sequence model which are provably minimax and can be computed in polynomial time.
\paragraph{Estimating a few co-ordinates.} Here, we again consider with the Gaussian sequence model described above, but we are now interested in the estimation of only a subset of the co-ordinates of $\theta$. Without loss of generality, we assume these are the first $k$ coordinates. The loss $M$ is the squared $L_2$ loss on the first $k$ coordinates. The following Theorem presents the invariance properties of this problem. It relies on the group $\mathbb{O}(k)\times\mathbb{O}(d-k)$, which is defined as the set of orthogonal matrices of the form $g=\begin{bmatrix} g_1 & 0 \\0 & g_2 \end{bmatrix}$ where $g_1\in\mathbb{O}(k)$ and $g_2 \in\mathbb{O}(d-k)$.
\begin{theorem}
\label{thm:sequence_model_invariance_few}
The statistical game described above is invariant under the action of the group $\mathbb{O}(k)\times\mathbb{O}(d-k)$. Moreover, the quotient space $\Theta/\mathbb{O}(k)\times\mathbb{O}(d-k)$ is homeomorphic to the ball of radius $B$ centered at origin in $\mathbb{R}^2$ and the reduced statistical game is given by
\begin{equation}
\label{eqn:sequence_model_few_simplified}
        \inf_{\hat{\theta} \in \mathcal{M}_{\mathcal{D},G}} \sup_{b_1^2+b_2^2\leq B^2}  R(\hat{\theta}, [b_1\mathbf{e}_{1,k}, b_2\mathbf{e}_{1,d-k}]),
\end{equation}
where $\mathbf{e}_{1,k}$ is the first standard basis vector in $\mathbb{R}^{k}$ and  $\mathcal{M}_{\mathcal{D},G}$ represents the set of randomized estimators which are invariant to the actions of orthogonal group. 
\end{theorem}

\subsection{Linear Regression}
\label{sec:invariance_regression}
In the problem of linear regression with random design we are given $n$ independent samples $D_n = \{(X_i,Y_i)\}_{i=1}^n$ generated from a linear model $Y_i = X_i^T\theta^* + \epsilon_i$, where \mbox{$X_i \sim \mathcal{N}(0, I)$,} and $\epsilon_i \sim \mathcal{N}(0, 1)$. We assume the true regression vector is bounded and satisfies $\|\theta^*\|_2 \leq B$.  Our goal is to design minimax estimator for estimating $\theta^*$ from $D_n$, w.r.t squared error loss. 
This leads us to the following min-max problem
\begin{equation}
\label{eqn:linear_regression}
\inf\limits_{\hat{\theta}\in\mathcal{M}_{\mathcal{D}}}\sup\limits_{||\theta||_2\leq B}R(\hat{\theta},\theta)\equiv\mathbb{E}_{D_n}\left[||\hat{\theta}(D_n)-\theta||_2^2\right].
\end{equation}
\begin{theorem}
\label{thm:linear_regression_invariance}
The statistical game in Equation~\eqref{eqn:linear_regression} is invariant under the action of the orthogonal group $\mathbb{O}(d)$, where  the action of $g \in \mathbb{O}(d)$ on $((X,Y), \theta)$ is defined as $g((X,Y), \theta) = ((gX,Y), g\theta)$. Moreover, the quotient space $\Theta/\mathbb{O}(d)$ is homeomorphic to the interval $[0,B]$ and the reduced statistical game is given by
\begin{equation}
\label{eqn:linear_regression_reduced}
    \inf_{\hat{\theta} \in \mathcal{M}_{\mathcal{D},G}} \sup_{b\in[0,B]}  R(\hat{\theta}, b\mathbf{e}_1),
\end{equation}
where $\mathcal{M}_{\mathcal{D},G}$ represents the set of randomized estimators which are invariant to the actions of orthogonal group. 
\end{theorem}

\subsection{Normal Covariance Estimation}
\label{sec:invariance_covariance}
In the problem of normal covariance estimation we are given $n$ independent samples  $\mathbb{X}^n = \{X_i\}_{i=1}^n$ drawn from $N(0,\Sigma)$. Here, we assume that the true $\Sigma$ has a bounded operator norm and satisfies  $\|\Sigma\|_2 \leq B$. Our goal is to construct an estimator for $\Sigma$ which is minimax w.r.t the entropy loss, which is defined as
\[M(\Sigma_1,\Sigma_2)=\mbox{tr}\left(\Sigma_1^{-1}\Sigma_2\right)-\log|\Sigma_1^{-1}\Sigma_2|-d.\]
This leads us to the following min-max problem
\begin{equation}
\label{eqn:covariance}
\inf\limits_{\hat{\Sigma}\in\mathcal{M}_{\mathcal{D}}}\sup\limits_{\Sigma \in \Xi}R(\hat{\Sigma},\Sigma)\equiv\mathbb{E}_{\mathbb{X}^n}\left[M(\hat{\Sigma}(\mathbb{X}^n), \Sigma)\right],
\end{equation}
where $\Xi = \{\Sigma: ||\Sigma||_2\leq B\}$.
\begin{theorem}
\label{thm:cov_estimation_invariance}
The statistical game defined by normal covariance estimation with entropy loss is invariant under the action of the orthogonal group $\mathbb{O}(d)$, where the action of $g\in\mathbb{O}(d)$ on $(X, \Sigma)$ is defined as $g(X_i, \Sigma) = (gX_i, g\Sigma g^T$). Moreover the quotient space $\Xi/\mathbb{O}(d)$ is homeomorphic to $\Xi_G=\{\lambda\in\mathbb{R}^d: B \geq \lambda_1 \geq \dots \lambda_d > 0 \}$ and the reduced statistical game is given by
\begin{equation}
\label{eqn:covariance_simplified}
    \inf_{\hat{\Sigma} \in \mathcal{M}_{\mathcal{D},G}} \sup_{\lambda\in\Xi_G}  R(\hat{\Sigma}, \text{Diag}(\lambda)),
\end{equation}
where $\text{Diag}(\lambda)$ is the diagonal matrix whose diagonal entries are given by $\lambda$ and  $\mathcal{M}_{\mathcal{D},G}$ represents the set of randomized estimators which are invariant to the actions of orthogonal group.
\end{theorem}
The theorem shows that the maximization problem over $\Xi$ can essentially be reduced to an optimization problem over a $d$-dimensional space.
\subsection{Entropy estimation}
\label{sec:invariance_entropy}
In the problem of entropy estimation, we are given $n$ samples $\mathbb{X}^n = \{X_1, \dots X_n\}$ drawn from a discrete distribution $P = (p_1, \dots p_d)$. Here, the domain of each $X_i$ is given by $\mathcal{X} = \{1,2, \dots d\}$. Our goal is to estimate the entropy of $P$, which is defined as $f(P) = -\sum_{i=1}^d p_i\log_2{p_i}$, under the squared error loss. This leads us to the following min-max problem
\begin{equation}
\label{eqn:entropy}
\inf\limits_{\hat{f}\in\mathcal{M}_{\mathcal{D}}}\sup\limits_{P \in \mathcal{P}}R(\hat{f},P)\equiv\mathbb{E}_{\mathbb{X}^n}\left[\left(\hat{f}(\mathbb{X}^n) - f(P)\right)^2\right],
\end{equation}
where $\mathcal{P}$ is the set of all probability distributions supported on $d$ elements.
\begin{theorem}
\label{thm:entropy_estimation_invariance}
The statistical game in Equation~\eqref{eqn:entropy} is invariant to the action of the permutation group $\mathbb{S}_d$. The quotient space $\mathcal{P}/\mathbb{S}_d$ is homeomorphic to  $\mathcal{P}_G = \{P\in\mathbb{R}^d:1\geq p_1\geq\ldots\geq p_d\geq0,~\sum_ip_i=1\}$ and the reduced statistical game is given by
\begin{equation}
\label{eqn:entropy_simplified}
    \inf_{\hat{f} \in \mathcal{M}_{\mathcal{D},G}} \sup_{P\in\mathcal{P}_G}  R(\hat{f}, P),
\end{equation}
where $\mathcal{M}_{\mathcal{D},G}$ represents the set of randomized estimators which are invariant to the actions of permutation group.
\end{theorem}
 
\vspace{-0.1in}
\section{Finite Gaussian Sequence Model}
\label{sec:mean_estimation}
 In this section we consider the finite Gaussian sequence model described in Section~\ref{sec:invariance_sequence_model} and use Algorithm~\ref{alg:ftpl_stat_games} to construct a provably minimax estimator, which can be computed in polynomial time. 
  This problem has received a lot of attention in statistics because of its simplicity, relevance and its connections to non-parametric regression~\citep[see Chapter 1 of][]{johnstone2011gaussian}. When the radius of the domain $B$ is smaller than $1.15\sqrt{d}$, ~\citet{marchand2002minimax} show that the Bayes estimator with uniform prior on the boundary is a minimax estimator for the problem. For larger values of $B$, the exact minimax estimator is unknown. Several works have attempted to understand the properties of LFP in such  settings~\citep{casella1981estimating} and constructed approximate minimax estimators~\citep{bickel1981minimax}. In this work, we rely on Algorithm~\ref{alg:ftpl_stat_games} to construct an exact minimax estimator and an LFP, for any value of $B, d$.
 
 Recall, in Theorem~\ref{thm:sequence_model_invariance} we showed that the original min-max statistical game can be reduced to the  simpler problem in Equation~\eqref{eqn:sequence_model_reduced}
To use Algorithm~\ref{alg:ftpl_stat_games} to find a Nash equilibrium of the reduced game, we need efficient implementation of the required optimization oracles and a bound on their approximation errors.
The optimization problems corresponding to the oracles in Equations~\eqref{eqn:max_oracle},~\eqref{eqn:min_oracle} are given as follows
\begin{align*}
    \hat{\theta}_{t} \leftarrow \argmin_{\hat{\theta} \in \mathcal{D}_G} \Eover{b \sim P_t}{R(\hat{\theta}, b\mathbf{e}_1)},\quad b_{t}(\sigma) \leftarrow \argmax_{b \in [0,B]} \sum_{i = 1}^{t-1} R(\hat{\theta}_i, b\mathbf{e}_1) + \sigma b,
\end{align*}
where $\mathcal{D}_G$ is the set of deterministic invariant estimators and $P_t$ is the distribution of random variable $b_{t}(\sigma)$.  We now present efficient techniques for implementing these oracles (Algorithms~\ref{alg:mean_estimation_max_oracle},~\ref{alg:mean_estimation_min_oracle}).  Since the maximization problem is a $1$ dimensional optimization problem, grid search can be used to compute an approximate maximizer. The approximation error of the resulting oracle depends on the grid width and the number of samples used to compute the expectation in the risk $R(\hat{\theta}, b\mathbf{e}_1)$. Later, we show that $\text{poly}(d, B)$ grid points and samples suffice to have a small approximation error. The minimization problem, which requires finding an invariant estimator minimizing the integrated risk under any prior $P_t$, can also be efficiently implemented. As shown in Proposition~\ref{prop:mean_estimation_min_closed_form} below, the minimizer has a closed-form expression which depends on $P_t$ and modified Bessel functions. To compute an approximate minimizer of the problem, we approximate $P_t$ with its samples and rely on the closed-form expression. The approximation error of this oracle depends on the number of samples used to approximate $P_t$. We again show that $\text{poly}(d, B)$ samples suffice to have a small approximation error.
\begin{proposition}
\label{prop:mean_estimation_min_closed_form}
The optimizer $\hat{\theta}_t$ of the minimization problem defined above has the following closed-form expression
\[
\hat{\theta}_t(X) = \left(\frac{\displaystyle \Eover{b\sim P_t}{b^{3-d/2} e^{-b^2/2}I_{d/2}(b\|X\|_2) }}{\displaystyle \Eover{b\sim P_t}{b^{2-d/2}e^{-b^2/2} I_{d/2-1}(b\|X\|_2)}}\right)\frac{X}{\|X\|_2},
\]
where $I_{\nu}$ is the modified Bessel function of first kind of order $\nu$. 
\end{proposition}
\begin{algorithm}[t]
\caption{Maximization Oracle}
\label{alg:mean_estimation_max_oracle}
\begin{algorithmic}[1]
  \small
  \State \textbf{Input:} Estimators $\{\hat{\theta}_{i}\}_{i=1}^{t-1}$, perturbation $\sigma$, grid width $w$, number of samples for computation of expected risk $R(\hat{\theta},\theta)$: $N_1$
  \State Let $\{b_1,b_2\dots b_{B/w}\}$ be uniformly spaced points on $[0,B]$
  \For{$j = 1 \dots B/w$}
  \For{$i = 1\dots t-1$}
  \State Generate $N_1$ independent samples $\{X_{k}\}_{k=1}^{N_1}$ from the distribution $ \mathcal{N}(b_j \mathbf{e}_1, I)$
  \State Estimate $R(\hat{\theta}_i, b_j\mathbf{e}_1)$ as 
  $\frac{1}{N_1}\sum_{k=1}^{N_1} \|\hat{\theta}_i(X_k)-b\mathbf{e}_1\|_2^2.$
  \EndFor
  \State Evaluate the objective at $b_j$ using the above estimates
  \EndFor
  \State \textbf{Output:} $b_j$ which maximizes the objective
\end{algorithmic}
\end{algorithm}
\begin{algorithm}[t]
	\caption{Minimization Oracle}
	\label{alg:mean_estimation_min_oracle}
	\begin{algorithmic}[1]
		\small
		\State \textbf{Input:} Samples $\{b_i\}_{i=1}^{N_2}$ generated from distribution $P_t$.
		\State For any $X$, compute $\hat{\theta}_t(X)$ as
		\[
		\left(\frac{ \sum_{i=1}^{N_2} w_i b_iA(b_i\|X\|_2)}{\sum_{i=1}^{N_2}w_i }\right)\frac{X}{\|X\|_2},
		\]
		 where $A(\gamma)=\dfrac{I_{d/2}(\gamma)}{I_{d/2-1}(\gamma)}$, $w_i=b_i^{2-d/2}e^{-b_i^2/2}I_{d/2-1}(b_i\|X\|_2),$ and $I_\nu$ is the modified Bessel function of the first kind of order $\nu$.
	\end{algorithmic}
\end{algorithm}
We now show that using Algorithm~\ref{alg:ftpl_stat_games} for solving objective~\eqref{eqn:sequence_model_reduced} with Algorithms~\ref{alg:mean_estimation_max_oracle},~\ref{alg:mean_estimation_min_oracle} as optimization oracles, gives us a provably minimax estimator and an LFP for finite Gaussian sequence model.
\begin{theorem}
\label{thm:mean_estimation}
Suppose Algorithm~\ref{alg:ftpl_stat_games} is run for $T$ iterations with Algorithms~\ref{alg:mean_estimation_max_oracle},~\ref{alg:mean_estimation_min_oracle} as the maximization and minimization oracles. Suppose the  hyper-parameters of these algorithms are set as \mbox{$\eta = \frac{1}{B(B+1)\sqrt{T}}$}, $w=\frac{B}{T^{3/2}}$, \mbox{$N_1 = \frac{T^3}{(B+1)^2}, N_2 = \frac{T^4}{(B+1)^2}$.} Let $\hat{P}_t$ be the approximation of probability distribution $P_t$ used in the $t^{th}$ iteration of Algorithm~\ref{alg:ftpl_stat_games}. Moreover, let $\hat{\theta}_t$ be the output of Algorithm~\ref{alg:mean_estimation_min_oracle} in the $t^{th}$ iteration of Algorithm~\ref{alg:ftpl_stat_games}. 
\begin{enumerate}
    \item Then the averaged estimator $\hat{\theta}_{\text{avg}}(X) = \frac{1}{T}\sum_{i=1}^T\hat{\theta}_i(X)$ is approximately minimax and satisfies the following worst-case risk bound with probability at least $1-\delta$
\[
\sup_{\theta: \|\theta\|_2 \leq B}R(\hat{\theta}_{\text{avg}},\theta) \leq R^* + \tilde{O}\left(\frac{B^2(B+1)}{\sqrt{T}}\right),
\]
where $\tilde{O}(.)$ hides log factors and $R^*$ is the minimax risk.
\item  Define the mixture distribution $\hat{P}_{\AVG}$ as $\frac{1}{T}\sum_{i=1}^T \hat{P}_i$.  Let $\hat{P}_{\text{LFP}}$ be a probability distribution over $\mathbb{R}^d$ with density function defined as $\hat{p}_{\text{LFP}}(\theta) \propto \|\theta\|^{1-d}_2\hat{P}_{\AVG}(\|\theta\|_2) $, where $\hat{P}_{\AVG}(\|\theta\|_2)$ is the probability  mass placed by $\hat{P}_{\AVG}$ at $\|\theta\|_2$. Then $\hat{P}_{\text{LFP}}$ is approximately least favorable and satisfies the following with probability at least $1-\delta$
\[
\inf_{\hat{\theta} \in \mathcal{D}} R(\hat{\theta}, \hat{P}_{\text{LFP}}) \geq R^* - \tilde{O}\left(\frac{B^2(B+1)}{\sqrt{T}}\right),
\]
where the infimum is over the set of all estimators.
\end{enumerate}
\end{theorem}
We believe the polynomial factors in the bounds can be improved with a tighter analysis of the algorithm. The above Theorem shows that Algorithm~\ref{alg:ftpl_stat_games} learns an approximate minimax estimator in $\text{poly}(d,B)$ time. To the best our knowledge, this is the first result providing provable minimax estimators for finite Gaussian sequence model, for any value of $B$.

\vspace{-0.1in}
\section{Linear Regression}
\label{sec:regression}
  In this section we consider the linear regression problem described in Section~\ref{sec:invariance_regression} and provide a provably minimax estimator. Recall, in  Theorem~\ref{thm:linear_regression_invariance} we showed that the original min-max statistical game can be reduced to the simpler problem in Equation~\eqref{eqn:linear_regression_reduced}.
We now provide efficient implementations of the optimization oracles required by Algorithm~\ref{alg:ftpl_stat_games} for finding a Nash equilibrium of this game. 
 The optimization problems corresponding to the two optimization oracles are as follows
\begin{align*}
    \hat{\theta}_{t} \leftarrow \argmin_{\hat{\theta} \in \mathcal{D}_G} \Eover{b \sim P_t}{R(\hat{\theta}, b\mathbf{e}_1)},\quad b_{t}(\sigma) \leftarrow \argmax_{b \in [0,B]} \sum_{i = 1}^{t-1} R(\hat{\theta}_i, b\mathbf{e}_1) + \sigma b,
\end{align*}
where $\mathcal{D}_G$ is the set of deterministic invariant estimators and $P_t$ is the distribution of random variable $b_{t}(\sigma)$. 
 Similar to the Gaussian sequence model, the maximization oracle can be efficiently implemented via a grid search over $[0, B]$ (Algorithm~\ref{alg:regression_max_oracle}). 
 The solution to the minimization problem has a closed-form expression in terms of the mean and normalization constant of Fisher-Bingham distribution, which is a distribution obtained by constraining multivariate normal distributions to lie on the surface of unit sphere~\citep{kume2005saddlepoint}. 
Letting $\mathbb{S}^{d-1}$ be the unit sphere in $\mathbb{R}^d$, the probability density of a random variable $Z$ distributed according to Fisher-Bingham distribution is given by
\[
p(Z;A, \gamma) = C(A, \gamma)^{-1}\exp\left(-Z^TAZ  + \iprod{\gamma}{Z}\right),
\]
where $Z \in \mathbb{S}^{d-1}$, and $\gamma \in \mathbb{R}^d$, $A \in \mathbb{R}^{d\times d}$ are the parameters of the distribution with $A$ being positive semi-definite and $C(A, \gamma)$ is the normalization constant. Note that the mean of  Fisher-Bingham distribution is given by $C(A,\gamma)^{-1}\frac{\partial}{\partial \gamma}  C(A,\gamma)$. The following proposition obtains a closed-form expression for $\hat{\theta}_t$ in terms of $C(A, \gamma)$ and $\frac{\partial}{\partial \gamma}  C(A,\gamma)$.
\begin{proposition}
\label{prop:regression_min_closed_form}
The optimizer $\hat{\theta}_t$ of the minimization problem defined above has the following closed-form expression
\[
\hat{\theta}_t(D_n) = \frac{\displaystyle\Eover{b\sim P_t}{b^2 \frac{\partial}{\partial \gamma} C\left(2^{-1} b^2 \mathbf{X}^T\mathbf{X}, \gamma\right)\Big|_{\gamma = b\mathbf{X}^T\mathbf{Y}}}}{\displaystyle\Eover{b\sim P_t}{b  C\left(2^{-1} b^2\mathbf{X}^T\mathbf{X}, b\mathbf{X}^T\mathbf{Y}\right) }},
\]
where $\mathbf{X} = [X_1,X_2\dots X_n]^T$ and $\mathbf{Y} = [Y_1, Y_2 \dots Y_n]$.
\end{proposition}
We note that there exist a number of efficient techniques for computation of the mean and normalization constant of Fisher-Bingham distribution~\citep{kume2005saddlepoint, imhof1961computing}. In our experiments we rely on the technique of ~\citet{kume2005saddlepoint} (we relegate the details of this technique to Appendix~\ref{sec:appx_fisher_bingham_mean}). To compute an approximate optimizer of the minimization problem, we approximate $P_t$ with its samples and rely on the above closed-form expression. Algorithm~\ref{alg:regression_min_oracle} describes the resulting minimization oracle.
\begin{algorithm}[t]
\caption{Regression Maximization Oracle}
\label{alg:regression_max_oracle}
\begin{algorithmic}[1]
  \small
  \State \textbf{Input:} Estimators $\{\hat{\theta}_{i}\}_{i=1}^{t-1}$, perturbation $\sigma$, grid width $w$, number of samples for computation of expected risk $R(\hat{\theta},\theta)$: $N_1$
  \State Let $\{b_1,b_2\dots b_{B/w}\}$ be uniformly spaced points on $[0,B]$
  \For{$j = 1 \dots B/w$}
  \For{$i = 1\dots t-1$}
  \State Generate $N_1$ independent datasets $\{D_{n,k}\}_{k=1}^{N_1}$ from the linear model with true regression vector $b_j\mathbf{e}_1$
  \State Estimate $R(\hat{\theta}_i, b_j\mathbf{e}_1)$ as 
  $\frac{1}{N_1}\sum_{k=1}^{N_1} \|\hat{\theta}_i(D_{n,k})-b\mathbf{e}_1\|_2^2.$
  \EndFor
  \State Evaluate the objective at $b_j$ using the above estimates
  \EndFor
  \State \textbf{Output:} $b_j$ which maximizes the objective
\end{algorithmic}
\end{algorithm}
\begin{algorithm}[t]
	\caption{Regression Minimization Oracle}
	\label{alg:regression_min_oracle}
	\begin{algorithmic}[1]
		\small
		\State \textbf{Input:} Samples $\{b_i\}_{i=1}^{N_2}$ generated from distribution $P_t$ 
		\State For any $D_n$, compute $\hat{\theta}_t(D_n)$ as
		\[
		\hat{\theta}_t(D_n) = \frac{\sum_{i=1}^{N_2} b_i^2 \frac{\partial}{\partial \gamma} C\left(2^{-1} b_i^2 \mathbf{X}^T\mathbf{X}, \gamma\right)\Big|_{\gamma = b_i\mathbf{X}^T\mathbf{Y}}}{\sum_{i=1}^{N_2} b_i  C\left(2^{-1} b_i^2\mathbf{X}^T\mathbf{X}, b_i\mathbf{X}^T\mathbf{Y}\right)},
		\]
		 where $\mathbf{X} = [X_1,X_2\dots X_n]^T$ and $\mathbf{Y} = [Y_1, Y_2 \dots Y_n]$.
	\end{algorithmic}
\end{algorithm}
We now show that using Algorithm~\ref{alg:ftpl_stat_games} for solving objective~\eqref{eqn:linear_regression_reduced} with Algorithms~\ref{alg:regression_max_oracle},~\ref{alg:regression_min_oracle} as optimization oracles, gives us a provably minimax estimator and an LFP for linear regression. 
\begin{theorem}
\label{thm:regression}
Suppose Algorithm~\ref{alg:ftpl_stat_games} is run for $T$ iterations with Algorithms~\ref{alg:regression_max_oracle},~\ref{alg:regression_min_oracle} as the maximization and minimization oracles. Suppose the  hyper-parameters of these algorithms are set as \mbox{$\eta = \frac{1}{B(B\sqrt{n}+1)\sqrt{T}}$}, $w=\frac{B}{T^{3/2}}$, \mbox{$N_1 = \frac{T^3}{(B\sqrt{n}+1)^2}, N_2 = \frac{T^4}{(B\sqrt{n}+1)^2}$.} Let $\hat{P}_t$ be the approximation of probability distribution $P_t$ used in the $t^{th}$ iteration of Algorithm~\ref{alg:ftpl_stat_games}. Moreover, let $\hat{\theta}_t$ be the output of Algorithm~\ref{alg:regression_min_oracle} in the $t^{th}$ iteration of Algorithm~\ref{alg:ftpl_stat_games}. 
\begin{enumerate}
    \item Then the averaged estimator $\hat{\theta}_{\text{avg}}(D_n) = \frac{1}{T}\sum_{i=1}^T\hat{\theta}_i(D_n)$ is approximately minimax and satisfies the following worst-case risk bound with probability at least $1-\delta$
\[
\sup_{\theta: \|\theta\|_2 \leq B}R(\hat{\theta}_{\text{avg}},\theta) \leq R^* + \tilde{O}\left(B^2(B+1)\sqrt{\frac{n}{T}}\right).
\]
\item  Define the mixture distribution $\hat{P}_{\AVG}$ as $\frac{1}{T}\sum_{i=1}^T \hat{P}_i$.  Let $\hat{P}_{\text{LFP}}$ be a probability distribution over $\mathbb{R}^d$ with density function defined as $\hat{p}_{\text{LFP}}(\theta) \propto \|\theta\|^{1-d}_2\hat{P}_{\AVG}(\|\theta\|_2) $, where $\hat{P}_{\AVG}(\|\theta\|_2)$ is the probability  mass placed by $\hat{P}_{\AVG}$ at $\|\theta\|_2$. Then $\hat{P}_{\text{LFP}}$ is approximately least favorable and satisfies the following with probability at least $1-\delta$
\[
\inf_{\hat{\theta} \in \mathcal{D}} R(\hat{\theta}, \hat{P}_{\text{LFP}}) \geq R^* - \tilde{O}\left(B^2(B+1)\sqrt{\frac{n}{T}}\right).
\]
\end{enumerate}
\end{theorem}

\vspace{-0.1in}
\section{Covariance Estimation}
\label{sec:covariance}
In this section, we consider the problem of normal covariance estimation. Recall, in Section~\ref{sec:invariance_covariance} we showed that the problem is invariant to the action of the orthogonal group and can be reduced to the simpler problem in Equation~\eqref{eqn:covariance_simplified}.
The optimization problems corresponding to the oracles in Equations~\eqref{eqn:max_oracle},~\eqref{eqn:min_oracle} are as follows
\begin{align*}
    \hat{\Sigma}_{t} \leftarrow \argmin_{\hat{\Sigma} \in \mathcal{D}_G} \Eover{\lambda \sim P_t}{R(\hat{\Sigma}, \text{Diag}(\lambda))},\quad \lambda_{t}(\sigma) \leftarrow \argmax_{\lambda \in \Xi_G} \sum_{i = 1}^{t-1} R(\hat{\Sigma}_i, \text{Diag}(\lambda)) + \iprod{\lambda}{\sigma},
\end{align*}
where $\mathcal{D}_G$ is the set of deterministic invariant estimators and $P_t$ is the distribution of random variable $\lambda_{t}(\sigma)$. Note that the maximization problem involves optimization of a non-concave objective in $d$-dimensional space. So, implementing a maximization oracle with low approximation error can be computationally expensive, especially in high dimensions. Moreover, unlike finite Gaussian sequence model and linear regression, the minimization problem doesn't have a closed form expression, and it is not immediately clear how to efficiently implement a minimization oracle with low approximation error. 
In such scenarios, we show that one can rely on a combination of heuristics and problem structure to further reduce the computational complexity of the optimization oracles. Although relying on heuristics comes at the expense of theoretical guarantees, in later sections, we empirically demonstrate that the resulting estimators
have superior performance over classical estimators. 
We begin by showing that the domain of the outer minimization in Equation~\eqref{eqn:covariance_simplified} can be reduced to a smaller set of estimators. Our reduction relies on Blackwell's theorem, which shows that for convex loss functions $M$, there exists a minimax estimator which is a function of the sufficient statistic~\citep{Ibragimov81book}. We note that Blackwell's theorem is very general and can be applied to a wide range of problems, to reduce the computational complexity of the minimization oracle.
\begin{proposition}
\label{prop:covariance_minimax_estimator}
Consider the problem of normal covariance estimation. Let \mbox{$S_n = \frac{\sum_{i=1}^nX_iX_i^T}{n}$} be the empirical covariance matrix and let $U\Delta U^T$ be the eigen decomposition of $S_n$. Then there exists a minimax estimator which can be approximated arbitrarily well using estimators of the form $\hat{\Sigma}_{f,g}(\mathbb{X}^n) = U\tilde{\Sigma}_{f,g}(\Delta)U^T,$
where $\tilde{\Sigma}_{f,g}(\Delta)$ is a diagonal matrix whose $i^{th}$ diagonal entry is given by
\[
\tilde{\Sigma}_{f,g,i}(\Delta) = f\left(\Delta_{i},\sum_{j\neq i}g(\Delta_{i},\Delta_{j})\right),
\]
for some functions $f:\mathbb{R}^{d+1}\to\mathbb{R},g:\mathbb{R}^2\to\mathbb{R}^{d}$. Here, $\Delta_i$ is the $i^{th}$ diagonal entry of $\Delta$. Moreover, the optimization problem in Equation~\eqref{eqn:covariance_simplified} can be reduced to the following simpler problem
\begin{equation}
\label{eqn:covariance_reduced2}
\inf\limits_{\hat{\Sigma} \in \mathcal{M}_{f,g}}\sup\limits_{\lambda \in \Xi_G}R(\hat{\Sigma},\text{Diag}(\lambda))= R^*,
\end{equation}
where $\mathcal{M}_{f,g}$ is the set of probability distributions over estimators of the form $\hat{\Sigma}_{f,g}$. 
\end{proposition}
We now use Algorithm~\ref{alg:ftpl_stat_games} to solve the statistical game in Equation~\eqref{eqn:covariance_reduced2}. The optimization problems corresponding to the two optimization oracles are given by
\begin{align*}
    &\hat{f}_t,\hat{g}_t \leftarrow \argmin_{f,g} \Eover{\lambda \sim P_t}{R(\hat{\Sigma}_{f,g}, \text{Diag}(\lambda))},\quad \lambda_{t}(\sigma) \leftarrow \argmax_{\lambda \in \Xi_G} \sum_{i = 1}^{t-1} R(\hat{\Sigma}_{\hat{f}_i,\hat{g}_i}, \text{Diag}(\lambda)) + \iprod{\lambda}{\sigma}.
\end{align*}
We rely on heuristics to efficiently implement these oracles. To implement the minimization oracle, we use neural networks (which are universal function approximators) to parameterize functions $f,g$. Implementing the minimization oracle then boils down to the finding the parameters of these networks which minimize the objective. To implement the maximization oracle, we rely on global search techniques. In our experiments, we use DragonFly~\citep{kandasamy2019tuning}, which is a zeroth order optimization technique, to implement this oracle. Note that these heuristics do not come with any guarantees and as a result the oracles are not guaranteed to have a small approximation error. Despite this, we empirically demonstrate that the estimators learned using this approach have good performance.

\vspace{-0.1in}
\section{Entropy Estimation}
\label{sec:entropy}
In this section, we consider the problem of entropy estimation. Recall, in Section~\ref{sec:invariance_entropy} we showed that the problem is invariant to the action of permutation group and can be reduced to the simpler problem in Equation~\eqref{eqn:entropy_simplified}. 
Similar to the problem of covariance estimation, implementing the optimization oracles for this problem, with low approximation error, can be computationally expensive. So we again rely on heuristics and problem structure to reduce the computational complexity of optimization oracles. 
\begin{proposition}
\label{prop:entropy_minimax_estimator}
Consider the problem of entropy estimation. Let \mbox{$\hat{P}_n = (\hat{p}_1,\dots \hat{p}_d)$} be the observed empirical probabilities. Then there exists a minimax estimator which can be approximated arbitrarily well using estimators of the form $\hat{f}_{g,h}(\hat{P}_n) = g(\sum_{i=1}^dh(\hat{p}_i)),$ 
for some functions $g:\mathbb{R}^{d+1}\to\mathbb{R},h:\mathbb{R}\to\mathbb{R}^{d+1}$. Moreover, the optimization problem in Equation~\eqref{eqn:entropy_simplified} can be reduced to the following problem
\begin{equation}
\label{eqn:entropy_reduced2}
\inf_{\hat{f} \in \mathcal{M}_{g,h}} \sup_{P\in\mathcal{P}_G}  R(\hat{f}, P)= R^*,
\end{equation}
where $\mathcal{M}_{g,h}$ is the set of probability distributions over estimators of the form $\hat{f}_{g,h}$. 
\end{proposition}
The proof of this proposition is presented in Appendix \ref{sec:entropy-proof}. We now use Algorithm~\ref{alg:ftpl_stat_games} to solve the statistical game in Equation~\eqref{eqn:entropy_reduced2}. The optimization problems corresponding to the two optimization oracles are given by
\begin{align*}
    &\hat{g}_t,\hat{h}_t \leftarrow \argmin_{g,h} \Eover{P \sim P_t}{R(\hat{f}_{g,h}, P)},\quad P_{t}(\sigma) \leftarrow \argmax_{P \in \mathcal{P}_G} \sum_{i = 1}^{t-1} R(\hat{f}_{\hat{g}_i,\hat{h}_i}, P) + \iprod{P}{\sigma},
\end{align*}
where $P_t$ is the distribution of random variable $P_{t}(\sigma)$.
To implement the minimization oracle, we use neural networks to parameterize functions $g,h$. To implement the maximization oracle, we rely on DragonFly. 
\vspace{-0.1in}
\section{Experiments}
\label{sec:exps}

In this section, we present experiments showing performance of the proposed technique for constructing minimax estimators. While our primary focus is on the finite Gaussian sequence model and linear regression for which we provided provably minimax estimators, we also present experiments on other problems such as covariance and entropy estimation. For each of these problems, we begin by describing the setup as well as the baseline algorithms, before proceeding to a discussion of the experimental findings.

\subsection{Finite Gaussian Sequence Model}
In this section, we focus on experiments related to the finite Gaussian sequence model. We first consider the case where the risk is measured with respect to squared error loss, \emph{i.e.,} $M(\theta_1,\theta_2) = \|\theta_1-\theta_2\|_2^2$.
\vspace{-0.15in}
\paragraph{Proposed Technique.} We use Algorithm~\ref{alg:ftpl_stat_games} with optimization oracles described in  Algorithms~\ref{alg:mean_estimation_max_oracle},~\ref{alg:mean_estimation_min_oracle} to find minimax estimators for this problem. We set the hyper-parameters of our algorithm as follows: number of iterations of FTPL $T = 500$, grid width $w = 0.05\times B$, number of samples for computation of $R(\hat{\theta}, \theta)$ in Algorithm~\ref{alg:mean_estimation_max_oracle} $N_1= 1000$, number of samples generated from $P_t$ in Algorithm~\ref{alg:mean_estimation_min_oracle} $N_2 = 1000$. We note that these are default values and were not tuned. The randomness parameter $\eta$ in Algorithm~\ref{alg:ftpl_stat_games} was tuned using a coarse grid search. We report the performance of the following two estimators constructed using the iterates of Algorithm~\ref{alg:ftpl_stat_games}: (a) Averaged Estimator $\hat{\theta}_{\AVG}(X) = \frac{1}{T}\sum_{i=1}^T\hat{\theta}_i(X)$, (b) Bayes estimator for prior $\frac{1}{T}\sum_{i=1}^T\hat{P}_i$ which we refer to as ``Bayes estimator for avg. prior''. The performance of the randomized estimator $\hat{\theta}_{\RND}$ is almost identical to the performance of $\hat{\theta}_{\AVG}$. So we do not report its performance.
\vspace{-0.15in}
\paragraph{Baselines.} We compare our estimators with various baselines: (a) standard estimator $ \hat{\theta}(X) = X$, (b) James Stein estimator $\hat{\theta}(X) = \left(1-(d-3)/\|X\|_2^2\right)^+X$, where $c^+ = \max(0, c)$, (c) projection estimator (MLE) $\hat{\theta}(X) = \text{min}(\|X\|_2,B)\frac{X}{\|X\|_2}$, (d) Bayes estimator for uniform prior on the boundary; this estimator is known to be minimax for $B \leq 1.15\sqrt{d}$. 
\vspace{-0.15in}
\paragraph{Worst-case Risk.} We compare the performance of various estimators based on their worst-case risk. The worst-case risk of the standard estimator is equal to $d$. The worst case risk of all the other estimators is computed as follows. Since all these estimators are invariant to orthogonal group transformations, the risk $R(\hat{\theta}, \theta)$ only depends on $\|\theta\|_2$ and not its direction. So the worst-case risk can be obtained by solving the following optimization problem: $\max_{b \in [0,B]}R(\hat{\theta}, b\mathbf{e}_1)$, where $\mathbf{e}_1$ is the first standard basis vector. We use grid search to solve this problem, with $0.05\times B$ grid width. We use $10^4$ samples to approximately compute $R(\hat{\theta}, b\mathbf{e}_1)$ for any $\hat\theta, b$. 
\vspace{-0.15in}
\paragraph{Duality Gap.} For estimators derived from our technique, we also present the duality gap, which is defined as $\sup_{\theta \in \Theta} R(\hat{\theta}_{\AVG}, \theta) - \inf_{\hat{\theta} \in \mathcal{D}} R(\hat{\theta}, \frac{1}{T}\sum_{i = 1}^T \hat{P}_i)$. Duality gap quantifies the closeness of $(\hat{\theta}_{\AVG}, \frac{1}{T}\sum_{i = 1}^T \hat{P}_i)$ to a Nash equilibrium. Smaller the gap, closer we are to an equilibrium.
\vspace{-0.15in}
\paragraph{Results.}  Table~\ref{tab:mean_estimation} shows the performance of various estimators for various values of $d, B$ along with the duality gap for our estimator.  For $B=\sqrt{d}$, the estimators obtained using Algorithm~\ref{alg:ftpl_stat_games} have similar performance as the ``Bayes estimator for uniform prior   on boundary'', which is known to be minimax. For $B=2\sqrt{d}, 3\sqrt{d}$ for which the exact minimax estimator is unknown, we achieve  better performance than baselines. Finally, we note that the duality gap numbers presented in the table can be made smaller by running our algorithm for more iterations. When the dimension $d=1$, \citet{donoho1990minimax} derived lower bounds for the minimax risk, for various values of $B$. In Table~\ref{tab:mean_lower_bound}, we compare the worst risk of our estimator with these established lower bounds. It can be seen that the worst case risk of our estimator is close to the lower bounds.
\begin{table*}
  \caption{\small{Worst-case risk of various estimators for finite Gaussian sequence model. The risk is measured with respect to squared error loss. The worst-case risk of the estimators from Algorithm~\ref{alg:ftpl_stat_games} (last two rows) is smaller than the worst-case risk of baselines. The numbers in the brackets for Averaged Estimator represent the duality gap.}}
  \label{tab:mean_estimation}
  \centering
  \resizebox{\textwidth}{!}{
  \begin{tabular}{|c||c|c|c||c|c|c||c|c|c|}
    \hline
    &\multicolumn{9}{c|}{\textbf{Worst-case Risk}}\\
    \cline{2-10}
     & \multicolumn{3}{c||}{$\mathbf{B = \sqrt{d}}$}  & \multicolumn{3}{c|}{$\mathbf{B = 1.5\sqrt{d}}$} & \multicolumn{3}{c|}{$\mathbf{B = 2\sqrt{d}}$}                     \\ \cline{2-10}
    \textbf{Estimator} & $\mathbf{d = 10}$ & $\mathbf{d=20}$ &  $\mathbf{d = 30}$ & $\mathbf{d = 10}$& $\mathbf{d=20}$ &  $\mathbf{d = 30}$ & $\mathbf{d = 10}$& $\mathbf{d=20}$ &  $\mathbf{d = 30}$
    \\ \hline
    Standard&10&20&30&10&20&30&10&20&30\\\hline
    James Stein&6.0954&11.2427&16.073&7.9255&15.0530&21.3410&8.7317&16.6971&24.7261\\\hline
    Projection &8.3076&17.4788&26.7873&10.3308&20.3784&30.2464&10.1656&20.2360&30.3805\\\hline
    \begin{tabular}{@{}c@{}}Bayes estimator \\ for uniform prior \\  on boundary \end{tabular}&\textbf{4.8559}&\textbf{9.9909}&\textbf{14.8690}&11.7509&23.4726&35.2481&24.5361&49.0651&73.3158\\\hline
    \begin{tabular}{@{}c@{}}\textbf{Averaged} \\ \textbf{Estimator}\end{tabular}&\makecell{\textbf{4.7510}\\(0.1821)}&\makecell{\textbf{9.7299}\\(0.2973)}&\makecell{\textbf{14.8790}\\(0.0935)}&\makecell{\textbf{6.7990}\\(0.0733)}&\makecell{\textbf{13.8084}\\(0.2442)}&\makecell{\textbf{20.5704}\\( 0.0087)}&\makecell{\textbf{7.8504}\\(0.3046)}&\makecell{\textbf{15.6686}\\(0.2878)}&\makecell{\textbf{23.8758}\\(0.6820)}\\\hline
    \begin{tabular}{@{}c@{}}\textbf{Bayes estimator} \\ \textbf{for avg. prior} \end{tabular}&\textbf{4.9763}&\textbf{10.1273}&\textbf{14.8128}&\textbf{6.7866}&\textbf{13.8200}&\textbf{20.3043}&\textbf{7.8772}&\textbf{15.6333}&\textbf{23.5954}\\\hline
\end{tabular}
  }
\end{table*}

\vspace{-0.1in}
\begin{table}[tbh]
\caption{Comparison of the worst case risk of $\hat{\theta}_{\AVG}$ with established lower bounds from \cite{donoho1990minimax} for finite Gaussian sequence model with $d=1$.}
\label{tab:mean_lower_bound}
\begin{center}
\begin{tabular}{|c|c|c|c|c|}
\hline
& $\mathbf{B=1}$ & $\mathbf{B=2}$ & $\mathbf{B=3}$ & $\mathbf{B=4}$\\
\hline
\begin{tabular}{@{}c@{}}\textbf{Worst case risk of} \\ \textbf{Averaged Estimator}\end{tabular} & 0.456 & 0.688 & 0.799 & 0.869\\
\hline
\textbf{Lower bound} & 0.449 & 0.644 & 0.750 & 0.814\\
\hline
\end{tabular}\end{center}
\end{table}
\subsubsection{Estimating a few coordinates}
\begin{table*}[t]
  \caption{\small{Worst-case risk of various estimators for bounded normal mean estimation when the risk is evaluated with respect to squared loss on the first $k$ coordinates.}}
  \label{tab:mean_few_coordinates}
  \hspace{-0.1in}
  \resizebox{\columnwidth}{!}{\small{
  \begin{tabular}{|c||c|c|c||c|c|c||c|c|c|}
    \hline
    &\multicolumn{9}{c|}{\textbf{Worst-case Risk}}\\
    \cline{2-10}
     & \multicolumn{3}{c||}{$\mathbf{k = 1, B = \sqrt{d}}$}  & \multicolumn{3}{c||}{$\mathbf{k = 1, B = 2\sqrt{d}}$} & \multicolumn{3}{c|}{$\mathbf{k = 1, B = 3\sqrt{d}}$}\\ \cline{2-10}
    \textbf{Estimator} & $\mathbf{d = 10}$  &  $\mathbf{d = 20}$ & $\mathbf{d = 30}$ & $\mathbf{d = 10}$ & $\mathbf{d = 20}$ & $\mathbf{d = 30}$ & $\mathbf{d = 10}$ & $\mathbf{d = 20}$ & $\mathbf{d = 30}$\\ \hline
    Standard Estimator & 1 & 1 & 1 & 1 & 1 & 1 & 1 & 1 & 1\\\hline
    James-Stein Estimator & 2.3796 & 4.9005 & 7.3489 & 2.5087 & 4.9375 & 7.3760 & 2.4288 & 4.8951 & 7.3847\\\hline
    Projection Estimator & 1.0055 & 1.4430 & 2.0424 & 1.0263 & 1.1051 & 1.5077 & 1.0288 & 1.0310 & 1.0202\\\hline
    Best Linear Estimator & 0.9091 & 0.9524 & 0.9677 & 0.9756 & 0.9877 & 0.9917 & 0.9890 & 0.9945 & 0.9963\\\hline
    \begin{tabular}{@{}c@{}}\textbf{Bayes Estimator for}\\\textbf{average prior}\end{tabular} & \textbf{0.7955} & \textbf{0.8565} & \textbf{0.8996} & \textbf{0.9160} & \textbf{0.9496} & 0.9726 & 0.9611 & 1.0007 & 1.0172\\\hline
\textbf{Averaged Estimator} & \textbf{0.7939} & \textbf{0.8579} & \textbf{0.8955} & \textbf{0.9104} & \textbf{0.9497} & 0.9724 & 0.9640 & 1.0003 & 1.0101\\\hline
    \end{tabular}
    }}
 
    \hspace{-0.1in}
  \resizebox{\columnwidth}{!}{\small{
  \begin{tabular}{|c||c|c|c||c|c|c||c|c|c|}
    \hline
    &\multicolumn{9}{c|}{\textbf{Worst-case Risk}}\\
    \cline{2-10}
     & \multicolumn{3}{c||}{$\mathbf{k = d/2, B = \sqrt{d}}$}  & \multicolumn{3}{c||}{$\mathbf{k = d/2, B = 2\sqrt{d}}$} & \multicolumn{3}{c|}{$\mathbf{k = d/2, B = 3\sqrt{d}}$}\\ \cline{2-10}
    \textbf{Estimator} & $\mathbf{d = 10}$  &  $\mathbf{d = 20}$ & $\mathbf{d = 30}$ & $\mathbf{d = 10}$ & $\mathbf{d = 20}$ & $\mathbf{d = 30}$ & $\mathbf{d = 10}$ & $\mathbf{d = 20}$ & $\mathbf{d = 30}$\\ \hline
    Standard Estimator & 5 & 10 & 15 & 5 & 10 & 15 & 5 & 10 & 15\\\hline
    James-Stein Estimator & 4.1167 & 7.9200 & 11.6892  & 5.0109 & 9.7551 & 14.6568 & 5.0281 & 10.0155 & 14.9390\\\hline
    Projection Estimator & 7.1096 & 15.8166 & 24.8158 & 30.3166 & 66.1806 & 103.0456 & 73.4834 & 156.5076 & 241.1031\\\hline
    \begin{tabular}{@{}c@{}}\textbf{Bayes Estimator for}\\\textbf{average prior}\end{tabular} & \textbf{3.2611} & \textbf{6.5834} & \textbf{9.8189} & \textbf{4.2477} & \textbf{8.6564} & \textbf{13.0606} & \textbf{4.6359} & \textbf{9.2773} & \textbf{13.9678}\\\hline
\textbf{Averaged Estimator} & \textbf{3.2008} & \textbf{6.4763} & \textbf{9.7763} & \textbf{4.2260} & \textbf{8.6421} & \textbf{13.0353} & \textbf{4.6413} & \textbf{9.2760} & \textbf{13.9446} \\\hline
    \end{tabular}
    }}

\end{table*}
In this section we again consider the finite Gaussian sequence model, but with a different risk. We now measure the risk on only the first $k$ coordinates: $M(\theta_1,\theta_2) = \sum_{i=1}^{k} (\theta_1(i) - \theta_2(i))^2$. We present experimental results for $k = 1, d/2$. 
\vspace{-0.15in}
\paragraph{Proposed Technique.} Following Theorem~\ref{thm:sequence_model_invariance_few}, the original min-max objective can be reduced to the simpler problem in Equation~\eqref{eqn:sequence_model_few_simplified}. We use similar optimization oracles as in Algorithms~\ref{alg:mean_estimation_max_oracle},~\ref{alg:mean_estimation_min_oracle}, to solve this problem. The maximization problem is now a 2D optimization problem for which we use grid search. The minimization problem, which requires computation of Bayes estimators, can be solved analytically and has similar expression as the Bayes estimator in Algorithm~\ref{alg:mean_estimation_min_oracle} (see Appendix~\ref{sec:sequence_model_few} for details). We use a 2D grid of $0.05B$ width and length in the maximization oracle. We use the same hyper-parameters as above and run FTPL for $10000$ iterations for $k=1$ and $4000$ iterations for $k=d/2$.  
\vspace{-0.12in}
\paragraph{Worst-case Risk.} We compare our estimators with the same baselines described in the previous section. For the case of $k=1$, we also compare with the best linear estimator, which is known to be approximately  minimax with worst case risk smaller than $1.25$ times the minimax risk~\citep{donoho1994statistical}. Since all these estimators, except the best linear estimator, are invariant to the transformations of group $\mathbb{O}(k)\times\mathbb{O}(d-k)$, the max risk of these estimators can be written as $\max_{b_1^2 + b_2^2 \leq B^2} R(\hat{\theta}, [b_1 \mathbf{e}_{1,k}, b_2 \mathbf{e}_{1,d-k}])$. We solve this problem using $2D$ grid search.  The worst case risk of best linear estimator has a closed form expression.
\vspace{-0.15in}
\paragraph{Results.} Table~\ref{tab:mean_few_coordinates} shows the performance of various estimators for various values of $d, B$. It can be seen that for $B=\sqrt{d}$, our estimators have better performance than other baselines. The performance difference goes down for large $B$, which is as expected.  In order to gain insights about the estimator learned by our algorithm, we plot the contours of $\hat{\theta}_{\AVG}(X)$ in Figure~\ref{fig:mean_estimation_est_contours}, for the $k=1$ case, where the risk is measured on the first coordinate. It can be seen that when $X(1)$ is close to $0$, irrespective of other coordinates, the estimator just outputs $X(1)$ as its estimate of $\theta(1)$. When $X(1)$ if far from $0$, by looking along the corresponding vertical line, the estimator can be seen as outputting a shrinked version of $X(1)$, where the amount of shrinkage increases with the norm of $X(2:d)$. Note that this is unlike James Stein estimator which shrinks vectors with smaller norm more than larger norm vectors. 
\begin{figure}[tbh]
\centering
\includegraphics[width=0.3\textwidth]{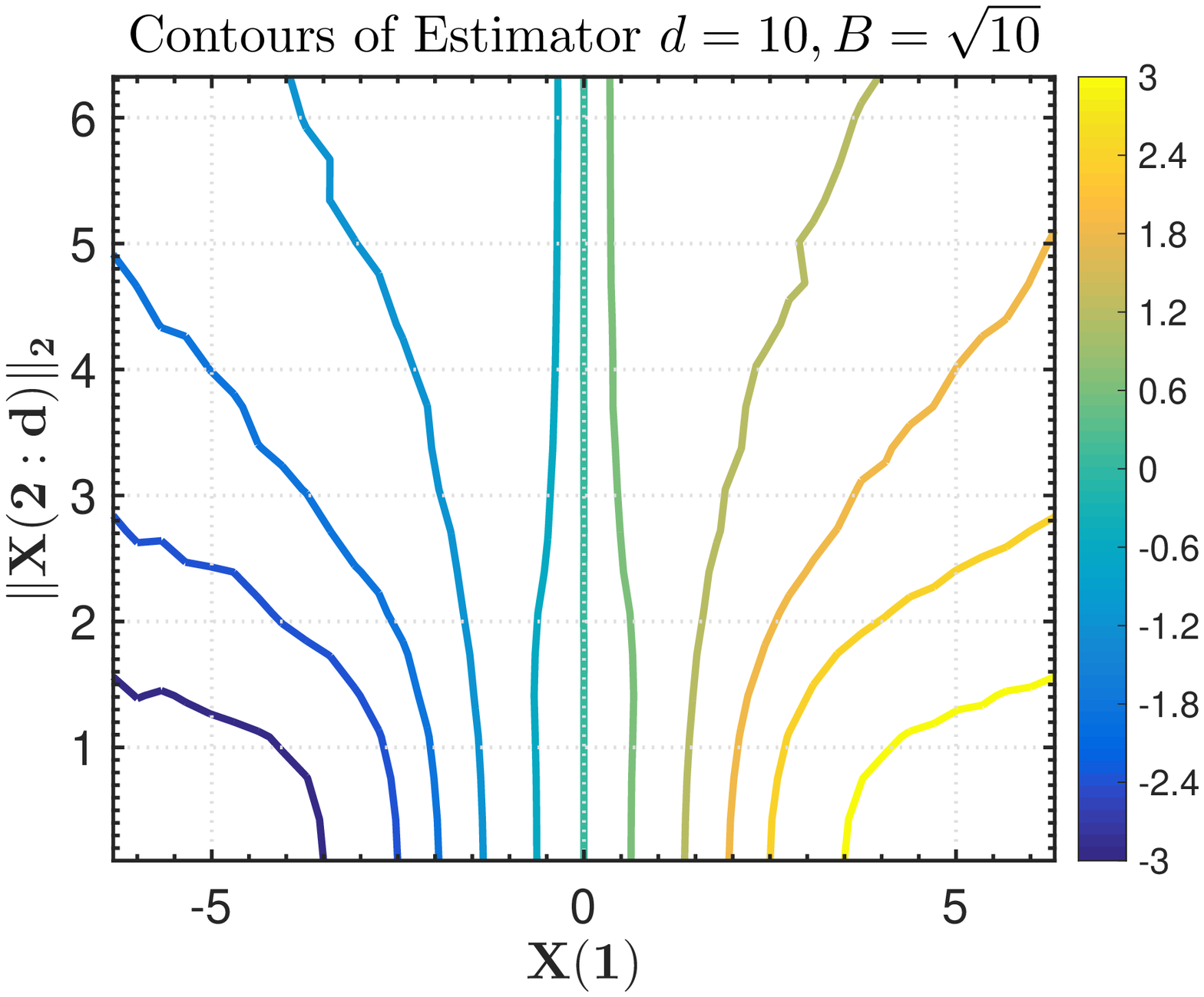}
\includegraphics[width=0.3\textwidth]{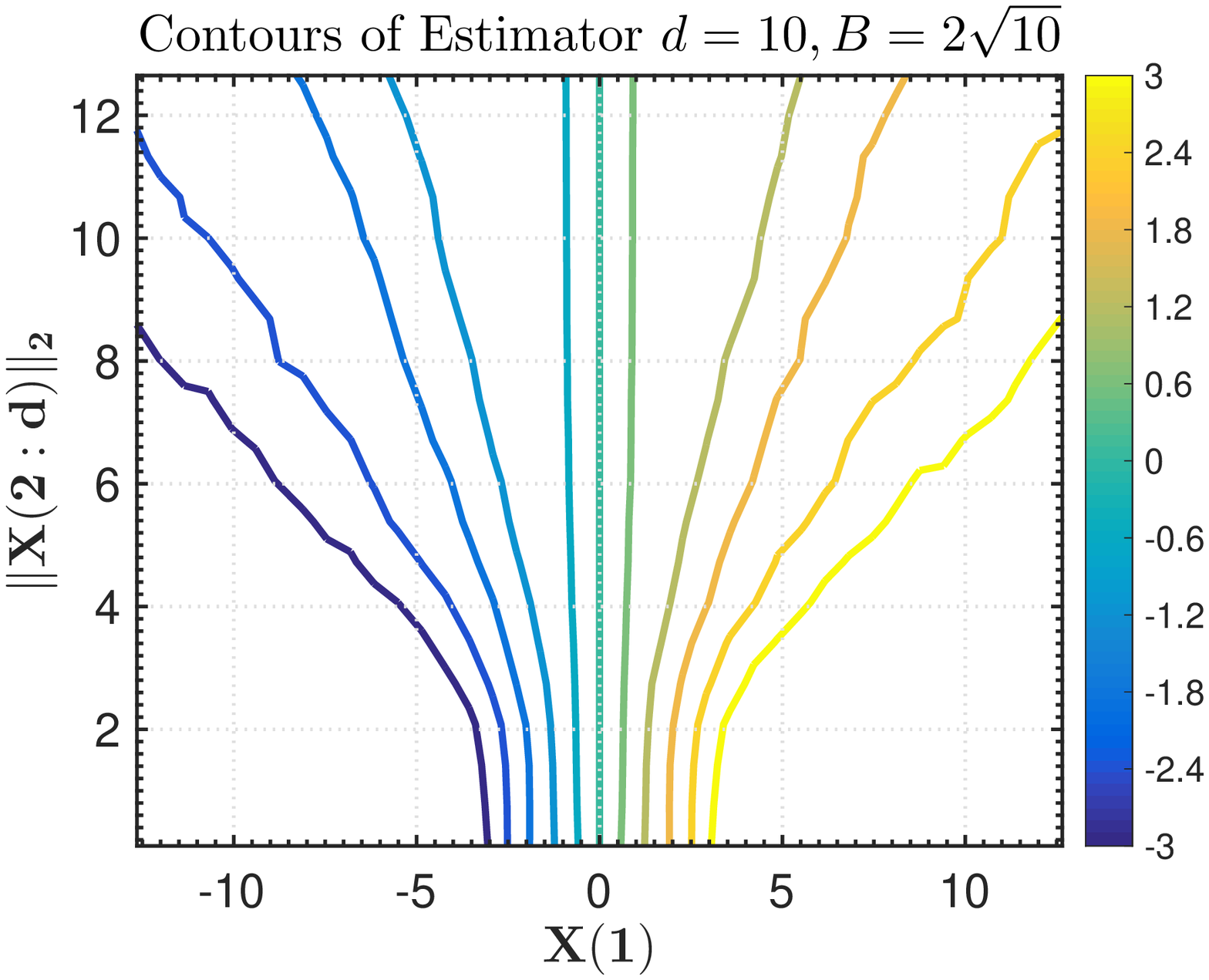}
\includegraphics[width=0.3\textwidth]{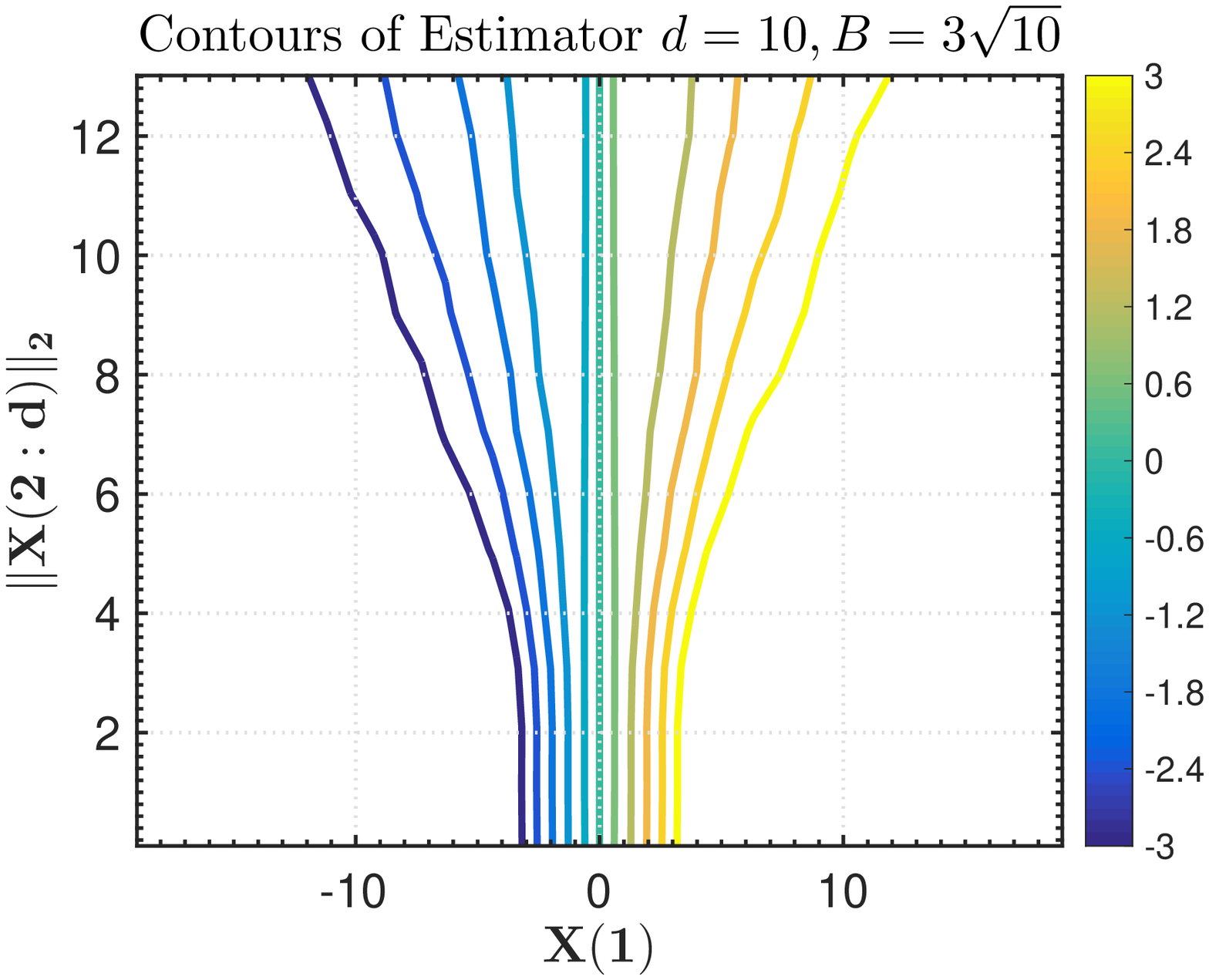}
\caption{\small{Contour plots of the estimator learned using Algorithm~\ref{alg:ftpl_stat_games} when the risk is evaluated on the first coordinate. $x$ axis shows the first coordinate of $X$, which is the input to the estimator. $y$ axis shows the norm of the rest of the coordinates of $X$. The contour bar shows $\hat{\theta}(1)$, the first co-ordinate of the output of the estimator. } }
\label{fig:mean_estimation_est_contours}
\vspace{-0.2in}
\end{figure}
\subsection{Linear Regression}
In this section we present experimental results on linear regression. We use Algorithm~\ref{alg:ftpl_stat_games} with optimization oracles described in Algorithms~\ref{alg:regression_max_oracle},~\ref{alg:regression_min_oracle} to find minimax estimators for this problem.
We use the same hyper-parameter settings as finite Gaussian sequence model, and run Algorithm~\ref{alg:ftpl_stat_games} for $T = 500$ iterations.  We compare the worst-case risk of minimax estimators obtained using our algorithm for various values of $(n,d, B)$, with ordinary least squares (OLS) and ridge regression estimators. Since all the estimators are invariant to the transformations of orthogonal group $\mathbb{O}(d)$, the max risk can be written as $\max_{b \in [0, B]} R(\hat{\theta}, b\mathbf{e}_1)$, which can be efficiently computed using grid search. Table~\ref{tab:linear_regression} presents the results from this experiment. It can be seen that we achieve  better performance than ridge regression for small values of $n/d$, $B$. For large values of $n/d$, $B$,  the performance of our estimator approaches ridge regression. The duality gap numbers presented in the Table suggest that the performance of our estimator can be improved for larger values of $n/d, B$, by choosing better hyper-parameters.

\vspace{-0.05in}
\begin{table*}[tbh]
  \caption{\small{Worst-case risk of various estimators for linear regression. The performance of ridge is obtained by choosing the best  regularization parameter. The numbers in the brackets for Averaged Estimator represent the duality gap.}}
  \label{tab:linear_regression}
  \hspace{-0.1in}
  \resizebox{\columnwidth}{!}{\small{
  \begin{tabular}{|c||c|c|c|c||c|c|c|c|}
    \hline
    &\multicolumn{8}{c|}{\textbf{Worst-case Risk}}\\
    \cline{2-9}
     & \multicolumn{4}{c||}{$\mathbf{n = 1.5\times d, B = 0.5\times \sqrt{d}}$}  & \multicolumn{4}{c|}{$\mathbf{n = 1.5\times d, B = \sqrt{d}}$}                      \\ \cline{2-9}
    \textbf{Estimator} & $\mathbf{d = 5}$  &  $\mathbf{d = 10}$ & $\mathbf{d = 15}$&$\mathbf{d = 20}$& $\mathbf{d = 5}$&   $\mathbf{d = 10}$   & $\mathbf{d = 15}$&$\mathbf{d = 20}$\\ \hline
    OLS&5.0000&2.5000&2.5000&2.2222&5.0000&2.5000&2.5000&2.2222\\\hline
    Ridge regression&0.6637&0.9048&1.1288&1.1926&1.3021&1.4837&1.6912 & 1.6704\\\hline
    \begin{tabular}{@{}c@{}}\textbf{Averaged} \\ \textbf{Estimator}\end{tabular}&\makecell{\textbf{0.5827}\\(0.0003)} &  \makecell{\textbf{0.8275}\\(0.0052)} & \makecell{\textbf{0.9839}\\(0.0187)} & \makecell{\textbf{1.0946}\\(0.0404)} & \makecell{\textbf{1.2030}\\(0.0981)} & \makecell{1.4615\\(0.1145)} & \makecell{1.6178\\(0.1768)} & \makecell{1.6593\\(0.1863)}\\\hline
    \begin{tabular}{@{}c@{}}\textbf{Bayes estimator} \\ \textbf{for avg. prior}\end{tabular} &\textbf{0.5827} & \textbf{0.8275} & \textbf{0.9844} & \textbf{1.0961} & \textbf{1.1750} & 1.4621 & 1.6265 & 1.6674\\\hline
    \end{tabular}
    }}
    
    \hspace{-0.1in}
  \resizebox{\columnwidth}{!}{\small{
  \begin{tabular}{|c||c|c|c|c||c|c|c|c|}
    \hline
    &\multicolumn{8}{c|}{\textbf{Worst-case Risk}}\\
    \cline{2-9}
     & \multicolumn{4}{c||}{$\mathbf{n = 2\times d, B = 0.5\times \sqrt{d}}$}    &   \multicolumn{4}{c|}{$\mathbf{n = 2\times d, B = \sqrt{d}}$}    \\ \cline{2-9}
    \textbf{Estimator} & $\mathbf{d = 5}$ &   $\mathbf{d = 10}$ & $\mathbf{d = 15}$&$\mathbf{d = 20}$ & $\mathbf{d = 5}$ &   $\mathbf{d = 10}$ & $\mathbf{d = 15}$&$\mathbf{d = 20}$\\ \hline
    OLS&1.2500&1.1111&1.0714&1.053&1.2500&1.1111&1.0714&1.053\\\hline
    Ridge regression&0.5225&0.6683&0.7594&0.8080&0.8166&0.8917&0.9305 & 0.9608\\\hline
    \begin{tabular}{@{}c@{}}\textbf{Averaged} \\ \textbf{Estimator}\end{tabular}&\makecell{0.4920\\(0.0038)} &  \makecell{\textbf{0.5991}\\(0.0309)} & \makecell{\textbf{0.6873}\\(0.0485)} & \makecell{\textbf{0.7339}\\(0.0428)} & \makecell{0.8044\\(0.0647)} &  \makecell{0.8615\\(0.0854)} & \makecell{0.9388\\(0.0996)} & \makecell{0.9621\\(0.1224)}\\\hline
    \begin{tabular}{@{}c@{}}\textbf{Bayes estimator} \\ \textbf{for avg. prior}\end{tabular} &0.4894  & \textbf{0.6004} & \textbf{0.6879} & \textbf{0.7320} & 0.8140 &  0.8618 & 0.9375 & 0.9656\\\hline
    \end{tabular}
    }}
\vspace{-0.2in}
\end{table*}

\subsection{Covariance Estimation}
\label{sec:covariance_estimation}

In this section we present experimental results on normal covariance estimation.

\vspace{-0.15in}
\paragraph{Minimization oracle.}
In our experiments we use neural networks, which are universal function approximators, to parameterize functions $f,g$ in Equation~\eqref{eqn:covariance_reduced2}. To be precise, we use two layer neural networks to parameterize each of these functions. Implementing the minimization oracle then boils down to finding the parameters of these networks which minimize $ \Eover{\lambda \sim P_t}{R(\hat{\Sigma}_{f,g}, \text{Diag}(\lambda))}$. In our experiments, we use stochastic gradient descent to learn these parameters.

\vspace{-0.15in}
\paragraph{Baselines.} We compare the performance of the estimators returned by Algorithm~\ref{alg:ftpl_stat_games} for various values of $(n,d,B)$, with empirical covariance $S_n$ and the James Stein estimator~\citep{james1992estimation} which is defined as
$
K_n\Delta_{JS} K_n^T,
$
where $K_n$ is a lower triangular matrix such that $S_n = K_nK_n^T$ and $\Delta_{JS}$ is a diagonal matrix with $i^{th}$ diagonal element equal to $\frac{1}{n+d-2i+1}$.

\vspace{-0.15in}
\paragraph{Results.}
We use worst-case risk to compare the performance of various estimators. To compute the worst-case risk, we again rely on DragonFly. We note that the worst-case computed using this approach may be inaccurate as DragonFly is not guaranteed to return a global optimum. So, we also compare the risk of various estimators at randomly generated $\Sigma$'s (see Appendix~\ref{sec:exps_appendix}). Table~\ref{tab:covariance} presents the results from this experiment. It can be seen that our estimators outperform empirical covariance for almost all the values of $n,d,B$ and outperform James Stein estimator for small values of $n/d$, $B$. For large values of $n/d$, $B$, our estimator has similar performance as JS. In this setting, we believe the performance of our estimators can be improved by running the algorithm with better hyper-parameters.

\vspace{-0.1in}
\begin{table*}[ht]
  \caption{\small{Worst-case risk of various estimators for covariance estimation for various configurations of $(n,d,B)$. The worst-case risks are obtained by taking a max of the worst-case risk estimate from DragonFly and the risks computed at randomly generated $\Sigma$'s.}}
  \label{tab:covariance}
  
  \hspace{-0.1in}
  \resizebox{\columnwidth}{!}{\small{
  \begin{tabular}{|c||c|c||c|c||c|c||c|c|}
    \hline
    &\multicolumn{8}{c|}{\textbf{Worst-case Risk}}\\
    \cline{2-9}
     & \multicolumn{2}{c||}{$\mathbf{n = 1.5\times d, B = 1}$}  & \multicolumn{2}{c||}{$\mathbf{n = 1.5\times d, B = 2}$} & \multicolumn{2}{c||}{$\mathbf{n = 1.5\times d, B = 4}$} & \multicolumn{2}{c|}{$\mathbf{n = 1.5\times d, B = 8}$}\\ \cline{2-9}
    \textbf{Estimator} & $\mathbf{d = 5}$  &  $\mathbf{d = 10}$ & $\mathbf{d = 5}$&$\mathbf{d = 10}$& $\mathbf{d = 5}$&   $\mathbf{d = 10}$   & $\mathbf{d = 5}$&$\mathbf{d = 10}$\\ \hline
    Empirical Covariance& 2.5245 & 5.1095 & 2.5245 & 5.1095 & 2.5245 & 5.1095 & 2.5245 & 5.1095\\\hline
    James-Stein Estimator& 2.1637 & 4.1704 & 2.1637 & 4.1704 & 2.1637 & 4.1704 & 2.1637 & 4.1704\\\hline
    \textbf{Averaged Estimator} & \textbf{1.8686} & \textbf{3.1910} & \textbf{1.9371} & \textbf{3.7019} & 2.0827 & 4.2454 & 2.1416 & \textbf{3.9864} \\\hline
    \end{tabular}
    }}
 
    \hspace{-0.1in}
  \resizebox{\columnwidth}{!}{\small{
  \begin{tabular}{|c||c|c||c|c||c|c||c|c|}
    \hline
    &\multicolumn{8}{c|}{\textbf{Worst-case Risk}}\\
    \cline{2-9}
     & \multicolumn{2}{c||}{$\mathbf{n = 2\times d, B = 1}$}  & \multicolumn{2}{c||}{$\mathbf{n = 2\times d, B = 2}$} & \multicolumn{2}{c||}{$\mathbf{n = 2\times d, B = 4}$} & \multicolumn{2}{c|}{$\mathbf{n = 2\times d, B = 8}$}\\ \cline{2-9}
    \textbf{Estimator} & $\mathbf{d = 5}$  &  $\mathbf{d = 10}$ & $\mathbf{d = 5}$&$\mathbf{d = 10}$& $\mathbf{d = 5}$&   $\mathbf{d = 10}$   & $\mathbf{d = 5}$&$\mathbf{d = 10}$\\ \hline
    Empirical Covariance& 1.8714 & 3.4550 & 1.8714 & 3.4550 & 1.8714 & 3.4550 & 1.8714 & 3.4550\\\hline
    James-Stein Estimator& 1.6686 & 2.9433 & 1.6686 & 2.9433 & 1.6686 & 2.9433 & 1.6686 & 2.9433\\\hline
    \textbf{Averaged Estimator}& \textbf{1.2330} & \textbf{2.1944} & 1.5237 & \textbf{2.6471} & 1.6050 & 3.0834 & 1.6500 & 2.9907\\\hline
    \end{tabular}
    }}
    
    \hspace{-0.1in}
    \resizebox{\columnwidth}{!}{\small{
  \begin{tabular}{|c||c|c||c|c||c|c||c|c|}
    \hline
    &\multicolumn{8}{c|}{\textbf{Worst-case Risk}}\\
    \cline{2-9}
     & \multicolumn{2}{c||}{$\mathbf{n = 3\times d, B = 1}$}  & \multicolumn{2}{c||}{$\mathbf{n = 3\times d, B = 2}$} & \multicolumn{2}{c||}{$\mathbf{n = 3\times d, B = 4}$} & \multicolumn{2}{c|}{$\mathbf{n = 3\times d, B = 8}$}\\ \cline{2-9}
    \textbf{Estimator} & $\mathbf{d = 5}$  &  $\mathbf{d = 10}$ & $\mathbf{d = 5}$&$\mathbf{d = 10}$& $\mathbf{d = 5}$&   $\mathbf{d = 10}$   & $\mathbf{d = 5}$&$\mathbf{d = 10}$\\ \hline
    Empirical Covariance& 1.1425 & 2.1224 & 1.1425 & 2.1224 & 1.1425 & 2.1224 & 1.1425 & 2.1224\\\hline
    James-Stein Estimator& 1.0487 & 1.9068 & 1.0487 & 1.9068 & 1.0487 & 1.9068 & \textbf{1.0487} & 1.9068\\\hline
    \textbf{Averaged Estimator} & \textbf{0.8579} & \textbf{1.3731} & 0.9557 & \textbf{1.7151} & 1.0879 & 1.9174 & 1.2266 & 2.0017\\\hline
    \end{tabular}
    }}

\end{table*}

\vspace{-0.15in}
\subsection{Entropy Estimation}
\label{sec:entropy_estimation}
In this section, we consider the problem of entropy estimation described in Section~\ref{sec:invariance_entropy}. Similar to covariance estimation, we use two layer neural networks to parameterize functions $g,h$ in Equation~\eqref{eqn:entropy_reduced2}. Implementing the minimization oracle then boils down to finding the parameters of these networks which minimize $\Eover{P \sim P_t}{R(\hat{f}_{g,h}, P)}$. We use stochastic gradient descent to solve this optimization problem.
\vspace{-0.15in}
\paragraph{Baselines.} We compare the performance of the estimators returned by Algorithm~\ref{alg:ftpl_stat_games} for various values of $(n,d),$ with the plugin MLE estimator $ -\sum_{i=1}^d\hat{p}_i\log{\hat{p}_i}$, and the minimax rate optimal estimator of~\citet{jiao2015minimax} (JVHW). The plugin estimator is known to be sub-optimal in the high dimensional regime, where $n < d$~\citep{jiao2015minimax}.
\vspace{-0.15in}
\paragraph{Results.} We compare the performance of various estimators based on their worst-case risk computed using DragonFly. Since DragonFly is not guaranteed to compute the worst-case risk, we also compare the estimators based on their risk at randomly generated distributions (see Appendix~\ref{sec:exps_appendix}).  Table~\ref{tab:entropy_estimation} presents the worst-case risk numbers. It can be seen that the plugin MLE estimator has a poor performance compared to JVHW and our estimator. Our estimator has similar performance as JVHW, which is the best known minimax estimator for entropy estimation. We believe the performance of our estimator can be improved with better hyper-parameters.

\vspace{-0.15in}
\begin{table*}[ht]
  \caption{\small{Worst-case risk of various estimators for entropy estimation, for various values of $(n,d)$. The worst-case risks are obtained by taking a max of the worst-case risk estimate from DragonFly and the risks computed at randomly generated distributions.}}
  \label{tab:entropy_estimation}
  \centering
   \resizebox{\textwidth}{!}{
\setlength\tabcolsep{4pt}
  \begin{tabular}{|p{0.75in}||p{0.45in}|p{0.45in}||p{0.45in}|p{0.45in}||p{0.45in}|p{0.45in}|p{0.45in}||p{0.45in}|p{0.45in}|p{0.45in}|p{0.45in}||}
    \hline
    &\multicolumn{10}{c|}{\textbf{Worst-case Risk}}\\
    \cline{2-11}
     & \multicolumn{2}{c||}{$\mathbf{d=10}$}  & \multicolumn{2}{c||}{$\mathbf{d=20}$} & \multicolumn{3}{c||}{$\mathbf{d=40}$} & \multicolumn{3}{c|}{$\mathbf{d=80}$}\\
     \cline{2-11}
     \textbf{Estimator} & $\mathbf{n=10}$ & $\mathbf{n=20}$ & $\mathbf{n=20}$ & $\mathbf{n=40}$ & $\mathbf{n=10}$ & $\mathbf{n=20}$ & $\mathbf{n=40}$ & $\mathbf{n=20}$ & $\mathbf{n=40}$ & $\mathbf{n=80}$\\
     \hline
    
    \textit{Plugin MLE} & 0.2895 & 0.1178 & 0.2512 & 0.0347 & 2.1613 & 0.8909 & 0.2710 & 2.2424 & 0.9142 & 0.2899\\
    \hline
     \textit{JVHW}~\citep{jiao2015minimax} & 0.3222 & 0.0797 & \textbf{0.1322} & 0.0489 & 0.6788 & 0.2699 & 0.0648 & \textbf{0.3751} & \textbf{0.1755} & 0.0974\\
    \hline
    \begin{tabular}{@{}c@{}}\textit{Averaged} \\\textit{Estimator}\end{tabular} & \textbf{0.1382} & 0.0723 & 0.1680 & 0.0439 & \textbf{0.5392} & \textbf{0.2320} & 0.0822 & 0.5084 & 0.2539 & 0.0672\\
    \hline
\end{tabular}
     }
\end{table*}

\vspace{-0.2in}
\section{Discussion}
\label{sec:conclusion}
In this paper, we introduced an algorithmic approach for constructing minimax estimators, where we attempt to directly solve the min-max statistical game associated with the estimation problem. This is unlike the traditional approach in statistics, where an estimator is first proposed and then its minimax optimality is certified by showing its worst-case risk matches the known lower bounds for the minimax risk. Our algorithm relies on techniques from online non-convex learning for solving the statistical game and requires access to certain optimization subroutines. Given access to these subroutines, our algorithm returns a minimax estimator and a least favorable prior. This reduces the problem of designing minimax estimators to a purely computational question of efficient implementation of these subroutines. While implementing these subroutines is computationally expensive in the worst case, we showed that one can rely on the structure of the problem to reduce their computational complexity. 
For the well studied problems of finite Gaussian sequence model and linear regression, we showed that our approach can be used to learn provably minimax estimators in $\text{poly}(d)$ time. For problems where provable implementation of the optimization subroutines is computationally expensive, we demonstrated that our framework can still be used together with heuristics to obtain estimators with better performance than existing (up to constant-factor) minimax estimators. We empirically demonstrated this on classical problems such as covariance and entropy estimation.
We believe our approach could be especially useful in high-dimensional settings where classical estimators are sub-optimal and not much is known about minimax estimators. In such settings, our approach can provide insights into least favourable priors and aid  statisticians in designing minimax estimators.

There are several avenues for future work. The most salient is a more comprehensive understanding of settings where the optimization subroutines can be efficiently implemented. In this work, we have mostly relied on invariance properties of statistical games to implement these subroutines. As described in Section~\ref{sec:games}, there are several other forms of problem structure that can be exploited to implement these subroutines. Exploring these directions can help us construct minimax estimators for several other estimation problems. Another direction for future work would be to modify our algorithm to learn an approximate minimax estimator (\emph{i.e.,} a rate optimal estimator), instead of an exact minimax estimator. There are several reasons why switching to approximate rather than exact minimaxity can be advantageous. First, with respect to our risk tolerance, it may suffice to construct an estimator whose worst-case risk is   constant factors worse than the minimax risk. Second, by switching to approximate minimaxity, we believe one can design algorithms requiring significantly weaker optimization subroutines than those required by our current algorithm. Third, the resulting algorithms might be less tailored or over-fit to the specific statistical model assumptions, so that the resulting algorithms will be much more broadly applicable. Towards the last point, we note that our minimax estimators could always be embedded within a model selection sub-routine, so that for any given data-set, one could select from a suite of minimax estimators using standard model selection criteria. Finally, it would be of interest to modify our algorithm to output a single estimator which is simultaneously minimax for various values of $n$, the number of observations.  
\small{
\bibliographystyle{unsrtnat}
\bibliography{main}{}

\begin{thebibliography}{55}
\providecommand{\natexlab}[1]{#1}
\providecommand{\url}[1]{\texttt{#1}}
\expandafter\ifx\csname urlstyle\endcsname\relax
  \providecommand{\doi}[1]{doi: #1}\else
  \providecommand{\doi}{doi: \begingroup \urlstyle{rm}\Url}\fi

\bibitem[Ferguson(2014)]{ferguson2014mathematical}
Thomas~S Ferguson.
\newblock \emph{Mathematical statistics: A decision theoretic approach},
  volume~1.
\newblock Academic press, 2014.

\bibitem[Lehmann and Casella(2006)]{lehmann2006theory}
Erich~L Lehmann and George Casella.
\newblock \emph{Theory of point estimation}.
\newblock Springer Science \& Business Media, 2006.

\bibitem[Tsybakov(2008)]{Tsybakov}
Alexandre~B. Tsybakov.
\newblock \emph{Introduction to Nonparametric Estimation}.
\newblock Springer Publishing Company, Incorporated, 1st edition, 2008.
\newblock ISBN 0387790519, 9780387790510.

\bibitem[Ibragimov and Has'minskii(1981)]{Ibragimov81book}
I.~A. Ibragimov and R.~Z. Has'minskii.
\newblock \emph{Statistical Estimation: Asymptotic Theory}.
\newblock New York, 1981.

\bibitem[Le~Cam(2012)]{le2012asymptotic}
Lucien Le~Cam.
\newblock \emph{Asymptotic methods in statistical decision theory}.
\newblock Springer Science \& Business Media, 2012.

\bibitem[Vaart(1998)]{vaart_1998}
A.~W. van~der Vaart.
\newblock \emph{Asymptotic Statistics}.
\newblock Cambridge Series in Statistical and Probabilistic Mathematics.
  Cambridge University Press, 1998.
\newblock \doi{10.1017/CBO9780511802256}.

\bibitem[Birg{\'e}(1983)]{birge1983approximation}
Lucien Birg{\'e}.
\newblock Approximation dans les espaces m{\'e}triques et th{\'e}orie de
  l'estimation.
\newblock \emph{Zeitschrift f{\"u}r Wahrscheinlichkeitstheorie und verwandte
  Gebiete}, 65\penalty0 (2):\penalty0 181--237, 1983.

\bibitem[Birg{\'e} and Massart(1993)]{birge1993rates}
Lucien Birg{\'e} and Pascal Massart.
\newblock Rates of convergence for minimum contrast estimators.
\newblock \emph{Probability Theory and Related Fields}, 97\penalty0
  (1-2):\penalty0 113--150, 1993.

\bibitem[Yang and Barron(1999)]{yang1999information}
Yuhong Yang and Andrew Barron.
\newblock Information-theoretic determination of minimax rates of convergence.
\newblock \emph{Annals of Statistics}, pages 1564--1599, 1999.

\bibitem[Wellner(2015)]{wellner}
Jon~A Wellner.
\newblock Maximum likelihood in modern times: the ugly, the bad, and the good.
\newblock 2015.
\newblock URL
  \url{https://www.stat.washington.edu/jaw/RESEARCH/TALKS/LeCam-v2.pdf}.

\bibitem[Jiao et~al.(2015)Jiao, Venkat, Han, and Weissman]{jiao2015minimax}
Jiantao Jiao, Kartik Venkat, Yanjun Han, and Tsachy Weissman.
\newblock Minimax estimation of functionals of discrete distributions.
\newblock \emph{IEEE Transactions on Information Theory}, 61\penalty0
  (5):\penalty0 2835--2885, 2015.

\bibitem[Cai and Low(2011)]{cai2011testing}
T~Tony Cai and Mark~G Low.
\newblock Testing composite hypotheses, hermite polynomials and optimal
  estimation of a nonsmooth functional.
\newblock \emph{The Annals of Statistics}, 39\penalty0 (2):\penalty0
  1012--1041, 2011.

\bibitem[{Valiant} and {Valiant}(2011)]{valiant}
G.~{Valiant} and P.~{Valiant}.
\newblock The power of linear estimators.
\newblock In \emph{2011 IEEE 52nd Annual Symposium on Foundations of Computer
  Science}, pages 403--412, Oct 2011.
\newblock \doi{10.1109/FOCS.2011.81}.

\bibitem[Butucea et~al.(2018)Butucea, Ndaoud, Stepanova, and
  Tsybakov]{butucea2018variable}
Cristina Butucea, Mohamed Ndaoud, Natalia~A Stepanova, and Alexandre~B
  Tsybakov.
\newblock Variable selection with hamming loss.
\newblock \emph{The Annals of Statistics}, 46\penalty0 (5):\penalty0
  1837--1875, 2018.

\bibitem[Wu and Yang(2016)]{wu2016minimax}
Yihong Wu and Pengkun Yang.
\newblock Minimax rates of entropy estimation on large alphabets via best
  polynomial approximation.
\newblock \emph{IEEE Transactions on Information Theory}, 62\penalty0
  (6):\penalty0 3702--3720, 2016.

\bibitem[Berger(1985)]{berger1985statistical}
James~O Berger.
\newblock \emph{Statistical Decision Theory and Bayesian Analysis}.
\newblock Springer Science \& Business Media, 1985.

\bibitem[Freund and Schapire(1996)]{freund1996game}
Yoav Freund and Robert~E Schapire.
\newblock Game theory, on-line prediction and boosting.
\newblock In \emph{COLT}, volume~96, pages 325--332. Citeseer, 1996.

\bibitem[Chen et~al.(2017)Chen, Lucier, Singer, and Syrgkanis]{chen2017robust}
Robert~S Chen, Brendan Lucier, Yaron Singer, and Vasilis Syrgkanis.
\newblock Robust optimization for non-convex objectives.
\newblock In \emph{Advances in Neural Information Processing Systems}, pages
  4705--4714, 2017.

\bibitem[Feige et~al.(2015)Feige, Mansour, and Schapire]{feige2015learning}
Uriel Feige, Yishay Mansour, and Robert Schapire.
\newblock Learning and inference in the presence of corrupted inputs.
\newblock In \emph{Conference on Learning Theory}, pages 637--657, 2015.

\bibitem[Johnstone(2002)]{johnstone2002function}
Iain~M Johnstone.
\newblock Function estimation and gaussian sequence models.
\newblock \emph{Unpublished manuscript}, 2\penalty0 (5.3), 2002.

\bibitem[Bickel(1981)]{bickel1981minimax}
PJ~Bickel.
\newblock Minimax estimation of the mean of a normal distribution when the
  parameter space is restricted.
\newblock \emph{The Annals of Statistics}, 9\penalty0 (6):\penalty0 1301--1309,
  1981.

\bibitem[Berry(1990)]{berry1990minimax}
J~Calvin Berry.
\newblock Minimax estimation of a bounded normal mean vector.
\newblock \emph{Journal of Multivariate Analysis}, 35\penalty0 (1):\penalty0
  130--139, 1990.

\bibitem[Marchand and Perron(2002)]{marchand2002minimax}
{\'E}ric Marchand and Fran{\c{c}}ois Perron.
\newblock On the minimax estimator of a bounded normal mean.
\newblock \emph{Statistics \& probability letters}, 58\penalty0 (4):\penalty0
  327--333, 2002.

\bibitem[Wald(1949)]{wald1949statistical}
Abraham Wald.
\newblock Statistical decision functions.
\newblock \emph{The Annals of Mathematical Statistics}, pages 165--205, 1949.

\bibitem[Von~Neumann et~al.(2007)Von~Neumann, Morgenstern, and
  Kuhn]{von2007theory}
John Von~Neumann, Oskar Morgenstern, and Harold~William Kuhn.
\newblock \emph{Theory of games and economic behavior (commemorative edition)}.
\newblock Princeton university press, 2007.

\bibitem[Yanovskaya(1974)]{yanovskaya1974infinite}
EB~Yanovskaya.
\newblock Infinite zero-sum two-person games.
\newblock \emph{Journal of Soviet Mathematics}, 2\penalty0 (5):\penalty0
  520--541, 1974.

\bibitem[Hald(1971)]{hald1971size}
Anders Hald.
\newblock The size of bayes and minimax tests as function of the sample size
  and the loss ratio.
\newblock \emph{Scandinavian Actuarial Journal}, 1971\penalty0 (1-2):\penalty0
  53--73, 1971.

\bibitem[Clarke and Barron(1994)]{clarke1994jeffreys}
Bertrand~S Clarke and Andrew~R Barron.
\newblock Jeffreys' prior is asymptotically least favorable under entropy risk.
\newblock \emph{Journal of Statistical planning and Inference}, 41\penalty0
  (1):\penalty0 37--60, 1994.

\bibitem[Ghosh(1964)]{ghosh1964uniform}
MN~Ghosh.
\newblock Uniform approximation of minimax point estimates.
\newblock \emph{The Annals of Mathematical Statistics}, pages 1031--1047, 1964.

\bibitem[Nelson(1966)]{nelson1966minimax}
Wayne Nelson.
\newblock Minimax solution of statistical decision problems by iteration.
\newblock \emph{The Annals of Mathematical Statistics}, pages 1643--1657, 1966.

\bibitem[Kempthorne(1987)]{kempthorne1987numerical}
Peter~J Kempthorne.
\newblock Numerical specification of discrete least favorable prior
  distributions.
\newblock \emph{SIAM Journal on Scientific and Statistical Computing},
  8\penalty0 (2):\penalty0 171--184, 1987.

\bibitem[Luedtke et~al.(2020)Luedtke, Carone, Simon, and
  Sofrygin]{Luedtkeeaaw2140}
Alex Luedtke, Marco Carone, Noah Simon, and Oleg Sofrygin.
\newblock Learning to learn from data: Using deep adversarial learning to
  construct optimal statistical procedures.
\newblock \emph{Science Advances}, 6\penalty0 (9), 2020.
\newblock \doi{10.1126/sciadv.aaw2140}.
\newblock URL \url{https://advances.sciencemag.org/content/6/9/eaaw2140}.

\bibitem[Hazan(2016)]{hazan2016introduction}
Elad Hazan.
\newblock Introduction to online convex optimization.
\newblock \emph{Foundations and Trends{\textregistered} in Optimization},
  2\penalty0 (3-4):\penalty0 157--325, 2016.

\bibitem[McMahan(2017)]{mcmahan2017survey}
H~Brendan McMahan.
\newblock A survey of algorithms and analysis for adaptive online learning.
\newblock \emph{The Journal of Machine Learning Research}, 18\penalty0
  (1):\penalty0 3117--3166, 2017.

\bibitem[Kalai and Vempala(2005)]{kalai2005efficient}
Adam Kalai and Santosh Vempala.
\newblock Efficient algorithms for online decision problems.
\newblock \emph{Journal of Computer and System Sciences}, 71\penalty0
  (3):\penalty0 291--307, 2005.

\bibitem[Suggala and Netrapalli(2019)]{ftpl_nonconvex}
Arun~Sai Suggala and Praneeth Netrapalli.
\newblock Online non-convex learning: Following the perturbed leader is
  optimal.
\newblock \emph{CoRR}, abs/1903.08110, 2019.
\newblock URL \url{http://arxiv.org/abs/1903.08110}.

\bibitem[Cesa-Bianchi and Lugosi(2006)]{cesa2006prediction}
Nicolo Cesa-Bianchi and Gabor Lugosi.
\newblock \emph{Prediction, learning, and games}.
\newblock Cambridge university press, 2006.

\bibitem[Krichene et~al.(2015)Krichene, Balandat, Tomlin, and
  Bayen]{krichene2015hedge}
Walid Krichene, Maximilian Balandat, Claire Tomlin, and Alexandre Bayen.
\newblock The hedge algorithm on a continuum.
\newblock In \emph{International Conference on Machine Learning}, pages
  824--832, 2015.

\bibitem[Agarwal et~al.(2018)Agarwal, Gonen, and Hazan]{gonen2018learning}
Naman Agarwal, Alon Gonen, and Elad Hazan.
\newblock Learning in non-convex games with an optimization oracle.
\newblock \emph{arXiv preprint arXiv:1810.07362}, 2018.

\bibitem[Donoho et~al.(1990)Donoho, Liu, and MacGibbon]{donoho1990minimax}
David~L Donoho, Richard~C Liu, and Brenda MacGibbon.
\newblock Minimax risk over hyperrectangles, and implications.
\newblock \emph{The Annals of Statistics}, pages 1416--1437, 1990.

\bibitem[Polyanskiy and Wu(2019)]{polyanskiy2019dualizing}
Yury Polyanskiy and Yihong Wu.
\newblock Dualizing le cam's method, with applications to estimating the
  unseens.
\newblock \emph{arXiv preprint arXiv:1902.05616}, 2019.

\bibitem[Hartigan(1983)]{hartigan1983asymptotic}
JA~Hartigan.
\newblock Asymptotic normality of posterior distributions.
\newblock In \emph{Bayes theory}, pages 107--118. Springer, 1983.

\bibitem[Kiefer(1957)]{kiefer1957invariance}
J~Kiefer.
\newblock \emph{The Annals of Mathematical Statistics}, 28\penalty0
  (3):\penalty0 573--601, 1957.

\bibitem[Johnstone(2011)]{johnstone2011gaussian}
Iain~M Johnstone.
\newblock Gaussian estimation: Sequence and wavelet models.
\newblock \emph{Unpublished manuscript}, 2011.

\bibitem[Casella and Strawderman(1981)]{casella1981estimating}
George Casella and William~E Strawderman.
\newblock Estimating a bounded normal mean.
\newblock \emph{The Annals of Statistics}, pages 870--878, 1981.

\bibitem[Kume and Wood(2005)]{kume2005saddlepoint}
Alfred Kume and Andrew~TA Wood.
\newblock Saddlepoint approximations for the bingham and fisher--bingham
  normalising constants.
\newblock \emph{Biometrika}, 92\penalty0 (2):\penalty0 465--476, 2005.

\bibitem[Imhof(1961)]{imhof1961computing}
Jean-Pierre Imhof.
\newblock Computing the distribution of quadratic forms in normal variables.
\newblock \emph{Biometrika}, 48\penalty0 (3/4):\penalty0 419--426, 1961.

\bibitem[Kandasamy et~al.(2019)Kandasamy, Vysyaraju, Neiswanger, Paria,
  Collins, Schneider, Poczos, and Xing]{kandasamy2019tuning}
Kirthevasan Kandasamy, Karun~Raju Vysyaraju, Willie Neiswanger, Biswajit Paria,
  Christopher~R. Collins, Jeff Schneider, Barnabas Poczos, and Eric~P Xing.
\newblock {Tuning Hyperparameters without Grad Students: Scalable and Robust
  Bayesian Optimisation with Dragonfly}.
\newblock \emph{arXiv preprint arXiv:1903.06694}, 2019.

\bibitem[Donoho(1994)]{donoho1994statistical}
David~L Donoho.
\newblock \emph{The Annals of Statistics}, 22\penalty0 (1):\penalty0 238--270,
  1994.

\bibitem[James and Stein(1992)]{james1992estimation}
William James and Charles Stein.
\newblock Estimation with quadratic loss.
\newblock In \emph{Breakthroughs in statistics}, pages 443--460. Springer,
  1992.

\bibitem[Brown and Purves(1973)]{brown1973measurable}
Lawrence~D Brown and R~Purves.
\newblock Measurable selections of extrema.
\newblock \emph{The annals of statistics}, 1\penalty0 (5):\penalty0 902--912,
  1973.

\bibitem[Wijsman(1990)]{wijsman1990invariant}
Robert~A Wijsman.
\newblock Invariant measures on groups and their use in statistics.
\newblock IMS, 1990.

\bibitem[Mardia and Jupp(2009)]{mardia2009directional}
Kanti~V Mardia and Peter~E Jupp.
\newblock \emph{Directional statistics}, volume 494.
\newblock John Wiley \& Sons, 2009.

\bibitem[Banerjee et~al.(2005)Banerjee, Dhillon, Ghosh, and
  Sra]{banerjee2005clustering}
Arindam Banerjee, Inderjit~S Dhillon, Joydeep Ghosh, and Suvrit Sra.
\newblock Clustering on the unit hypersphere using von mises-fisher
  distributions.
\newblock \emph{Journal of Machine Learning Research}, 6\penalty0
  (Sep):\penalty0 1345--1382, 2005.

\bibitem[Zaheer et~al.(2017)Zaheer, Kottur, Ravanbakhsh, Poczos, Salakhutdinov,
  and Smola]{zaheer2017deep}
Manzil Zaheer, Satwik Kottur, Siamak Ravanbakhsh, Barnabas Poczos, Ruslan~R
  Salakhutdinov, and Alexander~J Smola.
\newblock Deep sets.
\newblock In \emph{Advances in neural information processing systems}, pages
  3391--3401, 2017.

\end{thebibliography}
}

\clearpage
\onecolumn
\appendix
\section{Measurability of Bayes Estimators}
\label{sec:measurable_bayes_estimator}
For any prior $\Pi$, define $p_{\Pi}(\mathbb{X}^n)$ as 
\begin{align*}
    \int_{\theta} \prod_{i=1}^n p(X_i;\theta) d\Pi(\theta).
\end{align*}
For any prior $\Pi$, define estimator $\hat{\theta}_{\Pi}$ as follows
\begin{equation*}
\hat{\theta}_{\Pi}(\mathbb{X}^n) \in \argmin_{\tilde{\theta} \in \Theta} \Eover{\theta \sim \Pi(\cdot|\mathbb{X}^n)}{M(\tilde{\theta}, \theta)}.
\end{equation*}
Certain regularity conditions need to hold for this to be a Bayes estimator of $\Pi$. $\hat{\theta}_{\Pi}$ defined this way need not be a measurable function of $\mathbb{X}^n$. We now provide sufficient conditions on the statistical problem which guarantee measurability of $\hat{\theta}_{\Pi}$.
These conditions are from~\citet{brown1973measurable}. 
\begin{assumption}
The sample space $\mathcal{X}^n$ and the parameter set $\Theta$ are non-empty Borel sets.
\end{assumption}
\begin{assumption}
\label{asmpt:conditional_prob}
Let $\mathcal{B}(\mathcal{X}^n)$ be the Borel $\sigma$-algebra corresponding to the sample space $\mathcal{X}^n$ and $\mathcal{B}(\Theta)$ be the Borel $\sigma$-algebra corresponding to parameter space $\Theta$. Let $\Pi$ be a prior probability measure on $\Theta$. Suppose, for each $\theta \in \Theta$, $P_{\theta}$  is such that, for each $B\in\mathcal{B}(\mathcal{X}^n)$, the function $\theta \to P_{\theta}(B)$ is measurable w.r.t $\mathcal{B}(\Theta)$. 
\end{assumption}
\begin{assumption}
The loss function $M$ defined on $\Theta \times \Theta$ and taking non-negative real values, is measurable w.r.t $\mathcal{B}(\Theta)\times \mathcal{B}(\Theta)$. Moreover, $M(\cdot, \theta)$ is lower semi-continuous on $\Theta$, for each $\theta \in \Theta$.
\end{assumption}
Under these assumptions, when $\Theta$ is compact, \citet{brown1973measurable} show that there exists a Borel measurable function $\hat{\theta}_{\Pi}$ such that
\begin{equation*}
\hat{\theta}_{\Pi}(\mathbb{X}^n) \in \argmin_{\tilde{\theta} \in \Theta} \Eover{\theta \sim \Pi(\cdot|\mathbb{X}^n)}{M(\tilde{\theta}, \theta)}.
\end{equation*}
Moreover, $\hat{\theta}_{\Pi}$ is the Bayes estimator for $\Pi.$

\section{Minimax Estimators, LFPs and Nash Equilibirium}
\label{sec:minimax_lfp_ne}
\begin{proposition}
Consider the statistical game in Equation~\eqref{eqn:minimax_objective}.
If $(\hat{\theta}^*, P^*)$ is a mixed strategy NE of~\eqref{eqn:minimax_objective}, then the minmax and maxmin values of the linearized game are equal to each other. Moreover, $\hat{\theta}^*$ is a minimax estimator and  $P^*$ is an LFP. Conversely, if $\hat{\theta}^*$ is a minimax estimator, and $P^*$ is an LFP, and the minmax and maxmin values of the linearized game~\eqref{eqn:minimax_objective_mixed} are equal to each other, then $(\hat{\theta}^*, P^*)$ is a mixed strategy NE of~\eqref{eqn:minimax_objective}. Moreover, $\theta^*$ is a Bayes estimator for $P^*$.
\end{proposition}
\begin{proof}
Suppose $(\hat{\theta}^*, P^*)$ is a mixed strategy NE. Then, from the definition of mixed strategy NE, we have
\[
\sup_{P\in\mathcal{M}_{\Theta}}R(\hat{\theta}^*, P)\leq R(\hat{\theta}^*, P^*) \leq \inf_{\hat\theta \in \mathcal{M}_{\mathcal{D}}} R(\hat\theta, P^*).
\]
This further implies
\[
\inf_{\hat\theta \in \mathcal{M}_{\mathcal{D}}} \sup_{P\in\mathcal{M}_{\Theta}}  R(\hat\theta, P) \stackrel{(a)}{\leq} \sup_{P\in\mathcal{M}_{\Theta}}R(\hat{\theta}^*, P)\leq R(\hat{\theta}^*, P^*) \stackrel{(b)}{\leq} \inf_{\hat\theta \in \mathcal{M}_{\mathcal{D}}} R(\hat\theta, P^*) \stackrel{(c)}{\leq} \sup_{P\in\mathcal{M}_{\Theta}} \inf_{\hat\theta \in \mathcal{M}_{\mathcal{D}}} R(\hat\theta, P).
\]
Since $\inf_{\hat\theta \in \mathcal{M}_{\mathcal{D}}} \sup_{P\in\mathcal{M}_{\Theta}}  R(\hat\theta, P) \geq \sup_{P\in\mathcal{M}_{\Theta}} \inf_{\hat\theta \in \mathcal{M}_{\mathcal{D}}} R(\hat\theta, P)$, the above set of inequalities all hold with an equality and  imply that the minmax and maxmin values of the linearized game are equal to each other. Moreover, from $(a)$, we have $\sup_{P\in\mathcal{M}_{\Theta}}R(\hat{\theta}^*, P) = \inf_{\hat\theta \in \mathcal{M}_{\mathcal{D}}} \sup_{P\in\mathcal{M}_{\Theta}}  R(\hat\theta, P) $. This implies $\hat\theta^*$ is a minimax estimator. From $(c)$, we have $\inf_{\hat\theta \in \mathcal{M}_{\mathcal{D}}} R(\hat\theta, P^*) = \sup_{P\in\mathcal{M}_{\Theta}} \inf_{\hat\theta \in \mathcal{M}_{\mathcal{D}}} R(\hat\theta, P).$ This implies $P^*$ is an LFP. Finally, from $(b)$, we have $R(\hat{\theta}^*, P^*) = \inf_{\hat\theta \in \mathcal{M}_{\mathcal{D}}} R(\hat\theta, P^*)$. This implies $\hat\theta^*$ is a Bayes estimator for $P^*$.

We now prove the converse. Since $\hat\theta^*$ is a minimax estimator and $P^*$ is an LFP, we have
\[
\sup_{P\in\mathcal{M}_{\Theta}}R(\hat{\theta}^*, P)= \inf_{\hat\theta \in \mathcal{M}_{\mathcal{D}}} \sup_{P\in\mathcal{M}_{\Theta}}  R(\hat\theta, P),\quad \inf_{\hat\theta \in \mathcal{M}_{\mathcal{D}}} R(\hat\theta, P^*) = \sup_{P\in\mathcal{M}_{\Theta}} \inf_{\hat\theta \in \mathcal{M}_{\mathcal{D}}} R(\hat\theta, P).
\]
Moreover, since minmax and maxmin values of the linearized game are equal to each other, all the above 4 quantities are equal to each other. Since $R(\hat\theta^*,P^*) \leq \sup_{P\in\mathcal{M}_{\Theta}}R(\hat{\theta}^*, P)$ and $R(\hat\theta^*,P^*) \geq \inf_{\hat\theta \in \mathcal{M}_{\mathcal{D}}} R(\hat\theta, P^*)$, we have
\[
\sup_{P\in\mathcal{M}_{\Theta}}R(\hat{\theta}^*, P) = R(\hat\theta^*,P^*) = \inf_{\hat\theta \in \mathcal{M}_{\mathcal{D}}} R(\hat\theta, P^*).
\]
This shows that $(\hat\theta^*,P^*)$ is a mixed strategy NE of the linear game in Equation~\eqref{eqn:minimax_objective_mixed}.
\end{proof}
\section{Follow the Perturbed Leader (FTPL)}
\label{sec:ftpl}
We now describe the FTPL algorithm in more detail. We first introduce the notion of an offline optimization oracle, which takes as input a function $f:\mathcal{X}\rightarrow \mathbb{R}$ and a perturbation vector $\sigma$ and returns an approximate minimizer of $f(\x)-\iprod{\sigma}{\x}$.  An optimization oracle is called ``$(\alpha, \beta)$-approximate optimization oracle'' if it returns $\x^*\in \mathcal{X}$ such that  
\[
f(\x^*) - \iprod{\sigma}{\x^*} \leq \inf_{\x \in \mathcal{X}} f(\x) - \iprod{\sigma}{\x} + \alpha + \beta \|\sigma\|_1.
\]
Denote such an oracle with $\oracle{\alpha, \beta}{f,\sigma}$. Given access to such an oracle, the FTPL algorithm  is given by the following prediction rule (see Algorithm~\ref{alg:ftpl}) $$\x_t = \oracle{\alpha, \beta}{\sum_{i = 1}^{t-1}f_i,\sigma},$$
where $\sigma \in \mathbb{R}^d$ is a random perturbation such that $\{\sigma_{j}\}_{j = 1}^d \stackrel{i.i.d}{\sim} \text{Exp}(\eta)$ and $\text{Exp}(\eta)$ is the exponential distribution with parameter $\eta$.
We now state the following result from \citet{ftpl_nonconvex} which provides an upper bound on the expected regret of Algorithm~\ref{alg:ftpl}.
\begin{theorem}[Regret Bound]
\label{thm:ftpl}
Let $D$ be the $\ell_{\infty}$ diameter of $\mathcal{X}$. Suppose the losses encountered by the learner are $L$-Lipschitz w.r.t $\ell_{1}$ norm. For any fixed $\eta$, the predictions of Algorithm~\ref{alg:ftpl} satisfy the following regret bound 
\[
\Eover{}{\frac{1}{T}\sum_{t = 1}^T f_t(\x_t) - \frac{1}{T}\inf_{\x \in \mathcal{X}}\sum_{t=1}^Tf_t(\x)} \leq O\left(\eta d^2D L^2  + \frac{d(\beta T + D)}{\eta T} + \alpha + \beta d L\right).
\]
\end{theorem}
\begin{algorithm}[tbh]
\caption{Follow the Perturbed Leader (FTPL)}
\label{alg:ftpl}
\begin{algorithmic}[1]
  \small
  \State \textbf{Input:}   Parameter of exponential distribution $\eta$, approximate optimization subroutine $\mathcal{O}_{\alpha, \beta}$
  \For{$t = 1 \dots T$}
  \State Generate random vector $\sigma$ such that $\{\sigma_{j}\}_{j = 1}^d \stackrel{i.i.d}{\sim} \text{Exp}(\eta)$
  \State Predict $\x_t$ as
  \[
  \x_t = \oracle{\alpha, \beta}{\sum_{i = 1}^{t-1}f_i,\sigma}.\]
  \State Observe loss function $f_t$
  \EndFor
\end{algorithmic}
\end{algorithm}
\section{Minimax Estimation via Online Learning}
\subsection{Proof of Proposition~\ref{prop:no_regret_to_games}}
We have the following bounds on the regret of the minimization and maximization players
\begin{align*}
    \sum_{t=1}^TR(\hat{\theta}_t, P_t)-\inf_{\hat{\theta} \in \mathcal{D}} \sum_{t=1}^TR(\hat{\theta}, P_t) &\leq \epsilon_1(T),\\
    \sup_{\theta\in\Theta}\sum_{t=1}^TR(\hat{\theta}_t, \theta) -\sum_{t=1}^TR(\hat{\theta}_t, P_t)&\leq \epsilon_2(T).
\end{align*}
Now consider the following
\begin{equation}
\label{eqn:ftpl_games_approx_eql}
\begin{array}{ll}
    &\displaystyle\inf_{\hat{\theta} \in \mathcal{D}} \frac{1}{T}\sum_{t=1}^TR(\hat{\theta}, P_t) \vspace{0.1in}\\
    &\displaystyle \quad \geq \frac{1}{T}\sum_{t=1}^TR(\hat{\theta}_t, P_t) - \frac{\epsilon_1(T)}{T} \vspace{0.1in}\\
   & \displaystyle \quad \geq \sup_{\theta\in\Theta}\frac{1}{T}\sum_{t=1}^TR(\hat{\theta}_t, \theta) - \frac{\epsilon_1(T)+\epsilon_2(T)}{T},
   \end{array}
\end{equation}
where the first and the second inequalities follow from the regret bounds of the minimization and maximization players. We further bound the LHS and RHS of the above inequality as follows
\[
\inf_{\hat{\theta} \in \mathcal{D}} \frac{1}{T}\sum_{t=1}^TR(\hat{\theta}, P_t) \leq \frac{1}{T^2}\sum_{t=1}^T\sum_{t'=1}^T R(\hat{\theta}_{t'}, P_t) = R(\hat{\theta}_{\RND}, P_{\AVG}),
\]
\[
\sup_{\theta\in\Theta}\frac{1}{T} \sum_{t=1}^TR(\hat{\theta}_t, \theta) \geq \frac{1}{T^2}\sum_{t=1}^T\sum_{t'=1}^T R(\hat{\theta}_{t'}, P_t) = R(\hat{\theta}_{\RND}, P_{\AVG}).
\]
Combining the previous two sets of inequalities gives us
\[
R(\hat{\theta}_{\RND}, P_{\AVG}) \geq \sup_{\theta\in \Theta}R(\hat{\theta}_{\RND}, \theta) - \frac{\epsilon_1(T)+\epsilon_2(T)}{T},
\]
\[
R(\hat{\theta}_{\RND}, P_{\AVG}) \leq \inf_{\hat{\theta} \in \mathcal{D}} R(\hat{\theta}, P_{\AVG}) + \frac{\epsilon_1(T)+\epsilon_2(T)}{T}.
\]
\subsection{Proof of Theorem~\ref{thm:ftpl_games}}
To prove the Theorem we first bound the regret of each player and then rely on Proposition~\ref{prop:no_regret_to_games} to show that the iterates converge to a NE.
Since the maximization player is responding using FTPL to the actions of minimization player, we rely on Theorem~\ref{thm:ftpl} to bound her regret. First note that the sequence of reward functions seen by the maximization player $R(\hat{\theta}_i,\cdot)$ are $L$-Lipschitz. Moreover, the domain $\Theta$ has $\ell_{\infty}$ diameter of $D$. So applying Theorem~\ref{thm:ftpl} gives us the following regret bound
\[
\mathbb{E}_{\sigma}\left[\sup_{\theta\in\Theta}\frac{1}{T}\sum_{t=1}^TR(\hat{\theta}_t, \theta) - \frac{1}{T}\sum_{t=1}^TR(\hat{\theta}_t, \theta_t(\sigma))\right] \leq O\left(\eta d^2D L^2  + \frac{d(\beta T + D)}{\eta T} + \alpha + \beta d L\right).
\]
Taking the expectation inside, we get the following
\begin{equation}
\label{eqn:ftpl_maxp}
\sup_{\theta\in\Theta}\frac{1}{T}\sum_{t=1}^TR(\hat{\theta}_t, \theta) - \frac{1}{T}\sum_{t=1}^TR(\hat{\theta}_t, P_t) \leq O\left(\eta d^2D L^2  + \frac{d(\beta T + D)}{\eta T} + \alpha + \beta d L\right).
\end{equation}
Since the minimization player is using BR, her regret is upper bounded by $0$. Plugging in these two regret bounds in Proposition~\ref{prop:no_regret_to_games} gives us the required result.
\subsection{Proof of Corollary~\ref{cor:minimax_theorem}}
Note that this corollary is only concerned about existence of minimax estimators and LFPs, and showing that minmax and maxmin values of Equation~\eqref{eqn:minimax_objective_mixed} are equal to each other. So we can ignore the approximation errors introduced by the oracles and set $\alpha = \beta = \alpha'=0$ in the results of Theorem~\ref{thm:ftpl_games} (that is, we assume access to exact optimization oracles, as we are only concerned with existence of NE and not about computational tractability of the algorithm). 
\paragraph{Minimax Theorem.} To prove the first part of the corollary, we set $\eta = \sqrt{\frac{1}{dL^2T}}$ in Theorem~\ref{thm:ftpl_games} and let $ T \to \infty$. We get
\begin{align*}
    & \sup_{\theta \in \Theta} R(\hat{\theta}_{\RND}, \theta) =  \inf_{\hat{\theta} \in \mathcal{D}} R(\hat{\theta}, P_{\AVG})\\
    &\implies \sup_{P \in \mathcal{M}_{\Theta}} R(\hat{\theta}_{\RND}, P) =  \inf_{\hat{\theta} \in \mathcal{M}_\mathcal{D}} R(\hat{\theta}, P_{\AVG})\\
    &\implies \inf_{\hat{\theta} \in \mathcal{M}_\mathcal{D}} \sup_{P \in \mathcal{M}_{\Theta}}  R(\hat{\theta}, P) \leq \sup_{P \in \mathcal{M}_{\Theta}} \inf_{\hat{\theta} \in \mathcal{M}_\mathcal{D}} R(\hat{\theta}, P).
\end{align*}
Since minmax value of any game is always greater than or equal to maxmin value of the game, we get 
\[
\inf_{\hat{\theta} \in \mathcal{M}_\mathcal{D}} \sup_{P \in \mathcal{M}_{\Theta}}  R(\hat{\theta}, P) = \sup_{P \in \mathcal{M}_{\Theta}} \inf_{\hat{\theta} \in \mathcal{M}_\mathcal{D}} R(\hat{\theta}, P) R^*.
\]
\paragraph{Existence of LFP.} We now show that the statistical game has an LFP. To prove this result, we make use of the following result on the compactness of probability spaces. If $\Theta$ is a compact space, then $\mathcal{M}_{\Theta}$ is sequentially compact; that is, any sequence $P_n\in \mathcal{M}_{\Theta}$ has a convergent subsequence converging to a point in $\mathcal{M}_{\Theta}$ (the notion of convergence here is weak convergence). Let $P_{\AVG,t} = \frac{1}{t}\sum_{i=1}^t P_i$ be the mixture distribution obtained from the first $t$ iterates of Algorithm~\ref{alg:ftpl_stat_games} when run with $\eta = \sqrt{\frac{1}{dL^2T}}$ and exact optimization oracles. Consider the sequence of probability measures $\{P_{\AVG,t}\}_{t=1}^{\infty}$. Since the parameter space $\Theta$ is compact, we know that there exists a converging subsequence $\{P_{\AVG,t_i}\}_{i=1}^{\infty}$. Let $P^*\in\mathcal{M}_{\Theta}$ be the limit of this sequence. In the rest of the proof, we show that $P^*$ is an LFP; that is, $\inf_{\hat{\theta} \in \mathcal{D}}R(\hat{\theta}, P^*) = R^*.$ Since $R(\hat{\theta}, \theta)$ is bounded, and Lipschitz in its second argument, we have 
\begin{align}
\label{eqn:lfp_existence_eq0}
\forall \hat{\theta}\in \mathcal{M}_{\mathcal{D}}\quad \lim_{i \to \infty}R(\hat{\theta}, P_{\AVG,t_i}) = R(\hat{\theta}, P^*).
\end{align}
This follows from the equivalent formulations of weak convergence of measures.
We now make use of the following result from Corollary~\ref{cor:ftpl_minimax_estimators} (which we prove later in Appendix~\ref{sec:proof_cor_ftpl_minimax_estimators})
\[
\inf_{\hat{\theta} \in \mathcal{D}} R(\hat{\theta}, P_{\AVG,t}) \geq  R^* - O(t^{-\frac{1}{2}}).
\]
Combining this with the fact that $\sup_{P \in \mathcal{M}_{\Theta}}\inf_{\hat{\theta} \in \mathcal{D}} R(\hat{\theta}, P) = R^*$, we get
\begin{align}
\label{eqn:lfp_existence_eq1}
\lim_{i \to \infty} \inf_{\hat{\theta} \in \mathcal{D}} R(\hat{\theta}, P_{\AVG,t_i}) = R^*.
\end{align}
Equations~\eqref{eqn:lfp_existence_eq0},~\eqref{eqn:lfp_existence_eq1} show that $\inf_{\hat{\theta} \in \mathcal{D}} R(\hat{\theta}, P_{\AVG,t_i})$, $R(\tilde{\theta}, P_{\AVG,t_i})$ are converging sequences as $i \to \infty$.
Since $\inf_{\hat{\theta} \in \mathcal{D}} R(\hat{\theta}, P_{\AVG,t_i}) \leq R(\tilde{\theta}, P_{\AVG,t_i})$ for all $i, \tilde{\theta} \in \mathcal{D}$, we have
\begin{align*}
&\lim_{i \to \infty} \inf_{\hat{\theta} \in \mathcal{D}} R(\hat{\theta}, P_{\AVG,t_i}) \leq  \lim_{i \to \infty} R(\tilde{\theta}, P_{\AVG,t_i}), \quad \forall \tilde{\theta} \in \mathcal{D}.
\end{align*}
From Equations~\eqref{eqn:lfp_existence_eq0},~\eqref{eqn:lfp_existence_eq1}, we then have
\begin{align*}
& R^* \leq  R(\tilde{\theta}, P^*), \quad \forall \tilde{\theta} \in \mathcal{D}\\
&\implies R^* \leq  \inf_{\hat{\theta} \in \mathcal{D}}R(\hat{\theta}, P^*),
\end{align*}
Combining this with the fact that $\sup_{P \in \mathcal{M}_{\Theta}}\inf_{\hat{\theta} \in \mathcal{D}} R(\hat{\theta}, P) = R^*$, we get
\[
\inf_{\hat{\theta} \in \mathcal{D}}R(\hat{\theta}, P^*) = R^*.
\]
This shows that $P^*$ is an LFP.
\paragraph{Existence of Minimax Estimator.}  To show the existence of a minimax estimator, we make use of the following result from~\citet{wald1949statistical}, which is concerned about the ``compactness'' of the space of estimators $\mathcal{M}_{\mathcal{D}}$. 
\begin{proposition}
Suppose $\Theta$ is compact w.r.t $\Delta_M(\theta_1,\theta_2) = \sup_{\theta \in \Theta}|M(\theta_1, \theta) - M(\theta_2,\theta)|.$ Moreover, suppose the risk $R$ is bounded. Then for any sequence of $\{\hat\theta_i\}_{i=1}^{\infty}$ of estimators there exists a subsequence $\{\hat\theta_{i_j}\}_{j=1}^{\infty}$ such that $\lim_{j\to\infty}\hat\theta_{i_j} = \hat\theta_0$ and for any $\theta \in \Theta$
\[
\lim\inf_{i\to\infty} R(\hat\theta_{i_j}, \theta) \geq R(\hat\theta_{0}, \theta).
\]
\end{proposition}
Let $\hat{\theta}_{\RND,t}$ be the randomized estimator obtained by uniformly sampling an estimator from $\{\hat\theta_i\}_{i=1}^t$. Consider the sequence of estimators $\{\hat{\theta}_{\RND,t}\}_{t=1}^{\infty}$. From the above proposition, we know that there exists a subsequence $\{\hat{\theta}_{\RND,t_j}\}_{j=1}^{\infty}$ and an estimator $\hat{\theta}^*$ such that $\lim\inf_{j\to\infty} R(\hat{\theta}_{\RND,t_j}, \theta) \geq R(\hat\theta^*, \theta)$. We now show that $\hat\theta^*$ is a minimax estimator; that is, we show that $\sup_{\theta\in\Theta} R(\hat\theta^*, \theta) = R^*$. We make use of the following result from Corollary~\ref{cor:ftpl_minimax_estimators}
\[
\sup_{\theta \in \Theta} R(\hat{\theta}_{\RND,t}, \theta) \leq  R^* + O(t^{-\frac{1}{2}}).
\]
Combining this with the fact that $\inf_{\hat{\theta} \in \mathcal{D}}\sup_{P \in \mathcal{M}_{\Theta}} R(\hat{\theta}, P) = R^*$, we get
\begin{align}
\label{eqn:minimax_existence_eq1}
\lim_{j \to \infty} \sup_{\theta \in \Theta} R(\hat{\theta}_{\RND,t_j}, \theta) = R^*.
\end{align}
Since $\sup_{\theta \in \Theta} R(\hat{\theta}_{\RND,t_j}, \theta) \geq R(\hat{\theta}_{\RND,t_j}, \Tilde{\theta})$ for any $j, \Tilde{\theta}\in \Theta,$ we have
\begin{align*}
&\lim\inf_{j \to \infty} \sup_{\theta \in \Theta} R(\hat{\theta}_{\RND,t_j}, \theta) \geq  \lim\inf_{j \to \infty} R(\hat{\theta}_{\RND,t_j}, \Tilde{\theta}) \geq R(\hat\theta^*,\theta), \quad \forall \tilde{\theta} \in \Theta.
\end{align*}
Since $\{ R(\hat{\theta}_{\RND,t_j}, \theta)\}_{j=1}^{\infty}$ is a converging sequence, we have $$\lim\inf_{j \to \infty} \sup_{\theta \in \Theta} R(\hat{\theta}_{\RND,t_j}, \theta) = \lim_{j \to \infty} \sup_{\theta \in \Theta} R(\hat{\theta}_{\RND,t_j}, \theta)=R^*.$$
This together with the previous inequality gives us $\sup_{\Tilde{\theta} \in \Theta}R(\hat{\theta}_{\RND,t_j}, \Tilde{\theta}) \leq R^*$. This shows that $\theta^*$ is a minimax estimator.

\subsection{Proof of Corollary~\ref{cor:ftpl_minimax_estimators}}
\label{sec:proof_cor_ftpl_minimax_estimators}
\paragraph{Minimax Estimator.}  From  Theorem~\ref{thm:ftpl_games} we have
\begin{align*}
    \sup_{\theta \in \Theta}R(\hat{\theta}_{\RND}, \theta) & = \sup_{\theta \in \Theta}\frac{1}{T}\sum_{i=1}^TR(\hat{\theta}_{i}, \theta) \\
    & \leq \inf_{\hat{\theta} \in \mathcal{D}} \frac{1}{T}\sum_{i=1}^TR(\hat{\theta}, P_i) + O\left(\eta d^2D L^2  + \frac{d(\beta T + D)}{\eta T} + \alpha + \alpha' + \beta d L\right)\\
    &=\inf_{\hat{\theta} \in \mathcal{M}_\mathcal{D}} \frac{1}{T}\sum_{i=1}^TR(\hat{\theta}, P_i) + O\left(\eta d^2D L^2  + \frac{d(\beta T + D)}{\eta T} + \alpha + \alpha' + \beta d L\right)\\
    &\stackrel{(a)}{\leq} \inf_{\hat{\theta} \in \mathcal{M}_\mathcal{D}} \sup_{P \in\mathcal{M}_\Theta} R(\hat{\theta}, P) + O\left(\eta d^2D L^2  + \frac{d(\beta T + D)}{\eta T} + \alpha + \alpha'+ \beta d L\right),
\end{align*}
where $(a)$ follows from the fact that $\sup_{\theta \in\Theta} R(\hat{\theta}, \theta) \geq \frac{1}{T}\sum_{i=1}^TR(\hat{\theta}, P_i) $.
Substituting $\eta = \sqrt{\frac{1}{dL^2T}}$ in the above equation shows that the randomized estimator is approximately minimax.
This completes the first part of the proof.
If the metric $M$ is convex in its first argument, then from Jensen's inequality we have
\[
\forall \theta, \quad R(\hat{\theta}_{\AVG}, \theta) \leq R(\hat{\theta}_{\RND}, \theta).
\]
This shows that the worst-case risk of $\hat{\theta}_{\AVG}$ is upper bounded as
\begin{align}
\label{eqn:cor_ftpl_partb}
\sup_{\theta \in \Theta}R(\hat{\theta}_{\AVG}, \theta)  \leq \inf_{\hat{\theta} \in \mathcal{M}_\mathcal{D}} \sup_{P \in\mathcal{M}_\Theta} R(\hat{\theta}, \theta) + O\left(\eta d^2D L^2  + \frac{d(\beta T + D)}{\eta T} + \alpha+\alpha' + \beta d L\right).
\end{align}
Substituting $\eta = \sqrt{\frac{1}{dL^2T}}$ in Equation~\eqref{eqn:cor_ftpl_partb} gives us the required bound on the worst-case risk of $\hat{\theta}_{\AVG}$.
\paragraph{LFP.} We now prove the results pertaining to LFP. From Theorem~\ref{thm:ftpl_games}, we have
\begin{align*}
    \inf_{\hat{\theta} \in \mathcal{M}_\mathcal{D}} R(\hat{\theta}, P_{\AVG}) &= \inf_{\hat{\theta} \in \mathcal{M}_\mathcal{D}}\frac{1}{T}\sum_{i=1}^T R(\hat{\theta}, P_{i})\\
    &\geq \sup_{P\in\mathcal{M}_{\Theta}} \frac{1}{T}\sum_{i=1}^TR(\hat{\theta}_i,P) - O\left(\eta d^2D L^2  + \frac{d(\beta T + D)}{\eta T} + \alpha + \alpha'+ \beta d L\right)\\
    &\geq \inf_{\hat{\theta}\in \mathcal{M}_{\mathcal{D}}} \sup_{P\in\mathcal{M}_{\Theta}}  R(\hat{\theta}, P) - O\left(\eta d^2D L^2  + \frac{d(\beta T + D)}{\eta T} + \alpha + \alpha'+ \beta d L\right).
\end{align*}
Substituting $\eta = \sqrt{\frac{1}{dL^2T}}$ in the above equation shows that $P_{\AVG}$ is approximately least favourable. Now consider the case where $M$ is convex in its first argument. To show that $\hat{\theta}_{\AVG}$ is an approximate Bayes estimator for $P_{\AVG}$, we again rely on Theorem~\ref{thm:ftpl_games} where we showed that
\[
\sup_{P\in\mathcal{M}_{\Theta}} \frac{1}{T}\sum_{i=1}^TR(\hat{\theta}_i,P)\leq \inf_{\hat{\theta}\in\mathcal{M}_{\mathcal{D}}} \frac{1}{T}\sum_{t=1}^T R(\hat{\theta},P_t) + O\left(\eta d^2D L^2  + \frac{d(\beta T + D)}{\eta T} + \alpha + \alpha' + \beta d L\right).
\]
Since $\frac{1}{T^2}\sum_{t=1}^T\sum_{t'=1}^T R(\hat{\theta}_{t'}, P_t) \leq \sup_{P\in\mathcal{M}_{\Theta}} \frac{1}{T}\sum_{i=1}^TR(\hat{\theta}_i,P)$, we have
\[
\frac{1}{T^2}\sum_{t=1}^T\sum_{t'=1}^T R(\hat{\theta}_{t'}, P_t)\leq \inf_{\hat{\theta}\in\mathcal{M}_{\mathcal{D}}} \frac{1}{T}\sum_{t=1}^T R(\hat{\theta},P_t) + O\left(\eta d^2D L^2  + \frac{d(\beta T + D)}{\eta T} + \alpha + \alpha' + \beta d L\right).
\]
Since $M$ is convex in its first argument, we have
\[
\frac{1}{T^2}\sum_{t=1}^T\sum_{t'=1}^T R(\hat{\theta}_{t'}, P_t)  \geq \frac{1}{T}\sum_{i=1}^T R(\hat{\theta}_{\AVG},P_i).
\]
Combining the above two equations shows that $\hat{\theta}_{\AVG}$ is an approximate Bayes estimator for $P_{\AVG}$.

\section{Invariance of Minimax Estimators}

\subsection{Proof of Theorem \ref{thm:invariance-est-lfp}}

In our proof, we rely on the following property of left Haar measure $\mu$ of a compact group $G$.  For any real valued integrable function $f$ on $G$ and any $g \in G$~\citep[see Chapter 7 of][]{wijsman1990invariant}
\begin{align}
    \label{eqn:haar_property}
    \int_{G} f(g^{-1} h)d\mu(h) = \int_{G} f(h) d\mu(h).
\end{align}
We now proceed to the proof of the Theorem. 
For any estimator $\hat{\theta}:\mathcal{X}^n\rightarrow\Theta$, define the following estimator $\hat{\theta}_{G}$
\[\hat{\theta}_{G}(\mathbb{X}^n)=\int_G g\hat{\theta}(g^{-1}\mathbb{X}^n)d\mu(g),\] where $\mu$ is the left Haar measure on $G$ and $g\mathbb{X}^n =\{gX_1, \dots gX_n\}$. The above integral is well defined because $\hat{\theta}$ is measurable, $G$ is compact and the action of the group $G$ is continuous.
We first show that $\hat{\theta}_{G}$ is invariant under group transformations $G$. For any $h \in G$, consider the following
\begin{align*}
    \hat{\theta}_{G}(h\mathbb{X}^n) &= \int_G g\hat{\theta}((g^{-1} h)\mathbb{X}^n)d\mu(g)\\
    &= \int_G h(h^{-1}g)\hat{\theta}((h^{-1}g)^{-1}\mathbb{X}^n)d\mu(g)\\
    &=  h \left[\int_G  (h^{-1}g)\hat{\theta}((h^{-1}g)^{-1}\mathbb{X}^n)d\mu(g)\right]\\
    &\stackrel{(a)}{=} h  \left[\int_G  g\hat{\theta}(g^{-1}\mathbb{X}^n)d\mu(g)\right]\\
    &= h\hat{\theta}_G(\mathbb{X}^n),
\end{align*}
where $(a)$ follows from Equation~\eqref{eqn:haar_property}.
This shows that $\hat{\theta}_G$ is an invariant estimator. We now show that the worst case risk of $\hat{\theta}_G$ is less than or equal to the worst case risk of $\hat{\theta}$.
Consider the following upper bound on the risk of $\hat{\theta}_G$ at any $\theta \in \Theta$
\begin{align*}
 R(\hat{\theta}_G,\theta)&=\mathbb{E}_{\mathbb{X}^n \sim P_\theta^n}\left[M(\hat{\theta}_G(\mathbb{X}^n),\theta)\right]\\
 & \leq \mathbb{E}_{\mathbb{X}^n \sim P_\theta^n}\left[\int_G M(g\hat{\theta}(g^{-1}\mathbb{X}^n),\theta) d\mu(g)\right]\quad \text{(convexity of $M$)}\\
&=\mathbb{E}_{\mathbb{X}^n \sim P_\theta^n}\left[\mathbb{E}_{g\sim\mu}\left[M(g\hat{\theta}(g^{-1}\mathbb{X}^n),\theta)\right]\right]\\
&\stackrel{(a)}{=}\mathbb{E}_{g\sim\mu}\left[\mathbb{E}_{\mathbb{X}^n \sim P_{g^{-1}\theta}^n}\left[M(g\hat{\theta}(\mathbb{X}^n),\theta)\right]\right]\quad \text{(change of variables)}\\
&\stackrel{(b)}{=}\mathbb{E}_{g\sim\mu}\left[\mathbb{E}_{\mathbb{X}^n \sim P_{g^{-1}\theta}^n}\left[M(\hat{\theta}(\mathbb{X}^n),g^{-1}\theta)\right]\right]\quad \text{(invariance of $M$)}\\
&= \mathbb{E}_{g\sim\mu}\left[R(\hat{\theta},g^{-1}\theta)\right]\\
&\leq \sup_{\theta' \in \Theta} R(\hat{\theta}, \theta') ,
\end{align*}
where $(a)$ follows from Fubini's theorem and change of variables $X'=g^{-1}X$ and the fact that if $X \sim P_{\theta}$, then $g^{-1}X \sim P_{g^{-1}\theta}$. $(b)$ follows from the invariance property of the metric $M$. This shows that $\sup_{\theta \in \Theta}  R(\hat{\theta}_G,\theta) \leq \sup_{\theta \in \Theta} R(\hat{\theta}, \theta)$.
This shows that we can always improve a given estimator by averaging over the group $G$ and hence there should be a minimax estimator which is invariant under the action of $G$.

\subsection{Proof of Theorem~\ref{thm:ftpl_minimax_invariance}}
\subsubsection{Intermediate Results}
We first prove some intermediate results which we require in the proof of Theorem~\ref{thm:ftpl_minimax_invariance}.
\begin{lemma}
\label{lem:invariant_risk}
Suppose $\hat{\theta}$ is a deterministic estimator that is invariant to group transformations $G$. Then $R(\hat{\theta}, \theta_1) = R(\hat{\theta}, \theta_2)$, whenever $\theta_1\sim \theta_2$.
\end{lemma}
\begin{proof}
Suppose $\theta_2 = g\theta_1$ for some $g \in G$. From the definition of $R(\hat{\theta}, g\theta_1)$ we have
\begin{align*}
    R(\hat{\theta}, \theta_2) = R(\hat{\theta}, g\theta_1) &= \Eover{\mathbb{X}^n \sim P_{g\theta_1}^n}{M(\hat{\theta}(\mathbb{X}^n), g\theta_1)}\\
    &= \Eover{\mathbb{X}^n \sim P_{g\theta_1}^n}{M(g^{-1}\hat{\theta}(\mathbb{X}^n), \theta_1)} \quad \text{(invariance of loss metric)}\\
    & = \Eover{\mathbb{X}^n \sim P_{g\theta_1}^n}{M(\hat{\theta}(g^{-1}\mathbb{X}^n), \theta_1)}\quad \text{(invariance of estimator)}\\
    & \stackrel{(a)}{=} \Eover{\mathbb{X}^n \sim P_{\theta_1}^n}{M(\hat{\theta}(\mathbb{X}^n), \theta_1)}\\
    & = R(\hat{\theta}, \theta_1),
\end{align*}
where $(a)$ follows from the fact that $gX\sim P_{g\theta}$ whenever $X\sim P_{\theta}$. This shows that $R(\hat{\theta}, \theta_1) = R(\hat{\theta}, \theta_2)$.
\end{proof}
\begin{lemma}
\label{lem:invariant_bayes_estimator}
Suppose $\Pi$ is a probability distribution which is invariant to group transformations $G$. For any deterministic estimator $\hat{\theta}$, there exists an invariant estimator $\hat{\theta}_G$ such that the Bayes risk of $\hat{\theta}_G$ is no larger than the Bayes risk of $\hat{\theta}$
\[
R(\hat{\theta}, \Pi) \geq R(\hat{\theta}_G, \Pi).
\]
\end{lemma}
\begin{proof}
Define estimator $\hat{\theta}_G$ as follows
\[\hat{\theta}_{G}(\mathbb{X}^n)=\int_G g\hat{\theta}(g^{-1}\mathbb{X}^n)d\mu(g),\]
where $\mu$ is the left Haar measure on $G$. Note that, in the proof of Theorem~\ref{thm:invariance-est-lfp} we showed that this estimator is invariance to the action of group $G$.
We now show that the Bayes risk of $\hat{\theta}_G$ is less than equal to the Bayes risk of $\hat{\theta}$. Consider the following
\begin{align*}
R(\hat{\theta}_G, \Pi) &= \mathbb{E}_{\theta \sim \Pi}[R(\hat{\theta}_G, \theta)] \\
&= \Eover{\theta \sim \Pi}{\Eover{\mathbb{X}^n \sim P^n_{\theta}}{M\left(\int_G g\hat{\theta}(g^{-1}\mathbb{X}^n)d\mu(g), \theta\right)}}\\
&\stackrel{(a)}{\leq} \Eover{\theta \sim \Pi}{\Eover{\mathbb{X}^n \sim P^n_{\theta}}{\Eover{g\sim \mu}{M\left( g\hat{\theta}(g^{-1}\mathbb{X}^n), \theta\right)}}}\\
&=  \Eover{g\sim \mu}{\Eover{\theta \sim \Pi}{\Eover{\mathbb{X}^n \sim P^n_{\theta}}{M\left( g\hat{\theta}(g^{-1}\mathbb{X}^n), \theta\right)}}}\\
&\stackrel{(b)}{=} \Eover{g\sim \mu}{\Eover{\theta \sim \Pi}{\Eover{\mathbb{X}^n \sim P^n_{\theta}}{M\left(\hat{\theta}(g^{-1}\mathbb{X}^n), g^{-1}\theta\right)}}}\\
&= \Eover{g\sim \mu}{\Eover{\theta \sim \Pi}{R(\hat{\theta}, g^{-1}\theta)}}\\
&\stackrel{(c)}{=} \Eover{\theta \sim \Pi}{R(\hat{\theta}, \theta)},
\end{align*}
where $(a)$ uses convexity of $M$ and follows from Jensen's inequality, $(b)$ follows from the invariance of $M$ and $(c)$ follows from the invariance of distribution $\Pi$ to actions of group $G$.
\end{proof}
\subsubsection{Main Argument}
We now proceed to the proof of Theorem~\ref{thm:ftpl_minimax_invariance}. We first prove the second part of the Theorem. The first part immediately follows from the proof of second part. Suppose $(\hat\theta^*_G, P^*_G)$ is an $\epsilon$-approximate mixed strategy Nash equilibirium of the reduced statistical game in  Equation~\eqref{eqn:minimax_objective_simplified}. Our goal is to construct an approximate Nash equilibrium of the original statistical game in Equation~\eqref{eqn:minimax_objective}, using $(\hat\theta^*_G, P^*_G)$.

Note that $\hat\theta^*_G$ is a randomized estimator over the set of deterministic invariant estimators $\mathcal{D}_G$ and $P^*_G$ is a distribution on the quotient space $\Theta/G$. 
To construct an approximate Nash equilibrium of the original statistical game~\eqref{eqn:minimax_objective}, we extend $P^*_G$ to the entire parameter space $\Theta$. We rely on Bourbaki's approach to measure theory, which is equivalent to classical measure theory in the setting of locally compact spaces we consider in this work~\citep{wijsman1990invariant}. In Bourbaki's approach, any measure $\nu$ on a set $\Theta$ is defined as a linear functional on the set of integrable functions (that is, a measure is defined by its action on integrable functions)
\[
\nu[f] = \int_{\Theta} f(\theta)d\nu(\theta).
\]
We define $P^*$, the extension of  $P^*_G$ to the entire parameter space $\Theta$, as follows
\[
P^*[f] = \int_{\Theta/ G}f'(\Theta_{\beta}) dP^*_G(\Theta_{\beta}),\]
where $f':\Theta/ G\to\mathbb{R}$ is a function that depends on $f$, and is defined as follows. First define $f_I:\Theta\to\mathbb{R}$, an invariant function constructed using $f$, as 
$
f_I(\theta) = \int_{\Theta} f(g\theta) d\mu(g),
$
where $\mu$ is the left invariant Haar measure of $G$. 
From Equation~\eqref{eqn:haar_property}, it is easy to see that $f_I(h\theta) = f_I(\theta)$, for all $h \in G$. So $f_I$ is constant on the equivalence classes of $\Theta$. So $f_I$ can be written in terms of a function  $f':\Theta/ G\to\mathbb{R}$, as follows
\[
f_I = f' \circ \gamma,
\]
where $\gamma: \Theta \rightarrow \Theta/G$ is the orbit projection function which projects $\theta \in \Theta$ onto the quotient space. 
We first show that $P^*$ defined this way is an invariant measure. To this end, we use the following equivalent definition of an invariant measure.
\begin{proposition}
A probability measure $\nu$ on $\Theta$ is invariant to transformations of group $G$ iff for any $\nu$-integrable function $f$ and for any $h\in G$, $\int f(\theta)d\nu(\theta) = \int f(h\theta) d\nu(\theta)$. \end{proposition}
Since $f_I$ is an invariant function, relying on the above proposition, it is easy to see that $P^*$ is an invariant measure.
We now show that $(\hat\theta^*_G, P^*)$  is an $\epsilon$-approximate mixed strategy Nash equilibrium of Equation~\eqref{eqn:minimax_objective}. Since $(\hat\theta^*_G, P^*_G)$ is an $\epsilon$-approximate Nash equilibrium of Equation~\eqref{eqn:minimax_objective_simplified}, we have
\begin{align}
\label{eqn:appr_eql_reduced_game}
    &\sup_{\Theta_{\beta} \in \Theta/G} R_G(\hat{\theta}_G^*, \Theta_{\beta}) - \epsilon \leq \mathbb{E}_{\Theta_{\beta} \sim P_G^*}[R_G(\hat{\theta}_G^*, \Theta_{\beta})] \leq \inf_{\hat{\theta} \in \mathcal{D}_G} \mathbb{E}_{\Theta_{\beta} \sim P_G^*}[R_G(\hat{\theta}, \Theta_{\beta})] + \epsilon,
\end{align}
where $\mathcal{D}_G$ is the set of deterministic invariant estimators. 
Now consider the following
\begin{align*}
    \mathbb{E}_{\theta \sim P^*}[R(\hat{\theta}^*_G, \theta)] & \stackrel{(a)}{=} \mathbb{E}_{\Theta_{\beta} \sim P_G^*}[R_G(\hat{\theta}^*_G, \Theta_{\beta})] \quad \text{(Lemma~\ref{lem:invariant_risk})}\\
    & \leq \inf_{\hat{\theta} \in \mathcal{D}_G} \mathbb{E}_{\Theta_{\beta} \sim P_G^*}[R_G(\hat{\theta}, \Theta_{\beta})] + \epsilon \quad \text{(Equation~\eqref{eqn:appr_eql_reduced_game})}\\
    & = \inf_{\hat{\theta} \in \mathcal{D}_G} \mathbb{E}_{\theta \sim P^*}[R(\hat{\theta}, \theta)] + \epsilon \quad \text{(definition of $P^*$)}\\
& \stackrel{(b)}{=} \inf_{\hat{\theta} \in \mathcal{D}}  \mathbb{E}_{\theta \sim P^*}[R(\hat{\theta}, \theta)] + \epsilon\quad \text{(Lemma~\ref{lem:invariant_bayes_estimator})},
\end{align*}
where $(a)$ follows from the definition of $P^*$ and Lemma~\ref{lem:invariant_risk}.  $(b)$ follows from the fact that for any invariant prior, there exists a Bayes estimator which is invariant to group transformations (Lemma~\ref{lem:invariant_bayes_estimator}). 
Next, we provide a lower bound for $\mathbb{E}_{\theta \sim P^*}[R(\hat{\theta}^*_G, \theta)]$
\begin{align*}
    \mathbb{E}_{\theta \sim P^*}[R(\hat{\theta}^*_G, \theta)] & =\mathbb{E}_{\Theta_{\beta} \sim P_G^*}[R_G(\hat{\theta}^*_G, \Theta_{\beta})] \\
    & \geq \sup_{\Theta_{\beta} \in \Theta/G} R_G(\hat{\theta}_G^*, \Theta_{\beta}) - \epsilon \\
& = \sup_{\theta \in \Theta} R(\hat{\theta}_G^*, \theta) - \epsilon \quad \text{(Lemma~\ref{lem:invariant_risk})}\\
\end{align*}
The upper and lower bounds for $\mathbb{E}_{\theta \sim P^*}[R(\hat{\theta}^*_G, \theta)]$ derived in the previous two equations shows that $(\hat\theta_G^*,P^*)$ is an $\epsilon$-approximate mixed strategy Nash equilibrium of the original statistical game in Equation~\ref{eqn:minimax_objective}. The above inequalites also show that
\[
\sup_{\theta \in \Theta} R(\hat{\theta}_G^*, \theta) - \epsilon \leq \mathbb{E}_{\Theta_{\beta} \sim P_G^*}[R_G(\hat{\theta}^*_G, \Theta_{\beta})] \leq \inf_{\hat{\theta} \in \mathcal{D}}  \mathbb{E}_{\theta \sim P^*}[R(\hat{\theta}, \theta)] + \epsilon.
\]
This, together with Equation~\eqref{eqn:appr_eql_reduced_game}, shows that
\[
\inf_{\hat{\theta} \in \mathcal{M}_\mathcal{D}} \sup_{\theta \in \Theta}  R(\hat{\theta}, \theta) = \inf_{\hat{\theta} \in \mathcal{M}_{\mathcal{D},G}} \sup_{\Theta_{\beta} \in \Theta/G}  R_G(\hat{\theta}, \Theta_{\beta}).
\]
\subsection{Applications of Invariance Theorem}
In our proofs, we establish homeomorphisms between the quotient spaces and another natural space over which we run our algorithm. Note that establishing a homeomorphism is sufficient since we are only dealing with Borel $\sigma$-algebras on our spaces and homeomorphism would imply that there is an isomorphism between the Borel $\sigma$-algebras of the two spaces. Hence, measures learnt on one space can be transferred to another.
\subsubsection{Proof of Theorem~\ref{thm:sequence_model_invariance}}
First note that for any $g\in \mathbb{O}(d)$ and $\theta \in \Theta$, we have $g\theta \in \Theta$ and the distribution of $gX$ is $P_{g\theta}$. 
Moreover, for any orthogonal matrix $g\in\mathbb{O}(d)$ we have $\|g\theta-gX\|^2=\|\theta-X\|^2$, which implies the  statistical game is invariant to group transformations $G$.

For the second part, note that for any $\theta_1,\theta_2\in\Theta$ such that $\|\theta_1\|_2=\|\theta_2\|_2$, $\exists g\in\mathbb{O}(d)$ s.t. $g\theta_1=\theta_2$. Mapping all elements to their norm gives us a bijection between the quotient space and the interval $[0,B]$. The continuity of this bijection and it's inverse can easily be checked using the standard basis for both the topologies.

\subsubsection{Proof of Theorem~\ref{thm:sequence_model_invariance_few}}
Note that for any $\theta\in\Theta$, $g\theta=[g_1\theta^{1:k},\  g_2\theta^{k+1:d}] \in \Theta$. Since $g_1$ is orthogonal, for any $\theta_1,\theta_2\in \Theta$ we have $\|g_1\theta_1^{1:k}-g_1\theta_2^{1:k}\|=\|\theta_1^{1:k}-\theta_2^{1:k}\|$. Hence the invariance of the statistical game follows.

Now, for any $\theta_1,\theta_2\in\Theta$ such that $\|\theta_1^{1:k}\|=\|\theta_2^{1:k}\|$ and $\|\theta_1^{k+1:d}\|=\|\theta_2^{k+1:d}\|$, $\exists g_1\in\mathbb{O}(k)$ and $g_2\in\mathbb{O}(d-k)$ such that $g_1\theta_1^{1:k}=\theta_2^{1:k}$ and $g_2\theta_1^{k+1:d}=\theta_2^{k+1:d}$. Hence $\exists g\in\mathbb{O}(k)\times\mathbb{O}(d-k)$ such that $g\theta_1=\theta_2$. This means that in each equivalence class the parameters $B_1=\|\theta_1^{1:k}\|^2$ and $B_2=\|\theta_1^{k+1:d}\|^2$ are constant. Since $\|\theta\|^2\leq B$ we have $B_1+B_2\leq B$, this gives us a bijection. The continuity of this bijection and it's inverse can easily be checked using the standard basis for both the topologies.

\subsubsection{Proof of Theorem~\ref{thm:linear_regression_invariance}}
We define the action of any $g\in\mathbb{O}(d)$ on the samples $\{(X_i,Y_i)\}_{i=1}^n$ as transforming them to $\{(gX_i,Y_i)\}_{i=1}^n$. Since $Y_i=X_i^T\theta+\epsilon_i=X_i^Tg^Tg\theta+\epsilon_i=(gX_i)^Tg\theta+\epsilon_i$ and $\|g\theta_1-g\theta_2\|=\|\theta_1-\theta_2\|$ for any $\theta_1,\theta_2\in\Theta$ we have the invariance of the statistical game. The rest of the proof uses similar arguments as in Theorem \ref{thm:sequence_model_invariance}.

\subsubsection{Proof of Theorem~\ref{thm:cov_estimation_invariance}}
First note that for any $\Sigma$ such that $\|\Sigma\|_2 \leq B$, and any $g \in \mathbb{O}(d)$, we have $\|g\Sigma g^T\| \leq B$. 
If $X\sim N(0, \Sigma)$ then for any $g\in\mathbb{O}(d)$
\[\mathbb{E}[gXX^Tg^T]=g\mathbb{E}[XX^T]g^T=g\Sigma g^T.\]
Hence $gX\sim N(0, g\Sigma g^T)$. Moreover, we have
\begin{align*}
&M(g\Sigma_1g^T,g\Sigma_2g^T)\\
&= \mbox{tr}\left((g\Sigma_1g^T)^{-1}g\Sigma_2g^T\right)-\log|(g\Sigma_1 g^T)^{-1}g\Sigma_2g^T|-d\\
&=\mbox{tr}\left(g\Sigma_1^{-1}g^Tg\Sigma_2^{-1}g^T\right)-\log|g\Sigma_1^{-1}g^Tg\Sigma_2^{-1}g^T|-d\\
&=\mbox{tr}(g\Sigma_1^{-1}\Sigma_2g^T)-\log|g\Sigma_1^{-1}\Sigma_2g^T|-d\\
&=M(\Sigma_1,\Sigma_2),
\end{align*}
where the last equality follows from the invariance of trace to multiplication with orthogonal matrices and the property of the determinant to split over the multiplication of matrices. This shows the desired invariance of the statistical game.

Now, consider two covariance matrices $\Sigma_1,\Sigma_2$ with singular value decompositions (SVD) $\Sigma_1=U_1\Delta_1U_1^T$ and $\Sigma_2=U_2\Delta_2U_2^T$ respectively. Here all matrices are square and of full rank. In particular, $\Delta_1$ and $\Delta_2$ are diagonal matrices with decreasing entries from left to right and, $U_1$ and $U_2$ are orthogonal matrices. Since the orthogonal group is transitive $\exists g\in\mathbb{O}(d)$ such that $gU_1=U_2$. If $\Delta_1=\Delta_2$ we have $g\Sigma_1g^T=\Sigma_2$. Hence under the action of $\mathbb{O}(d)$, all covariance matrices with the same singular values fall in the same equivalence class. It is easy to see that this is also a necessary condition. These equivalence classes naturally form a bijection with a sequence of $d$ decreasing positive real numbers bounded above by $B$. The continuity of this bijection and it's inverse can easily be checked using the standard basis for both the topologies.

\subsubsection{Proof of Theorem~\ref{thm:entropy_estimation_invariance}}
Let $P,Q$ be any two distributions on $d$ elements $\{1,\dots d\}$ such that $\exists g\in S_d$ s.t. $gP=Q$. They are indistinguishable from the samples they generate. Since the entropy is defined as
\[f(P)=-\sum\limits_{i=1}^dp_i\log(p_i)\]
it doesn't depend upon the ordering of the individual probabilites. Hence the statistical game is invariant under the action of $S_d$.

Since using a permutation we can always order a given set of probabilities in decreasing order, there is a natural bijection between the quotient space and the given space. The continuity of this map and it's inverse can easily be checked using the standard basis for both the topologies.

\subsubsection{Mixture of Gaussians}
In the problem of mixture of Gaussians we are given $n$ samples $X_1,\ldots,X_n\in\mathbb{R}^d$ which come from a mixture distribution of $k$ Gaussians with different means
\[
P_{\theta} = \sum_{i=1}^k p_i \mathcal{N}(\theta_i, \Sigma_i).
\]
We assume that all $k$ Gaussians have the same covariance, let's say identity, and we also assume that we know the mixture probabilities. Finally, we assume that the mean vectors $\theta_i$ are such that $\|\theta_i\|\leq B$. Under this setting we want to estimate the $k$ different means while minimizing the sum of the $L_2^2$ losses of all the estimates of the mean parameters.

We will show the invariance of this statistical game under the action of the group $G=\mathbb{O}(d)\times\mathbb{O}(d-1)\times\ldots\times\mathbb{O}(d-k+1)$. But first we describe an element in the group and it's operation on the parameter and sample space.

An element of $g\in G$ is made up of a sequence of $k$ orthonormal matrices ($g_1,\ldots,g_k$) such that for a given set of parameters $\theta=(\theta_1,\ldots,\theta_k)\in\mathbb{R}^{d\times k}$ (where each $\theta_i\in\mathbb{R}^d$) the matrix $g_i$ leaves the first $(i-1)$ parameters unchanged, i.e. for $j=1,\ldots,i-1~g_i\theta_j=\theta_j$. Hence the $i$th orthonormal matrix has $(d-i+1)$ degrees of freedom and can be viewed as an element in $\mathbb{O}(d-i+1)$.

The action of g on $\theta$ is defined as
\begin{align*}
g\theta &= g(\theta_1,\ldots,\theta_k)\\
&=(g\theta_1,\ldots,g\theta_k)\\
&=(g_k\ldots g_1\theta_1,\ldots,g_k\ldots g_1\theta_1)\\
&=(g_1\theta_1,\ldots,g_i\ldots g_1\theta_i,\ldots,g_k\ldots g_1\theta_k)
\end{align*}

where the last equality follows from the definition of our group. The group acts in a similar manner on the sample space, i.e., for an $X\in\mathcal{X}~gX=g_k\ldots g_1X$.

\begin{theorem}
\label{thm:gaussian_mixtures_invariance}
The statistical game defined by mixture of $k$-Gaussians with identity covariance and known mixture probabilities under $L_2^2$ loss is invariant under the action of the group $\mathbb{O}(d)\times\mathbb{O}(d-1)\times\ldots\times\mathbb{O}(d-k+1)$. Moreover, the quotient space is homeomorphic to $(0,B]^k\times[0,\pi]^{k\choose 2}$.
\end{theorem}
\begin{proof}
First we show the invariance of the mixture distribution $P_\theta=\sum_i p_i\mathcal{N}(\theta_i,I)$, i.e., if $X\sim P_\theta$ then $gX\sim P_{g\theta}$. Note that from the proof of Theorem \ref{thm:sequence_model_invariance} it follows that for a given normal distribution $N(\Tilde{\theta}, I)$ and an orthonormal matrix $h\in\mathbb{O}(d)$ s.t. $h\Tilde{\theta}=\Tilde{\theta}$ if $X\sim N(\Tilde{\theta}, I)$ then $hX\sim N(h\Tilde{\theta}, I)=N(\Tilde{\theta}, I)$. The invariance of $P$ follows directly from this by substituting each $\|X-\theta_i\|^2$ in the pdf with $\|g_k\ldots g_1X-g_k\ldots g_1\theta_i\|^2$ and the definition of the group. The $L_2^2$ loss is trivially invariant and hence we establish the invariance of the statistical game.

Now, notice that for any two given parameters $\theta=(\theta_1,\ldots,\theta_k),\phi=(\phi_1,\ldots,\phi_k)\in\mathbb{R}^{dk}$ if we have the property that $\forall i~\|\theta_i\|=\|\phi_i\|$ and $\forall i,j~\theta_i^T\theta_j=\phi_i^T\phi_j$ then we can find orthonormal matrices $g_1,\ldots,g_k$ s.t. $\forall i~g_i\ldots g_1\theta_i=\phi_i$. This follows from the following inductive argument: Assume we have $g_1,\ldots,g_{i-1}$ which satisfy the given constraints. Consider $\theta'=g_{i-1}\ldots g_1\theta_i$. We have $\forall j=1,\dots,i-1~\theta'^T\phi_j=\theta_i^T\theta_j=\phi_i^T\phi_j$ because $g^T=g^{-1}$. Now if $\phi_i$ lies in the span of $\phi_1,\ldots,\phi_{i-1}$ then $\theta'=\phi_i$ and we can pick $g_i$ to be any orthonormal matrix which doesn't transform this spanned space. Otherwise, we can pick an orthonormal matrix which rotates the axis orthogonal to the spanned subspace and in the direction of the high component of $\theta'$ to the corresponding axis for $\phi_i$. This completes the desired construction.

It is easy to see that given $\theta,\phi,g$ which satisfy $g\theta=\phi$, we have $\forall i~\|\theta_i\|=\|\phi_i\|$ and $\forall i,j~\theta_i^T\theta_j=\phi_i^T\phi_j$. Hence the equivalence classes are defined uniquely by the norms of the individual gaussians and the angles between them, since there are $k$ different norms and $k\choose 2$ many angles we can establish a bijection between the quotient space and $(0,B]^k\times[0,\pi]^{k\choose 2}$. The continuity of this map and it's inverse can easily be checked using the standard basis for both the topologies.
\end{proof}

\section{Finite Gaussian Sequence Model}
\label{sec:appx_mean_estimation}
\subsection{Proof of Proposition~\ref{prop:mean_estimation_min_closed_form}}
In this section we derive a closed-form expression for the minimizer $\hat{\theta}_t$ of the following objective
\[
\argmin_{\hat{\theta} \in \mathcal{D}_G} \Eover{b \sim P_t}{R(\hat{\theta}, b\mathbf{e}_1)},
\]
where $\mathcal{D}_G$ is the set of deterministic estimators which are invariant to transformations of orthogonal group $\mathbb{O}(d)$.
From Lemma~\ref{lem:invariant_risk}, we know that for any invariant estimator $\hat{\theta} \in \mathcal{D}_G$ and any $g \in  \mathbb{O}(d),$ $R(\hat{\theta}, b\mathbf{e}_1) = R(\hat{\theta}, bg\mathbf{e}_1)$. So the above problem can be rewritten as follows
\[
\argmin_{\hat{\theta} \in \mathcal{D}_G} \Eover{b \sim P_t}{\Eover{\theta \sim U_b}{R(\hat{\theta}, \theta)}},
\]
where $U_{b}$ is the uniform distribution over spherical shell of radius $b$, centered at origin; that is, its density $u_b(\theta)$ is defined as
\[
u_b(\theta) \propto \begin{cases}
0, & \quad \text{if } \|\theta\|_2 \neq b\\
b^{-d+1}, & \quad \text{otherwise}
\end{cases}.
\]
The above optimization problem can be further rewritten as
\[
\argmin_{\hat{\theta} \in \mathcal{D}_G} R(\hat{\theta}, \Pi_t),
\]
where $R(\hat{\theta}, \Pi_t) \stackrel{\text{def}}{=} \Eover{\theta \sim \Pi_t}{R(\hat{\theta}, \theta)},$ and $\Pi_t$ is the distribution of a random variable $\theta$ which is generated by first sampling $b$ from $P_t$ and then generating a sample from $U_b$. Note that $\Pi_t$ is a spherically symmetric distribution.
From Lemma~\ref{lem:invariant_bayes_estimator}, we know that the Bayes estimator corresponding to any invariant prior is an invariant estimator. So the minimization over $\mathcal{D}_G$ in the above optimization problem can be replaced with minimization over the set of all estimators $\mathcal{D}$. This leads us to the following equivalent optimization problem
\[
\argmin_{\hat{\theta} \in \mathcal{D}} R(\hat{\theta}, \Pi_t).
\]
Let $\hat{\theta}_t$ be the minimizer of this equivalent problem. We now obtain an expression for $\hat{\theta}_t(X)$ in terms of modified Bessel functions.
Let $\Pi_t(\cdot|X)$ be the posterior distribution of $\theta$ given the data $X$ and let $p(X;\theta)$ be the probability density function for distribution $P_{\theta}$.
Since the risk is measured with respect to $\ell_2^2$ metric,  the Bayes estimator $\hat{\theta}_t(X)$ is given by the posterior mean  
\begin{align*}
    \hat{\theta}_t(X) & = \Eover{\theta \sim \Pi_t(\cdot|X)}{\theta}\\& = \frac{\Eover{\theta \sim \Pi_t}{\theta p(X;\theta)}}{\Eover{\theta \sim \Pi_t}{p(X;\theta)}}\\& = \frac{\Eover{b \sim P_t}{\int \theta u_{b}(\theta) p(X;\theta)d\theta}}{\Eover{b \sim P_t}{\int  u_{b}(\theta) p(X;\theta)d\theta}}\quad \text{(definition of $\Pi_t$)}\\& = \frac{\Eover{b \sim P_t}{b^{-d+1}\int_{\|\theta\|_2 = b} \theta  p(X;\theta)d\theta}}{\Eover{b \sim P_t}{b^{-d+1} \int_{\|\theta\|_2 = b}   p(X;\theta)d\theta}} \quad \text{(since $U_{b}$ is uniform on sphere)}\\
    & = \frac{\Eover{b \sim P_t}{b^{-d+1}e^{-b^2/2}\int_{\|\theta\|_2 = b} \theta  e^{\iprod{X}{\theta}}d\theta}}{\Eover{b \sim P_t}{b^{-d+1} e^{-b^2/2} \int_{\|\theta\|_2 = b}   e^{\iprod{X}{\theta}}d\theta}} \\
    & = \frac{\Eover{b \sim P_t}{b^2e^{-b^2/2}\int_{\|\theta\|_2 = 1} \theta  e^{b\iprod{X}{\theta}}d\theta}}{\Eover{b \sim P_t}{b e^{-b^2/2} \int_{\|\theta\|_2 = 1}   e^{b\iprod{X}{\theta}}d\theta}}\quad \text{(change of variables)} .
\end{align*}
We now obtain a closed-form expression for the terms $\int_{\|\theta\|_2 = 1}  \theta e^{b\iprod{X}{\theta}}d\theta$ and $\int_{\|\theta\|_2 = 1}  e^{b\iprod{X}{\theta}}d\theta$ appearing in the RHS of the above equation. We do this by relating them to the mean and normalization constant of Von Mises-Fisher (vMF) distribution, which is a probability distribution on the unit sphere centered at origin in $\mathbb{R}^d$. This distribution is usually studied in directional statistics~\citep{mardia2009directional}. The probability density function of a random unit vector $Z\in\mathbb{R}^d$ distributed according to vMF distribution is given by
\[
p(Z; \mu, \kappa) = \frac{\kappa^{d/2-1}}{(2\pi)^{d/2}I_{d/2-1}(\kappa)}\exp(\kappa \iprod{\mu}{Z}),
\]
where $\kappa \geq 0,\|\mu\|_2 = 1$, $I_{\nu}$ is the modified Bessel function of the first kind of order $\nu$.
Using the fact that a probability density function integrates to $1$, we get the following closed-form expression for $\int_{\|\theta\|_2 = 1}  e^{b\iprod{X}{\theta}}d\theta$
\begin{equation}
\label{eqn:von_mises_norm_const}
    \int_{\|\theta\|_2 = 1}  e^{b\iprod{X}{\theta}}d\theta = \frac{(2\pi)^{d/2}I_{d/2-1}(b\|X\|_2)}{(b\|X\|_2)^{d/2-1}}.
\end{equation}
To get a closed-form expression for $\int_{\|\theta\|_2 = 1}  \theta e^{b\iprod{X}{\theta}}d\theta$, we relate it to mean of vMF distribution. We have the following expression for the mean of a random vector distributed according to vMF distribution~\citep{banerjee2005clustering}
\begin{align*}
    \int_{\|Z\|=1} Zp(Z; \mu, \kappa)dZ & =  \frac{I_{d/2}(\kappa)}{I_{d/2-1}(\kappa)}\mu.
\end{align*}
Using the above equality, we get the following expression for $\int_{\|\theta\|_2 = 1}  \theta e^{b\iprod{X}{\theta}}d\theta$
\begin{equation}
\label{eqn:von_mises_mean}
\int_{\|\theta\|_2 = 1}  \theta e^{b\iprod{X}{\theta}}d\theta = \frac{(2\pi)^{d/2}I_{d/2}(b\|X\|_2)}{(b\|X\|_2)^{d/2-1}}\frac{X}{\|X\|_2}.\end{equation}
Substituting Equations~\eqref{eqn:von_mises_norm_const},~\eqref{eqn:von_mises_mean} in the expression for $\hat{\theta}_t(X)$ obtained above, we get an expression for $\hat{\theta}_t(X)$ which involves  the modified Bessel function $I_{\nu}$ and integrals over variable $b$. We note that $I_{\nu}$ can be computed to very high accuracy and there exist accurate implementations of $I_{\nu}$ in a number of programming languages. So in our analysis of the approximation error of Algorithm~\ref{alg:mean_estimation_min_oracle}, we assume the error from the computation of $I_{\nu}$ is $0$. 
\subsection{Proof of Theorem~\ref{thm:mean_estimation}}
Before we present the proof of the Theorem we present useful intermediate
results which we require in our proof.
\subsubsection{Intermediate Results}
\begin{lemma}[Lipschitz Continuity]
\label{lem:risk_lipscthiz}
Consider the problem of finite Gaussian sequence model. Let \mbox{$\Theta = \{\theta: \theta \in \mathbb{R}^d, \|\theta\|_2 \leq B\}$} be the ball of radius $B$ centered at origin in $\mathbb{R}^d$. Let $\hat{\theta}$ be any estimator which maps $X$ to an element in $\Theta$. Then the risk \mbox{$R(\hat{\theta}, \theta) = \Eover{X \sim \mathcal{N}(\theta, I)}{\|\hat{\theta}(X)-\theta\|_2^2}$} is Lipschitz continuous in its second argument w.r.t $\ell_2$ norm over the domain $\Theta$, with Lipschitz constant $4(B+\sqrt{d}B^2)$. Moreover, \mbox{$R(\hat{\theta}, b\mathbf{e}_1) = \Eover{X \sim \mathcal{N}(\theta, I)}{\|\hat{\theta}(X)-b\mathbf{e}_1\|_2^2}$} is Lipschitz continuous in $b$ over the domain $[0,B]$, with Lipschitz constant $4(B+B^2)$.
\end{lemma}
\begin{proof}
Let $R_{\hat{\theta}}(\theta) = R(\hat{\theta}, \theta)$. The gradient of $R_{\hat{\theta}}(\theta)$ with respect to $\theta$ is given by
\[
\nabla_{\theta} R_{\hat{\theta}}(\theta) = \Eover{X \sim \mathcal{N}(\theta, I)}{2(\theta-\hat{\theta}(X)) + (X-\theta)\|\hat{\theta}(X)-\theta\|_2^2}.
\]
The norm of $\nabla_{\theta} R_{\hat{\theta}}(\theta)$ can be upper bounded as follows
\begin{align*}
    \|\nabla_{\theta} R_{\hat{\theta}}(\theta)\|_2 & \leq \Big|\Big|\Eover{X \sim \mathcal{N}(\theta, I)}{2(\theta-\hat{\theta}(X))}\Big|\Big|_2 + \Big|\Big|\Eover{X \sim \mathcal{N}(\theta, I)}{ (X-\theta)\|\hat{\theta}(X)-\theta\|_2^2}\Big|\Big|_2\\
    & \stackrel{(a)}{\leq} 4B + \Eover{\x \sim \mathcal{N}(\theta, I)}{ \|X-\theta\|_2\|\hat{\theta}(X)-\theta\|_2^2}\\
    & \stackrel{(b)}{\leq} 4B + 4B^2\Eover{X \sim \mathcal{N}(\theta, I)}{ \|X-\theta\|_2}\\
    & \leq 4B + 4\sqrt{d}B^2,
\end{align*}
where the first term in $(a)$ follows from the fact that $\theta, \hat{\theta}(X) \in \Theta$ and  the second term follows from Jensen's inequality. This shows that $R_{\hat{\theta}}(\theta)$ is Lipschitz continuous over $\Theta$.  This finishes the first part of the proof. To show that $R(\hat{\theta}, b\mathbf{e}_1)$ is Lipschitz continuous in $b$, we use similar arguments. Let $R_{\hat{\theta}}(b) = R(\hat{\theta}, b\mathbf{e}_1)$. Then
\begin{align*}
    \Big| R_{\hat{\theta}}'(b)\Big| & = \Big|\left\langle\mathbf{e}_1, \nabla_{\theta} R_{\hat{\theta}}(\theta)\Big|_{\theta = b\mathbf{e}_1}\right\rangle\Big|\\
    &\stackrel{(a)}{\leq} \Big|\Eover{X \sim \mathcal{N}(b\mathbf{e}_1, I)}{2(b-[\hat{\theta}(X)]_1)}\Big| + \Big|\Big|\Eover{X \sim \mathcal{N}(b\mathbf{e}_1, I)}{ (X_1-b)\|\hat{\theta}(X)-b\mathbf{e}_1\|_2^2}\Big|\Big|_2\\
    & \leq 4B + \Eover{\x \sim \mathcal{N}(b\mathbf{e}_1, I)}{ |X_1-b|\|\hat{\theta}(X)-b\mathbf{e}_1\|_2^2}\\
    & \leq 4B + 4B^2\Eover{X \sim \mathcal{N}(b\mathbf{e}_1, I)}{ |X_1-b|}\\
    & \leq 4B + 4B^2,
\end{align*}
where $(a)$ follows from the expression for $\nabla_{\theta} R_{\hat{\theta}}(\theta)$ obtained above.
\end{proof}

\begin{lemma}[Approximation of risk]
\label{lem:risk_approximation}
Consider the setting of Lemma~\ref{lem:risk_lipscthiz}. Let $\hat{\theta}$ be any estimator which maps $X$ to an element in $\Theta$. Let $\{X_i\}_{i=1}^N$ be $N$ i.i.d samples from $\mathcal{N}(\theta, I)$. Then with probability at least $1-\delta$
\[
\Big| \frac{1}{N}\sum_{i=1}^N\|\hat{\theta}(X_i)-\theta\|_2^2 - R_{\hat{\theta}}(\theta) \Big| \leq 4B^2\sqrt{\frac{\log{\frac{1}{\delta}}}{N}}.
\]
\end{lemma}
\begin{proof}
The proof of the Lemma relies on concentration properties of sub-Gaussian random variables. Let $Z(X) = \|\hat{\theta}(X)-\theta\|^2$. Note that $R_{\hat{\theta}}(\theta) = \Eover{X\sim \mathcal{N}(\theta, I)}{Z(X)}$. Since $Z(X)$ is bounded by $4B^2$, it is a sub-Gaussian random variable. Using Hoeffding bound we get
\[
\Big|\frac{1}{N}\sum_{i=1}^N Z(X_i)- \E{Z(X)}\Big| \leq 4B^2\sqrt{\frac{\log{\frac{1}{\delta}}}{N}}, \quad \text{w.p} \geq 1-\delta.
\]
\end{proof}
\subsubsection{Main Argument}
The proof relies on Corollary~\ref{cor:ftpl_minimax_estimators} to show that the averaged estimator $\hat{\theta}_{\AVG}$ is approximately minimax and $\hat{P}_{\text{LFP}}$ is approximately least favorable. Here is a rough sketch of the proof. We first apply the corollaries on the following reduced statistical game that we are aiming to solve
\begin{equation*}
     \inf_{\hat{\theta} \in \mathcal{D}_G} \sup_{b \in [0, B]}  R(\hat{\theta}, b\mathbf{e}_1).
 \end{equation*}
To apply these corollaries, we need the risk $R(\hat{\theta}, b\mathbf{e}_1)$ to be Lipscthiz continuous in $b$. This holds for us because of Lemma~\ref{lem:risk_lipscthiz}. 
Next, we convert the guarantees for the reduced statistical game to the orginial statistical game to show that we learn a minimax estimator and LFP for finite Gaussian sequence model.

To use Corollary~\ref{cor:ftpl_minimax_estimators}, we first need to bound $\alpha, \beta, \alpha'$, the approximation errors of the optimization subroutines described in Algorithms~\ref{alg:mean_estimation_max_oracle},~\ref{alg:mean_estimation_min_oracle}. A major part of the proof involves bounding these quantities.
\paragraph{Approximation error of Algorithm~\ref{alg:mean_estimation_max_oracle}.} There are two causes for error in the optimization oracle described in Algorithm~\ref{alg:mean_estimation_max_oracle}: (a) grid search and (b)  approximate computation of risk $R(\hat{\theta}, b\mathbf{e}_1)$. We now bound the  error due to both (a) and (b). 
From Lemma~\ref{lem:risk_approximation} we know that for any estimator $\hat{\theta}_i$ and grid point $b_j$, the following holds with probability at least $1-\delta$
    \[
    \Big| \frac{1}{N_1}\sum_{k=1}^{N_1} \|\hat{\theta}_i(X_k)-b_j\mathbf{e}_1\|_2^2 - R(\hat{\theta}_i, b_j\mathbf{e}_1)\Big| \leq 4B^2\sqrt{\frac{\log{\frac{1}{\delta}}}{N_1}}.
    \]
    Taking a union bound over all estimators $\{\hat{\theta}_i\}_{i=1}^T$  and grid points $\{b_j\}_{j=1}^{B/w}$, we can show that with probability at least $1-\delta$, the following holds for all $i\in [T],j\in [B/w]$ 
    \begin{equation}
    \label{eqn:aux_concentration_risk}
    \Big| \frac{1}{N_1}\sum_{k=1}^{N_1} \|\hat{\theta}_i(X_k)-b_j\mathbf{e}_1\|_2^2 - R(\hat{\theta}_i, b_j\mathbf{e}_1)\Big| \leq 4B^2\sqrt{\frac{\log{\frac{BT}{w\delta}}}{N_1}}.
    \end{equation}

Let $f_{t,\sigma}(b)$ be the actual objective we would like to optimize in iteration $t$ of Algorithm~\ref{alg:ftpl_stat_games}, which is given by
\[
f_{t,\sigma}(b) = \sum_{i = 1}^{t-1} R(\hat{\theta}_i, b\mathbf{e}_1) + \sigma b.
\]
Let $\hat{f}_{t,\sigma}(b)$ be the approximate objective we are optimizing by replacing $R(\hat{\theta}_i, b\mathbf{e}_1)$ with its approximate estimate. Let $b^*_t$ be a maximizer of $f_{t,\sigma}(b)$ and $b^*_{t,\text{approx}}$ be the maximizer of $\hat{f}_{t,\sigma}(b)$ (which is also the output of Algorithm~\ref{alg:mean_estimation_max_oracle}). Finally, let $b^*_{t,\text{NN}}$ be the point on the grid which is closest to $b^*_t$. Using Lemma~\ref{lem:risk_lipscthiz} we first show that  $f_{t,\sigma}(b)$ is Lipschitz continuous in $b$. The derivative of $f_{t,\sigma}(b)$ with respect to $b$ is given by
\begin{align*}
    f_{t,\sigma}'(b) & = \sum_{i = 1}^{t-1} \left\langle \mathbf{e}_1,\nabla_{\theta} R(\hat{\theta}_i, \theta)\Big|_{\theta = b\mathbf{e}_1}\right\rangle + \sigma 
\end{align*}
Using Lemma~\ref{lem:risk_lipscthiz}, the magnitude of $f_{t,\sigma}'(b)$ can be upper bounded as
\[
|f_{t,\sigma}'(b)| \leq 4(t-1)(B+B^2) + \sigma.
\]
This shows that $f_{t,\sigma}(b)$ is Lipschitz continuous in $b$. We now bound $f_{t,\sigma}(b^*_t) - f_{t,\sigma}(b^*_{t,\text{approx}})$, the approximation error of the optimization oracle
\begin{align*}
    f_{t,\sigma}(b^*_t) & \stackrel{(a)}{\leq } f_{t,\sigma}(b^*_{t,\text{NN}}) + \left(4t(B+B^2) + \sigma\right) w\\
    & \stackrel{(b)}{\leq} \hat{f}_{t,\sigma}(b^*_{t,\text{NN}}) +  4tB^2\sqrt{\frac{\log{\frac{BT}{w\delta}}}{N_1}} + \left(4t(B+B^2) + \sigma\right) w\\
    & \stackrel{(c)}{\leq} \hat{f}_{t,\sigma}(b^*_{t,\text{approx}}) +  4tB^2\sqrt{\frac{\log{\frac{BT}{w\delta}}}{N_1}} + \left(4t(B+B^2) + \sigma\right) w \\
    & \stackrel{(d)}{\leq} f_{t,\sigma}(b^*_{t,\text{approx}}) +  8tB^2\sqrt{\frac{\log{\frac{BT}{w\delta}}}{N_1}} + \left(4t(B+B^2) + \sigma\right) w,
\end{align*}
where $(a)$ follows from Lipschitz property of the loss function and $(b), (d)$ follow from Equation~\eqref{eqn:aux_concentration_risk} and hold with probability at least $1-\delta$ and $(c)$ follows from the optimality of $b^*_{t,\text{approx}}$.
This shows that Algorithm~\ref{alg:mean_estimation_max_oracle} is a $\left(O\left(TB^2\sqrt{\frac{\log{\frac{BT}{w\delta}}}{N_1}} + TB(1+B)w \right),w\right)$-approximate maximization oracle; that is
\[
\alpha = O\left(TB^2\sqrt{\frac{\log{\frac{BT}{w\delta}}}{N_1}} + TB(1+B)w\right),\quad \beta = w.
\]
\paragraph{Approximation error of
Algorithm~\ref{alg:mean_estimation_min_oracle}.} There are two sources of approximation error in Algorithm~\ref{alg:mean_estimation_min_oracle}: (a) computation of modified Bessel functions $I_{\nu}$, and (b) approximation of $P_t$ with its samples. In this analysis we assume that $I_{\nu}$ can be computed to very high accuracy. This is a reasonable assumption because many programming languages have accurate and efficient implementations of $I_{\nu}$. So the main focus here is on bounding the error from approximation of $P_t$.

First, note that since we are using grid search to optimize the maximization problem, the true distribution $P_t$ for which we are supposed to compute the Bayes estimator is a discrete distribution supported on grid points $\{b_1,\dots b_{B/w}\}$. Algorithm~\ref{alg:mean_estimation_min_oracle} does not compute the Bayes estimator for $P_t$. Instead, we generate samples from $P_t$ and use them as a proxy for $P_t$. Let $\hat{P}_t$ be the empirical distribution obtained by sampling $N_2$ points from $P_t$. Let $p_{t,j}$ be the probability mass on grid point $b_j$. Using Bernstein inequality we can show that the following holds with probability at least $1-\delta$ 
\begin{equation}
\label{eqn:aux_concentration_prior}
    \forall j \in [B/w]\quad |\hat{p}_{t,j}-p_{t,j}| \leq \sqrt{p_{t,j}\frac{\log{\frac{B}{w\delta}}}{N_2}}.
\end{equation}
    Define estimators $\hat{\theta}_t', \hat{\theta}_{t}$ as 
    \[
    \hat{\theta}_t' \leftarrow \argmin_{\hat{\theta} \in \mathcal{D}_G} \Eover{b \sim P_t}{R(\hat{\theta}, b\mathbf{e}_1)}, \quad \hat{\theta}_{t} \leftarrow \argmin_{\hat{\theta} \in \mathcal{D}_G} \Eover{b \sim \hat{P}_t}{R(\hat{\theta}, b\mathbf{e}_1)}.
    \]
    $\hat{\theta}_t'$ is what we ideally want to compute. $\hat{\theta}_{t}$ is what we end up computing using Algorithm~\ref{alg:mean_estimation_min_oracle}.
    We now show that $\hat{\theta}_{t}$ is an approximate minimizer of the left hand side optimization problem above. To this end, we try to bound the following quantity
    \begin{equation*}
        \Eover{b \sim P_t}{R(\hat{\theta}_{t}, b\mathbf{e}_1) - R(\hat{\theta}_{t}', b\mathbf{e}_1)}.
    \end{equation*}
    Let $f_{t}(\hat{\theta}) = \Eover{b \sim P_t}{R(\hat{\theta}, b\mathbf{e}_1)}$ and $\hat{f}_{t}(\hat{\theta}) = \Eover{b \sim \hat{P}_t}{R(\hat{\theta}, b\mathbf{e}_1)}$. We would like to bound the quantity $f_t(\hat{\theta}_{t}) - f_t(\hat{\theta}_{t}')$. Consider the following 
    \begin{align*}
        f_t(\hat{\theta}_{t}) & \stackrel{(a)}{\leq} \hat{f}_t(\hat{\theta}_{t}) + \frac{4B^3}{w}\sqrt{\frac{\log{\frac{B}{w\delta}}}{N_2}}\\
        & \stackrel{(b)}{\leq} \hat{f}_t(\hat{\theta}_{t}') + \frac{4B^3}{w}\sqrt{\frac{\log{\frac{B}{w\delta}}}{N_2}}\\
        & \stackrel{(c)}{\leq }f_t(\hat{\theta}_{t}') + \frac{8B^3}{w}\sqrt{\frac{\log{\frac{B}{w\delta}}}{N_2}},
    \end{align*}
    where $(a)$ follows from Equation~\eqref{eqn:aux_concentration_prior} and the fact that the risk $R(\hat{\theta},\theta)$ of any estimator is bounded by $4B^2$, $(b)$ follows since $\hat{\theta}_{t}$ is a minimizer of $\hat{f}_t$ and $(c)$ follows from Equation~\eqref{eqn:aux_concentration_prior}. This shows that with probability at least $1-\delta$, Algorithm~\ref{alg:mean_estimation_min_oracle} is an $O\left(\frac{B^3}{w}\sqrt{\frac{\log{\frac{B}{w\delta}}}{N_2}}\right)$-approximate optimization oracle; that is,
    \[
    \alpha' = O\left(\frac{B^3}{w}\sqrt{\frac{\log{\frac{B}{w\delta}}}{N_2}}\right).
    \]
    
\paragraph{Minimax Estimator.} We are now ready to show that $\hat{\theta}_{\AVG}$ is an approximate minimax estimator. Instantiating Corollary~\ref{cor:ftpl_minimax_estimators} for the reduced statistical game gives us the following bound, which holds with probability at least $1-\delta$
    \[
    \sup_{b \in [0,B]} R(\hat{\theta}_{\AVG},b\mathbf{e}_1) \leq \inf_{\hat{\theta} \in \mathcal{D}_G}\sup_{b \in [0,B]}R(\hat{\theta},b\mathbf{e}_1)  + \Tilde{O}\left(\frac{B^2(B+1)}{\sqrt{T}} + \alpha  +\alpha' + \beta B(B+1) \sqrt{T}\right),
    \]
    where we used the fact that the risk $R(\hat{\theta}, b\mathbf{e}_1)$ is $4B(B+1)$-Lipschitz continuous w.r.t $b$. The $\Tilde{O}$ notation in the above inequality hides logarithmic factors. Plugging in the values of $\alpha, \alpha',\beta$ in the above equation gives us 
    \[
    \sup_{b \in [0,B]} R(\hat{\theta}_{\AVG},b\mathbf{e}_1) \leq \inf_{\hat{\theta} \in \mathcal{D}_G}\sup_{b \in [0,B]}R(\hat{\theta},b\mathbf{e}_1)  + \Tilde{O}\left(\frac{B^2(B+1)}{\sqrt{T}}\right).
    \]
    We now convert this bound to a bound on the original statistical game. From Theorem~\ref{thm:ftpl_minimax_invariance} we know that $\inf_{\hat{\theta} \in \mathcal{D}_G}\sup_{b \in [0,B]}R(\hat{\theta},b\mathbf{e}_1) = \inf_{\hat{\theta} \in \mathcal{D}}\sup_{\theta \in \Theta}R(\hat{\theta},\theta) = R^*$. Since the estimator $\hat{\theta}_{\AVG}$ is invariant to transformations of orthogonal group, we have $R(\hat{\theta}_{\AVG},\theta) = R(\hat{\theta}_{\AVG},\|\theta\|_2\mathbf{e}_1)$ for any $\theta \in \Theta$. Using these two results in the above inequality, we get 
    \[
    \sup_{\theta \in \Theta} R(\hat{\theta}_{\AVG},\theta) = \sup_{b \in [0,B]} R(\hat{\theta}_{\AVG},b\mathbf{e}_1) \leq R^*  + \Tilde{O}\left(\frac{B^2(B+1)}{\sqrt{T}}\right).
    \]
    This shows that the worst-case risk of $\hat{\theta}_{\AVG}$ is close to the minimax risk $R^*$. This finishes the first part of the proof.
    \paragraph{LFP.} To prove the second part, we rely on Corollary~\ref{cor:ftpl_minimax_estimators}. Instantiating it for the reduced statistical game gives us
    \[
    \inf_{\hat{\theta} \in \mathcal{D}_G} \frac{1}{T}\sum_{t=1}^T \Eover{b\sim P_t}{R(\hat{\theta},b\mathbf{e}_1 )} \geq R^* -   \Tilde{O}\left(\frac{B^2(B+1)}{\sqrt{T}} + \alpha  +\alpha' + \beta B(B+1) \sqrt{T}\right).
    \]
    Plugging in the values of $\alpha, \alpha',\beta$ in the above equation gives us
    \[
    \inf_{\hat{\theta} \in \mathcal{D}_G} \frac{1}{T}\sum_{t=1}^T \Eover{b\sim P_t}{R(\hat{\theta},b\mathbf{e}_1 )} \geq R^* - \Tilde{O}\left(\frac{B^2(B+1)}{\sqrt{T}}\right).
    \]
    From Equation~\eqref{eqn:aux_concentration_prior} we know that $P_t$ is close to $\hat{P}_t$ with high probability. Using this, we can replace $P_t$ in the above bound with $\hat{P}_t$ and obtain the following bound, which holds with probability at least $1-\delta$
    \begin{align}
    \label{eqn:lfp}
        \inf_{\hat{\theta} \in \mathcal{D}_G} \frac{1}{T}\sum_{t=1}^T \Eover{b\sim \hat{P}_t}{R(\hat{\theta},b\mathbf{e}_1 )} \geq R^* - \Tilde{O}\left(\frac{B^2(B+1)}{\sqrt{T}}\right).
    \end{align}
    In the rest of the proof, we show that $\inf_{\hat{\theta} \in \mathcal{D}_G} \frac{1}{T}\sum_{t=1}^T \Eover{b\sim \hat{P}_t}{R(\hat{\theta},b\mathbf{e}_1 )} = \inf_{\hat{\theta}} R(\hat{\theta}, \hat{P}_{\text{LFP}})$. Recall, the density function of $\hat{P}_{\text{LFP}}$ is given by: $\hat{p}_{\text{LFP}}(\theta) \propto \|\theta\|^{1-d}_2\hat{P}_{\AVG}(\|\theta\|_2) $, where $\hat{P}_{\AVG}(\|\theta\|_2)$ is the probability  mass placed by $\hat{P}_{\AVG}$ at $\|\theta\|_2$. This distribution is equivalent to the distribution of a random variable which is generated by first sampling $b$ from $\hat{P}_t$ and then sampling $\theta$ from the uniform distribution on $(d-1)$ dimensional  sphere of radius $b$, centered at origin in $\mathbb{R}^{d}$. Using this equivalence, we can equivalently rewrite $R(\hat{\theta}, \hat{P}_{\text{LFP}})$ for any estimator $\hat{\theta}$ as
    \[
    R(\hat{\theta}, \hat{P}_{\text{LFP}}) = \frac{1}{T}\sum_{t=1}^T\Eover{b\sim \hat{P}_t}{\Eover{\theta \sim U}{R(\hat{\theta}, b\theta)}},
    \]
    where $U$ is the uniform distribution on the $(d-1)$ dimensional unit sphere centered at origin, in $\mathbb{R}^{d}$. Next, from Lemma~\ref{lem:invariant_bayes_estimator}, we know that the Bayes estimator corresponding to any invariant prior is an invariant estimator. Since  $\hat{P}_{\text{LFP}}$ is an invariant distribution, we have
    \[
    \inf_{\hat{\theta}\in \mathcal{D}} R(\hat{\theta}, \hat{P}_{\text{LFP}}) = \inf_{\hat{\theta}\in \mathcal{D}_G} R(\hat{\theta}, \hat{P}_{\text{LFP}}) = \inf_{\hat{\theta}\in \mathcal{D}_G}\frac{1}{T}\sum_{t=1}^T\Eover{b\sim \hat{P}_t}{\Eover{\theta \sim U}{R(\hat{\theta}, b\theta)}}.
    \]
    From Lemma~\ref{lem:invariant_risk} we know that for any invariant estimator $\hat{\theta},$ we have $R(\hat{\theta},\theta_1) = R(\hat{\theta}, \theta_2)$, whenever $\theta_1\sim\theta_2$. Using this result in the above equation gives us
    \[
    \inf_{\hat{\theta}\in \mathcal{D}} R(\hat{\theta}, \hat{P}_{\text{LFP}}) = \inf_{\hat{\theta}\in \mathcal{D}_G}\frac{1}{T}\sum_{t=1}^T\Eover{b\sim \hat{P}_t}{R(\hat{\theta}, b\mathbf{e}_1)}.
    \]
    Combining the above result with Equation~\eqref{eqn:lfp} shows that $\hat{P}_{\text{LFP}}$ is approximately least favorable.

\subsection{Loss on few co-ordinates}
\label{sec:sequence_model_few}
In this section, we present the optimization oracles for the problem of finite Gaussian sequence model, when the loss is evaluated on a few co-ordinates.  Recall, in  Theorem~\ref{thm:sequence_model_invariance_few} we showed that the original min-max statistical game can be reduced to the following simpler problem
\begin{equation}
        \inf_{\hat{\theta} \in \mathcal{M}_{\mathcal{D},G}} \sup_{b:b[1]^2+b[2]^2\leq B^2}  R(\hat{\theta}, [b[1]\mathbf{e}_{1,k}, b[2]\mathbf{e}_{1,d-k}]),
\end{equation}
where $b[j]$ represents the $j^{th}$ co-ordinate of $b$.
 We now provide efficient implementations of the optimization oracles required by Algorithm~\ref{alg:ftpl_stat_games} for finding a Nash equilibrium of this game. 
 The optimization problems corresponding to the two optimization oracles are as follows
\begin{align*}
    &\hat{\theta}_{t} \leftarrow \argmin_{\hat{\theta} \in \mathcal{D}_G} \Eover{b \sim P_t}{R(\hat{\theta}, [b[1]\mathbf{e}_{1,k}, b[2]\mathbf{e}_{1,d-k}])},\\
    &b_{t}(\sigma) \leftarrow \argmax_{b:b[1]^2+b[2]^2\leq B^2} \sum_{i = 1}^{t-1} R(\hat{\theta}_i, [b[1]\mathbf{e}_{1,k}, b[2]\mathbf{e}_{1,d-k}]) + \iprod{\sigma}{b},
\end{align*}
where $\mathcal{D}_G$ is the set of deterministic invariant estimators and $P_t$ is the distribution of random variable $b_{t}(\sigma)$. 
The maximization oracle can be efficiently implemented via a grid search over $\{b:b[1]^2+b[2]^2\leq B^2\}$ (see Algorithm~\ref{alg:mean_estimation_few_max_oracle}). The minimization oracle can also be efficiently implemented. The minimizer has a closed form expression which depends on $P_t$ and modified Bessel functions (see Algorithm~\ref{alg:mean_estimation_few_min_oracle}).
 
\begin{algorithm}[H]
\caption{Maximization Oracle}
\label{alg:mean_estimation_few_max_oracle}
\begin{algorithmic}[1]
  \small
  \State \textbf{Input:} Number of coordinates to evaluate loss on $k$, estimators $\{\hat{\theta}_{i}\}_{i=1}^{t-1}$, perturbation $\sigma$, grid width $w$, number of samples for computation of expected risk $R(\hat{\theta},\theta)$: $N_1$
  \State Let $\{b_1,b_2\dots b_{N(w)}\}$ be the $w$-covering of $\{b:b[1]^2+b[2]^2\leq B^2\}$
  \For{$j = 1 \dots N(w)$}
  \For{$i = 1\dots t-1$}
  \State Generate $N_1$ independent samples $\{X_{l}\}_{l=1}^{N_1}$ from the following distribution $$ \mathcal{N}([b_j[1]\mathbf{e}_{1,k}, b_j[2]\mathbf{e}_{1,d-k}], I)$$
  \State Estimate $R(\hat{\theta}_i, [b_j[1]\mathbf{e}_{1,k}, b_j[2]\mathbf{e}_{1,d-k}])$ as 
  $$\frac{1}{N_1}\sum_{l=1}^{N_1} \|\hat{\theta}_i(X_l)[1:k]-b_j[1]\mathbf{e}_{1,k}\|_2^2.$$
  \EndFor
  \State Evaluate the objective at $b_j$ using the above estimates
  \EndFor
  \State \textbf{Output:} $b_j$ which maximizes the objective
\end{algorithmic}
\end{algorithm}
\vspace{-0.05in}
\begin{algorithm}[H]
	\caption{Minimization Oracle}
	\label{alg:mean_estimation_few_min_oracle}
	\begin{algorithmic}[1]
		\small
		\State \textbf{Input:} Samples $\{b_i\}_{i=1}^{N_2}$ generated from distribution $P_t$, number of coordinates to evaluate loss on $k$.
		\State For any $X$, compute $\hat{\theta}_t(X)$ as
		\[
		\left(\frac{ \sum_{i=1}^{N_2} w_i b_i[1]A_k(b_i[1]\|X[1:k]\|_2)}{\sum_{i=1}^{N_2}w_i }\right)\frac{X[1:k]}{\|X[1:k]\|_2},
		\]
		 where $A_k(\gamma)=\dfrac{I_{k/2}(\gamma)}{I_{k/2-1}(\gamma)}$, $$w_i=b_i[1]^{2-\frac{k}{2}}b_i[2]^{2-\frac{d-k}{2}}e^{-\frac{\|b\|^2}{2}}I_{k/2-1}(b_i[1]\|X[1:k]\|_2)I_{(d-k)/2-1}(b_i[2]\|X[k+1:d]\|_2),$$ and $I_\nu$ is the modified Bessel function of the first kind of order $\nu$.
	\end{algorithmic}
\end{algorithm}

\section{Linear Regression}
\label{sec:aux_linear_regression}
\subsection{Proof of Proposition~\ref{prop:regression_min_closed_form}}
In this section we derive a closed-form expression for the minimizer $\hat{\theta}_t$ of the following objective
\[
\argmin_{\hat{\theta} \in \mathcal{D}_G} \Eover{b \sim P_t}{R(\hat{\theta}, b\mathbf{e}_1)}.
\]
Using the same arguments as in proof of Proposition~\ref{prop:mean_estimation_min_closed_form}, we can show that the above optimization problem can be rewritten as the following equivalent optimization problem over the set of all deterministic estimators
\[
\argmin_{\hat{\theta} \in \mathcal{D}} \Eover{\theta \sim \Pi_t}{R(\hat{\theta}, \theta)},
\]
where $\Pi_t$ is the distribution of a random variable $\theta$ which is generated by first sampling a $b$ from $P_t$ and then drawing a random sample from $U_b$, the uniform distribution on a spherical shell of radius $b$.
The density function of $U_b$ is given by 
\[
u_b(\theta) \propto \begin{cases}
0, & \quad \text{if } \|\theta\|_2 \neq b\\
b^{-d+1}, & \quad \text{otherwise}
\end{cases}.
\]
Since the risk is measured with respect to $\ell_2^2$ metric,  the minimizer $\hat{\theta}_t(D_n)$ is given by the posterior mean 
\begin{align*}
    \hat{\theta}_t(D_n) & = \Eover{\theta \sim \Pi_t(\cdot|D_n)}{\theta}\\& = \frac{\Eover{\theta \sim \Pi_t}{\theta p(D_n;\theta)}}{\Eover{\theta \sim \Pi_t}{p(D_n;\theta)}}\\
    & = \frac{\Eover{b\sim P_t}{\int \theta u_{b}(\theta) p(D_n;\theta)d\theta}}{\Eover{b\sim P_t}{\int  u_{b}(\theta)p(D_n;\theta)d\theta}}\\
    & = \frac{\Eover{b\sim P_t}{b^{-d+1} \int_{\|\theta\|_2 = b} \theta p(D_n;\theta)d\theta }}{\Eover{b\sim P_t}{b^{-d+1}  \int_{\|\theta\|_2 = b}  p(D_n;\theta)d\theta}}\\
    & = \frac{\Eover{b\sim P_t}{b^{-d+1} \int_{\|\theta\|_2 = b} \theta e^{-\frac{\|\mathbf{Y}-\mathbf{X}\theta\|_2^2}{2}}d\theta}}{\Eover{b\sim P_t}{b^{-d+1}  \int_{\|\theta\|_2 = b}  e^{-\frac{\|\mathbf{Y}-\mathbf{X}\theta\|_2^2}{2}}d\theta} }\\
    & =  \frac{\Eover{b\sim P_t}{b^{2} \int_{\|\theta\|_2 = 1} \theta e^{-\frac{b^2\|\mathbf{X}\theta\|_2^2 - 2b\left\langle \theta, \mathbf{X}^T\mathbf{Y}\right\rangle }{2}}d\theta }}{\Eover{b\sim P_t}{b  \int_{\|\theta\|_2 = 1}  e^{-\frac{b^2\|\mathbf{X}\theta\|_2^2 - 2b\left\langle \theta, \mathbf{X}^T\mathbf{Y}\right\rangle }{2}}d\theta}}\quad \text{(change of variables)}.
\end{align*}
We now relate the terms appearing in the above expression to the mean and normalization constant of Fisher-Bingham (FB) distribution. As stated in Section~\ref{sec:regression}, the probability density function of a random unit vector $Z\in\mathbb{R}^d$ distributed according to FB distribution is given by
\[
p(Z;A, \gamma) = C(A, \gamma)^{-1}\exp\left(-Z^TAZ  + \iprod{\gamma}{Z}\right),
\]
where $Z \in \mathbb{S}^{d-1}$, and $\gamma \in \mathbb{R}^d$, $A \in \mathbb{R}^{d\times d}$ are the parameters of the distribution with $A$ being positive semi-definite and $C(A, \gamma)$ is the normalization constant which is given by
\[
C(A, \gamma) = \int_{\|Z\|_2 = 1} \exp\left(-Z^TAZ  + \iprod{\gamma}{Z}\right) dZ.
\]
The mean of $Z$ is given by
\[
\int_{\|Z\|_2 = 1} Z p(Z; A, \gamma)dZ = C(A, \gamma)^{-1}\int_{\|Z\|_2 = 1} Z\exp\left(-Z^TAZ  + \iprod{\gamma}{Z}\right)dZ = C(A, \gamma)^{-1}\frac{\partial}{\partial \gamma} C(A, \gamma).
\]
Using these in the  previously derived expression for $\hat{\theta}(D_n)$ gives us the required result.
\subsection{Mean and normalization constant of Fisher-Bingham distribution}
\label{sec:appx_fisher_bingham_mean}
In this section, we present our technique for computation of $C\left(A, \gamma\right)$. Once we have an accurate technique for its computation, computing $\frac{\partial}{\partial \gamma} C(A,\gamma)$ should be straight forward as one can rely on efficient numerical differentiation techniques for its computation. Recall, to implement Algorithm~\ref{alg:regression_min_oracle} we need to compute $C\left(2^{-1} b^2\mathbf{X}^T\mathbf{X}, b\mathbf{X}^T\mathbf{Y}\right)$. 
Let $\hat{\Sigma} = \frac{1}{n}\mathbf{X}^T\mathbf{X}$ and let $U\Lambda U^T$ be its eigen decomposition. Then it is easy to see that $C\left(2^{-1} b^2\mathbf{X}^T\mathbf{X}, b\mathbf{X}^T\mathbf{Y}\right)$ can be rewritten as
\[
C\left(2^{-1} b^2\mathbf{X}^T\mathbf{X}, b\mathbf{X}^T\mathbf{Y}\right) = C(2^{-1} nb^2\Lambda, b U^T\mathbf{X}^T\mathbf{Y}).
\]
So it suffices to compute $C(A, \gamma)$ for some positive semi-definite, diagonal matrix $A$ and vector $\gamma$. Let $a_{i}$ be the $i^{th}$ diagonal entry of $A$ and let $\gamma_i$ be the $i^{th}$ element of $\gamma$. 
\citet{kume2005saddlepoint} derive the following expression for $C(A,\gamma)$
\[
C(A,\gamma) = (2\pi)^{d/2} \left(\prod_{i=1}^d a_i^{-1/2}\right) \exp\left(\frac{1}{4}\sum_{i=1}^d \frac{\gamma_i^2}{a_i}\right) f_{A,\gamma}(1),
\]where $f_{A,\gamma}$ is the probability density of a non-central chi-squared random variable $\sum_{i=1}^d z_i^2$ with $z_i \sim \mathcal{N}(\frac{\gamma_i}{2a_i},\frac{1}{2a_i})$. There are number of efficient techniques for computation of $f_{A,\gamma}(1)$~\citep{imhof1961computing, kume2005saddlepoint}. 
We first present the technique of \citet{imhof1961computing} for exact computation of $f_{A,\gamma}(1)$.  
\citet{imhof1961computing} showed that $f_{A,\gamma}(1)$ can be written as the following integral
\[
f_{A,\gamma}(1) = \pi^{-1}\int_{0}^{\infty} [\rho(u)]^{-1}\cos{\zeta(u)} du,
\]
where $\rho:\mathbb{R}\to\mathbb{R}$ and $\zeta:\mathbb{R}\to\mathbb{R}$ are defined as
\begin{align*}
\zeta(u) &= \frac{1}{2}\sum_{i = 1}^d\left(\tan^{-1}\left(\frac{u}{2a_i}\right) + \frac{\gamma_i^2}{8a_i^3}\left(1 + \frac{u^2}{4a_i^2}\right)^{-1}u\right) - \frac{1}{2}u,\\
\rho(u) &= \prod_{i=1}^d\left(1 + \frac{u^2}{4a_i^2}\right)^{1/4}\exp\left(\frac{1}{32}\frac{(u\gamma_i/a_i^2)^2}{1+\frac{u^2}{4a_i^2}}\right).
\end{align*}
One can rely on numerical integration techniques to compute the above integral to desired accuracy. In our analysis of the approximation error of Algorithm~\ref{alg:regression_min_oracle}, we assume the error from the computation of $f_{A,\gamma}(1)$ is negligible. 

Before we conclude this subsection, we present another technique for computation of $f_{A,\gamma}(1)$, which is typically faster than the above approach. This approach was proposed by \citet{kume2005saddlepoint} and relies on the saddle point density approximation technique. While this approach is faster, the downside of it is that it only provides an approximate estimate of $f_{A,\gamma}(1)$. To explain this method, we first present some facts about non-central chi-squared random variables. The cumulant generating function of a non-central chi-squared random variable with density $f_{A,\gamma}$ is given by
\[
K(t) = \sum_{i=1}^d \left(-\frac{1}{2}\log\left(1-\frac{t}{a_i}\right) + \frac{1}{4}\frac{\gamma_i^2}{a_i-t} - \frac{\gamma_i^2}{4a_i}\right) \quad (t < \min_i a_i).
\]
The first derivative of $K(t)$ is given by
\[
K^{(1)}(t) = \sum_{i=1}^d \left(\frac{1}{2}\frac{1}{a_i-t} + \frac{1}{4}\frac{\gamma_i^2}{(a_i-t)^2}\right),
\]
and higher derivatives are given by
\[
K^{(j)}(t) = \sum_{i=1}^d \left(\frac{(j-1)!}{2}\frac{1}{(a_i-t)^j} + \frac{j!}{4}\frac{\gamma_i^2}{(a_i-t)^{j+1}}\right),\quad (j \geq 2).
\]
Let $\hat{t}$ be the unique solution in $(-\infty, \min_i a_i)$ to the saddle point equation $K^{(1)}(\hat{t}) = 1$. \citet{kume2005saddlepoint} show that $\hat{t}$ has finite upper and lower bounds
\[
\min_i a_i - \frac{d}{4}-\frac{1}{2}\left(\frac{d^2}{4} + d\max_i \gamma_i^2\right)^{1/2} \leq \hat{t} \leq \min_i a_i - \frac{1}{4} - \frac{1}{2}\left(\frac{1}{4} + \gamma_{\text{min}}^2\right)^{1/2},
\]
where $\gamma_{\text{min}}$ is equal to $\gamma_{i^*}$ for $i^* = \argmin_i a_i$. So, to find $\hat{t}$, one can perform grid search in the above range.
Given $\hat{t}$, the first-order saddle point density approximation of $f_{A,\gamma}(1)$ is given by
\[
\hat{f}_{A,\gamma,1}(1) = \left(2\pi K^{(2)}(\hat{t})\right)^{-1/2} \exp(K(\hat{t})-\hat{t}).
\]
The second-order saddle point density approximation of $Z_{g,h}(1)$ is given by
\[
\hat{f}_{A,\gamma,2}(1) = \hat{f}_{A,\gamma,1}(1)(1+T),
\]
where $T = \frac{1}{8}\hat{\rho}_4 - \frac{5}{24}\hat{\rho}_3^2$, where $\hat{\rho}_j = K^{(j)}(\hat{t})/(K^{(2)}(\hat{t}))^{j/2}$.
\subsection{Proof of Theorem~\ref{thm:regression}}
Before we present the proof of the Theorem we present useful intermediate
results which we require in our proof.
\subsubsection{Intermediate Results}
\begin{lemma}[Lipschitz Continuity]
\label{lem:risk_lipscthiz_regression}
Consider the problem of linear regression described in Section~\ref{sec:invariance_regression}. Let \mbox{$\Theta = \{\theta: \theta \in \mathbb{R}^d, \|\theta\|_2 \leq B\}$} and let $\hat{\theta}$ be any estimator which maps the data $D_n=\{(X_i,Y_i)\}_{i=1}^n$ to an element in $\Theta$. Then the risk \mbox{$R(\hat{\theta}, \theta) = \Eover{D_n}{\|\hat{\theta}(D_n)-\theta\|_2^2}$} is Lipschitz continuous in its second argument w.r.t $\ell_2$ norm over the domain $\Theta$, with Lipschitz constant $4(B+B^2\sqrt{nd})$. Moreover, \mbox{$R(\hat{\theta}, b\mathbf{e}_1) = \Eover{D_n}{\|\hat{\theta}(D_n)-b\mathbf{e}_1\|_2^2}$} is Lipschitz continuous in $b$ over the domain $[0,B]$, with Lipschitz constant $4(B+B^2\sqrt{n})$.
\end{lemma}
\begin{proof}
Let $R_{\hat{\theta}}(\theta) = R(\hat{\theta}, \theta)$. The gradient of $R_{\hat{\theta}}(\theta)$ with respect to $\theta$ is given by
\[
\nabla_{\theta} R_{\hat{\theta}}(\theta) = \Eover{D_n}{2(\theta-\hat{\theta}(D_n))} + \Eover{D_n}{\|\hat{\theta}(D_n)-\theta\|_2^2 \mathbf{X}^T(\mathbf{Y}-\mathbf{X}\theta)},
\]
where $\mathbf{X} = [X_1, X_2, \dots X_n]^T, \mathbf{Y} = [Y_1, \dots Y_n]$.
The norm of $\nabla_{\theta} R_{\hat{\theta}}(\theta)$ can be upper bounded as follows
\begin{align*}
    \|\nabla_{\theta} R_{\hat{\theta}}(\theta)\|_2 & \leq \Big|\Big|\Eover{D_n}{2(\theta-\hat{\theta}(D_n))}\Big|\Big|_2 + \Big|\Big|\Eover{D_n}{\|\hat{\theta}(D_n)-\theta\|_2^2 \mathbf{X}^T(\mathbf{Y}-\mathbf{X}\theta)}\Big|\Big|_2\\
    & \stackrel{(a)}{\leq} 4B + \Eover{D_n}{ \|\mathbf{X}^T(\mathbf{Y}-\mathbf{X}\theta)\|_2\|\hat{\theta}(D_n)-\theta\|_2^2}\\
    & \stackrel{(b)}{\leq} 4B + 4B^2\Eover{D_n}{ \|\mathbf{X}^T(\mathbf{Y}-\mathbf{X}\theta)\|_2}\\
    & \leq 4B + 4B^2\sqrt{nd},
\end{align*}
where the first term in $(a)$ follows from the fact that $\theta, \hat{\theta}(X) \in \Theta$ and  the second term follows from Jensen's inequality. This shows that $R_{\hat{\theta}}(\theta)$ is Lipschitz continuous over $\Theta$.  This finishes the first part of the proof. To show that $R(\hat{\theta}, b\mathbf{e}_1)$ is Lipschitz continuous in $b$, we use similar arguments. Let $R_{\hat{\theta}}(b) = R(\hat{\theta}, b\mathbf{e}_1)$. Then
\begin{align*}
    \Big| R_{\hat{\theta}}'(b)\Big| & = \Big|\left\langle\mathbf{e}_1, \nabla_{\theta} R_{\hat{\theta}}(\theta)\Big|_{\theta = b\mathbf{e}_1}\right\rangle\Big|\\
    &\stackrel{(a)}{\leq} \Big|\Eover{D_n}{2(b-[\hat{\theta}(D_n)]_1)}\Big| + \Big|\Big|\Eover{D_n}{ \mathbf{e}_1^T\mathbf{X}^T(\mathbf{Y}-\mathbf{X}\theta)\|\hat{\theta}(D_n)-b\mathbf{e}_1\|_2^2}\Big|\Big|_2\\
    & \leq 4B + 4B^2\Eover{D_n}{ |\mathbf{e}_1^T\mathbf{X}^T(\mathbf{Y}-\mathbf{X}\theta)|}\\
    & \leq 4B + 4B^2\sqrt{n},
\end{align*}
where $(a)$ follows from our bound for $\|\nabla_{\theta} R_{\hat{\theta}}(\theta)\|_2$ obtained above.
\end{proof}

\begin{lemma}[Approximation of risk]
\label{lem:risk_approximation_regression}
Consider the setting of Lemma~\ref{lem:risk_lipscthiz_regression}. Let $\hat{\theta}$ be any estimator which maps $D_n$ to an element in $\Theta$. Let $\{D_{n,k}\}_{k=1}^N$ be $N$ independent datasets generated from the linear regression model with true parameter $\theta$. Then with probability at least $1-\delta$
\[
\Big| \frac{1}{N}\sum_{i=1}^N\|\hat{\theta}(D_{n,i})-\theta\|_2^2 - R_{\hat{\theta}}(\theta) \Big| \leq 4B^2\sqrt{\frac{\log{\frac{1}{\delta}}}{N}}
\]
\end{lemma}
\begin{proof}
The proof of the Lemma relies on concentration properties of sub-Gaussian random variables. Let $Z(D_n) = \|\hat{\theta}(D_{n})-\theta\|^2$. Note that $R_{\hat{\theta}}(\theta) = \Eover{D_n}{Z(D_n)}$. Since $Z(D_n)$ is bounded by $4B^2$, it is a sub-Gaussian random variable. Using Hoeffding bound we get
\[
\Big|\frac{1}{N}\sum_{i=1}^N Z(D_{n,i})- \E{Z(D_n)}\Big| \leq 4B^2\sqrt{\frac{\log{\frac{1}{\delta}}}{N}}, \quad \text{w.p} \geq 1-\delta.
\]
\end{proof}
\subsubsection{Main Argument}
The proof uses exactly the same arguments as in the proof of Theorem~\ref{thm:mean_estimation}. The only difference between the two proofs are the Lipschitz constants derived in Lemmas~\ref{lem:risk_lipscthiz},~\ref{lem:risk_lipscthiz_regression}. The Lipschitz constant in the case of regression is $O(B+B^2\sqrt{n})$, whereas in the case of finite Gaussian sequence model it is $O(B+B^2)$.

\paragraph{Approximation Error of Algorithm~\ref{alg:regression_max_oracle}.} 
There are two causes for error in the optimization oracle described in Algorithm~\ref{alg:regression_max_oracle}: (a) grid search and (b)  approximate computation of risk $R(\hat{\theta}, b\mathbf{e}_1)$. We now bound the  error due to both (a) and (b). 
    From Lemma~\ref{lem:risk_approximation_regression} we know that for any estimator $\hat{\theta}_i$ and grid point $b_j$, the following holds with probability at least $1-\delta$
    \[
    \Big| \frac{1}{N_1}\sum_{k=1}^{N_1} \|\hat{\theta}_i(D_{n,k})-b_j\mathbf{e}_1\|_2^2 - R(\hat{\theta}_i, b_j\mathbf{e}_1)\Big| \leq 4B^2\sqrt{\frac{\log{\frac{1}{\delta}}}{N_1}}.
    \]
    Taking a union bound over all estimators $\{\hat{\theta}_i\}_{i=1}^T$  and grid points $\{b_j\}_{j=1}^{B/w}$, we can show that with probability at least $1-\delta$, the following holds for all $i\in [T],j\in [B/w]$ 
    \begin{equation}
    \label{eqn:aux_concentration_risk_regression}
    \Big| \frac{1}{N_1}\sum_{k=1}^{N_1} \|\hat{\theta}_i(D_{n,k})-b_j\mathbf{e}_1\|_2^2 - R(\hat{\theta}_i, b_j\mathbf{e}_1)\Big| \leq 4B^2\sqrt{\frac{\log{\frac{BT}{w\delta}}}{N_1}}.
    \end{equation}
Let $f_{t,\sigma}(b)$ be the actual objective we would like to optimize in iteration $t$ of Algorithm~\ref{alg:ftpl_stat_games}, which is given by
\[
f_{t,\sigma}(b) = \sum_{i = 1}^{t-1} R(\hat{\theta}_i, b\mathbf{e}_1) + \sigma b.
\]
Let $\hat{f}_{t,\sigma}(b)$ be the approximate objective we are optimizing by replacing $R(\hat{\theta}_i, b\mathbf{e}_1)$ with its approximate estimate. Let $b^*_t$ be a maximizer of $f_{t,\sigma}(b)$ and $b^*_{t,\text{approx}}$ be the maximizer of $\hat{f}_{t,\sigma}(b)$ (which is also the output of Algorithm~\ref{alg:regression_max_oracle}). Finally, let $b^*_{t,\text{NN}}$ be the point on the grid which is closest to $b^*_t$. Using Lemma~\ref{lem:risk_lipscthiz_regression} we first show that  $f_{t,\sigma}(b)$ is Lipschitz continuous in $b$. The derivative of $f_{t,\sigma}(b)$ with respect to $b$ is given by
\begin{align*}
    f_{t,\sigma}'(b) & = \sum_{i = 1}^{t-1} \left\langle \mathbf{e}_1,\nabla_{\theta} R(\hat{\theta}_i, \theta)\Big|_{\theta = b\mathbf{e}_1}\right\rangle + \sigma 
\end{align*}
Using Lemma~\ref{lem:risk_lipscthiz_regression}, the magnitude of $f_{t,\sigma}'(b)$ can be upper bounded as
\[
|f_{t,\sigma}'(b)| \leq 4(t-1)(B+B^2\sqrt{n}) + \sigma.
\]
This shows that $f_{t,\sigma}(b)$ is Lipschitz continuous in $b$. We now bound $f_{t,\sigma}(b^*_t) - f_{t,\sigma}(b^*_{t,\text{approx}})$, the approximation error of the optimization oracle
\begin{align*}
    f_{t,\sigma}(b^*_t) & \stackrel{(a)}{\leq } f_{t,\sigma}(b^*_{t,\text{NN}}) + \left(4t(B+B^2\sqrt{n}) + \sigma\right) w\\
    & \stackrel{(b)}{\leq} \hat{f}_{t,\sigma}(b^*_{t,\text{NN}}) +  4tB^2\sqrt{\frac{\log{\frac{BT}{w\delta}}}{N_1}} + \left(4t(B+B^2\sqrt{n}) + \sigma\right) w\\
    & \stackrel{(c)}{\leq} \hat{f}_{t,\sigma}(b^*_{t,\text{approx}}) +  4tB^2\sqrt{\frac{\log{\frac{BT}{w\delta}}}{N_1}} + \left(4t(B+B^2\sqrt{n}) + \sigma\right) w \\
    & \stackrel{(d)}{\leq} f_{t,\sigma}(b^*_{t,\text{approx}}) +  8tB^2\sqrt{\frac{\log{\frac{BT}{w\delta}}}{N_1}} + \left(4t(B+B^2\sqrt{n}) + \sigma\right) w,
\end{align*}
where $(a)$ follows from Lipschitz property of the loss function and $(b), (d)$ follow from Equation~\eqref{eqn:aux_concentration_risk_regression} and hold with probability at least $1-\delta$ and $(c)$ follows from the optimality of $b^*_{t,\text{approx}}$.
This shows that Algorithm~\ref{alg:regression_max_oracle} is a $\left(O\left(TB^2\sqrt{\frac{\log{\frac{BT}{w\delta}}}{N_1}} + TB(1+B\sqrt{n})w \right),w\right)$-approximate maximization oracle; that is
\[
\alpha = O\left(TB^2\sqrt{\frac{\log{\frac{BT}{w\delta}}}{N_1}} + TB(1+B\sqrt{n})w\right),\quad \beta = w.
\]

\paragraph{Approximation Error of Algorithm~\ref{alg:regression_min_oracle}.} There are two sources of approximation error in Algorithm~\ref{alg:regression_min_oracle}: (a) computation of mean and normalization constant of FB distribution, and (b) approximation of $P_t$ with its samples. In this analysis we assume that mean and normalization constant of FB distribution can be computed to very high accuracy. So the main focus here is on bounding the error from approximation of $P_t$. 

First, note that since we are using grid search to optimize the maximization problem, the true distribution $P_t$ for which we are supposed to compute the Bayes estimator is a discrete distribution supported on grid points $\{b_1,\dots b_{B/w}\}$. Algorithm~\ref{alg:regression_min_oracle} does not compute the Bayes estimator for $P_t$. Instead, we generate samples from $P_t$ and use them as a proxy for $P_t$. Let $\hat{P}_t$ be the empirical distribution obtained by sampling $N_2$ points from $P_t$. Let $p_{t,j}$ be the probability mass on grid point $b_j$. Using Bernstein inequality we can show that the following holds with probability at least $1-\delta$ 
\begin{equation}
\label{eqn:aux_concentration_prior_regression}
    \forall j \in [B/w]\quad |\hat{p}_{t,j}-p_{t,j}| \leq \sqrt{p_{t,j}\frac{\log{\frac{B}{w\delta}}}{N_2}}.
\end{equation}
    Define estimators $\hat{\theta}_t', \hat{\theta}_{t}$ as 
    \[
    \hat{\theta}_t' \leftarrow \argmin_{\hat{\theta} \in \mathcal{D}_G} \Eover{b \sim P_t}{R(\hat{\theta}, b\mathbf{e}_1)}, \quad \hat{\theta}_{t} \leftarrow \argmin_{\hat{\theta} \in \mathcal{D}_G} \Eover{b \sim \hat{P}_t}{R(\hat{\theta}, b\mathbf{e}_1)}.
    \]
    $\hat{\theta}_t'$ is what we ideally want to compute. $\hat{\theta}_{t}$ is what we end up computing using Algorithm~\ref{alg:regression_min_oracle}.
    We now show that $\hat{\theta}_{t}$ is an approximate minimizer of the left hand side optimization problem above. To this end, we try to bound the following quantity
    \begin{equation*}
        \Eover{b \sim P_t}{R(\hat{\theta}_{t}, b\mathbf{e}_1) - R(\hat{\theta}_{t}', b\mathbf{e}_1)}.
    \end{equation*}
    Let $f_{t}(\hat{\theta}) = \Eover{b \sim P_t}{R(\hat{\theta}, b\mathbf{e}_1)}$ and $\hat{f}_{t}(\hat{\theta}) = \Eover{b \sim \hat{P}_t}{R(\hat{\theta}, b\mathbf{e}_1)}$. We would like to bound the quantity $f_t(\hat{\theta}_{t}) - f_t(\hat{\theta}_{t}')$. Consider the following 
    \begin{align*}
        f_t(\hat{\theta}_{t}) & \stackrel{(a)}{\leq} \hat{f}_t(\hat{\theta}_{t}) + \frac{4B^3}{w}\sqrt{\frac{\log{\frac{B}{w\delta}}}{N_2}}\\
        & \stackrel{(b)}{\leq} \hat{f}_t(\hat{\theta}_{t}') + \frac{4B^3}{w}\sqrt{\frac{\log{\frac{B}{w\delta}}}{N_2}}\\
        & \stackrel{(c)}{\leq }f_t(\hat{\theta}_{t}') + \frac{8B^3}{w}\sqrt{\frac{\log{\frac{B}{w\delta}}}{N_2}},
    \end{align*}
    where $(a)$ follows from Equation~\eqref{eqn:aux_concentration_prior_regression} and the fact that the risk $R(\hat{\theta},\theta)$ of any estimator is bounded by $4B^2$, $(b)$ follows since $\hat{\theta}_{t}$ is a minimizer of $\hat{f}_t$ and $(c)$ follows from Equation~\eqref{eqn:aux_concentration_prior_regression}. This shows that with probability at least $1-\delta$, Algorithm~\ref{alg:regression_min_oracle} is an $O\left(\frac{B^3}{w}\sqrt{\frac{\log{\frac{B}{w\delta}}}{N_2}}\right)$-approximate optimization oracle; that is,
    \[
    \alpha' = O\left(\frac{B^3}{w}\sqrt{\frac{\log{\frac{B}{w\delta}}}{N_2}}\right).
    \]
    
    The rest of the proof is same as the proof of Theorem~\ref{thm:mean_estimation} and involves substituting the approximation errors computed above in Corollary~\ref{cor:ftpl_minimax_estimators}.
     \paragraph{Minimax Estimator.} We now show that $\hat{\theta}_{\AVG}$ is an approximate minimax estimator. Instantiating Corollary~\ref{cor:ftpl_minimax_estimators} for the reduced statistical game gives us the following bound, which holds with probability at least $1-\delta$
    \[
    \sup_{b \in [0,B]} R(\hat{\theta}_{\AVG},b\mathbf{e}_1) \leq \inf_{\hat{\theta} \in \mathcal{D}_G}\sup_{b \in [0,B]}R(\hat{\theta},b\mathbf{e}_1)  + \Tilde{O}\left(\frac{B^2(B\sqrt{n}+1)}{\sqrt{T}} + \alpha  +\alpha' + \beta B(B\sqrt{n}+1) \sqrt{T}\right),
    \]
    where we used the fact that the risk $R(\hat{\theta}, b\mathbf{e}_1)$ is $4B(B\sqrt{n}+1)$-Lipschitz continuous w.r.t $b$. The $\Tilde{O}$ notation in the above inequality hides logarithmic factors. Plugging in the values of $\alpha, \alpha',\beta$ in the above equation gives us 
    \[
    \sup_{b \in [0,B]} R(\hat{\theta}_{\AVG},b\mathbf{e}_1) \leq \inf_{\hat{\theta} \in \mathcal{D}_G}\sup_{b \in [0,B]}R(\hat{\theta},b\mathbf{e}_1)  + \Tilde{O}\left(\frac{B^2(B\sqrt{n}+1)}{\sqrt{T}}\right).
    \]
    We now convert this bound to a bound on the original statistical game. From Theorem~\ref{thm:ftpl_minimax_invariance} we know that $\inf_{\hat{\theta} \in \mathcal{D}_G}\sup_{b \in [0,B]}R(\hat{\theta},b\mathbf{e}_1) = \inf_{\hat{\theta} \in \mathcal{D}}\sup_{\theta \in \Theta}R(\hat{\theta},\theta) = R^*$. Since the estimator $\hat{\theta}_{\AVG}$ is invariant to transformations of orthogonal group, we have $R(\hat{\theta}_{\AVG},\theta) = R(\hat{\theta}_{\AVG},\|\theta\|_2\mathbf{e}_1)$ for any $\theta \in \Theta$. Using these two results in the above inequality, we get 
    \[
    \sup_{\theta \in \Theta} R(\hat{\theta}_{\AVG},\theta) = \sup_{b \in [0,B]} R(\hat{\theta}_{\AVG},b\mathbf{e}_1) \leq R^*  + \Tilde{O}\left(\frac{B^2(B\sqrt{n}+1)}{\sqrt{T}}\right).
    \]
    This shows that the worst-case risk of $\hat{\theta}_{\AVG}$ is close to the minimax risk $R^*$. This finishes the first part of the proof.
    \paragraph{LFP.} To prove the second part, we rely on Corollary~\ref{cor:ftpl_minimax_estimators}. Instantiating it for the reduced statistical game gives us
    \[
    \inf_{\hat{\theta} \in \mathcal{D}_G} \frac{1}{T}\sum_{t=1}^T \Eover{b\sim P_t}{R(\hat{\theta},b\mathbf{e}_1 )} \geq R^* -   \Tilde{O}\left(\frac{B^2(B\sqrt{n}+1)}{\sqrt{T}} + \alpha  +\alpha' + \beta B(B\sqrt{n}+1) \sqrt{T}\right).
    \]
    Plugging in the values of $\alpha, \alpha',\beta$ in the above equation gives us
    \[
    \inf_{\hat{\theta} \in \mathcal{D}_G} \frac{1}{T}\sum_{t=1}^T \Eover{b\sim P_t}{R(\hat{\theta},b\mathbf{e}_1 )} \geq R^* - \Tilde{O}\left(\frac{B^2(B\sqrt{n}+1)}{\sqrt{T}}\right).
    \]
    From Equation~\eqref{eqn:aux_concentration_prior} we know that $P_t$ is close to $\hat{P}_t$ with high probability. Using this, we can replace $P_t$ in the above bound with $\hat{P}_t$ and obtain the following bound, which holds with probability at least $1-\delta$
    \begin{align}
    \label{eqn:lfp_regression}
        \inf_{\hat{\theta} \in \mathcal{D}_G} \frac{1}{T}\sum_{t=1}^T \Eover{b\sim \hat{P}_t}{R(\hat{\theta},b\mathbf{e}_1 )} \geq R^* - \Tilde{O}\left(\frac{B^2(B\sqrt{n}+1)}{\sqrt{T}}\right).
    \end{align}
    In the rest of the proof, we show that $\inf_{\hat{\theta} \in \mathcal{D}_G} \frac{1}{T}\sum_{t=1}^T \Eover{b\sim \hat{P}_t}{R(\hat{\theta},b\mathbf{e}_1 )} = \inf_{\hat{\theta}} R(\hat{\theta}, \hat{P}_{\text{LFP}})$. From the definition of $\hat{P}_{\text{LFP}}$, we can equivalently rewrite $R(\hat{\theta}, \hat{P}_{\text{LFP}})$ for any estimator $\hat{\theta}$ as
    \[
    R(\hat{\theta}, \hat{P}_{\text{LFP}}) = \frac{1}{T}\sum_{t=1}^T\Eover{b\sim \hat{P}_t}{\Eover{\theta \sim U}{R(\hat{\theta}, b\theta)}},
    \]
    where $U$ is the uniform distribution on the $(d-1)$ dimensional unit sphere centered at origin, in $\mathbb{R}^{d}$. Next, from Lemma~\ref{lem:invariant_bayes_estimator}, we know that the Bayes estimator corresponding to any invariant prior is an invariant estimator. Since  $\hat{P}_{\text{LFP}}$ is an invariant distribution, we have
    \[
    \inf_{\hat{\theta}\in \mathcal{D}} R(\hat{\theta}, \hat{P}_{\text{LFP}}) = \inf_{\hat{\theta}\in \mathcal{D}_G} R(\hat{\theta}, \hat{P}_{\text{LFP}}) = \inf_{\hat{\theta}\in \mathcal{D}_G}\frac{1}{T}\sum_{t=1}^T\Eover{b\sim \hat{P}_t}{\Eover{\theta \sim U}{R(\hat{\theta}, b\theta)}}.
    \]
    From Lemma~\ref{lem:invariant_risk} we know that for any invariant estimator $\hat{\theta},$ we have $R(\hat{\theta},\theta_1) = R(\hat{\theta}, \theta_2)$, whenever $\theta_1\sim\theta_2$. Using this result in the above equation gives us
    \[
    \inf_{\hat{\theta}\in \mathcal{D}} R(\hat{\theta}, \hat{P}_{\text{LFP}}) = \inf_{\hat{\theta}\in \mathcal{D}_G}\frac{1}{T}\sum_{t=1}^T\Eover{b\sim \hat{P}_t}{R(\hat{\theta}, b\mathbf{e}_1)}.
    \]
    Combining the above result with Equation~\eqref{eqn:lfp_regression} shows that $\hat{P}_{\text{LFP}}$ is approximately least favorable.
\section{Covariance Estimation}
\subsection{Proof of Proposition~\ref{prop:covariance_minimax_estimator}}
In this proof, we rely on permutation invariant functions and a representer theorem for such functions. A function $f:\mathbb{R}^d\to \mathbb{R}$ is called permutation invariant, if for any permutation $\pi$ and any $X\in\mathbb{R}^d$
\[f(\pi(X)) = f(X).\] 
The following proposition provides a representer theorem for such functions.
\begin{proposition}[\citet{zaheer2017deep}]
\label{prop:representer_thm_perm_invariant}
A function $f(X)$ from $\mathbb{R}^d$ to $\mathbb{R}$ is permutation invariant and continuous iff it can be decomposed in the form $\rho(\sum_{i=1}^d \phi(X_i))$, for some suitable transformations $\phi:\mathbb{R}\to\mathbb{R}^{d+1}$ and $\rho:\mathbb{R}^{d+1}\to\mathbb{R}$.
\end{proposition}

We now prove Proposition~\ref{prop:covariance_minimax_estimator}. First note that from Blackwell's theorem we know that there exists a minimax estimator which is just a function of the sufficient statistic, which in this case is the empirical covariance $S_n = \frac{1}{n}\sum_{i=1}^nX_iX_i^T$~\citep[see Theorem 2.1 of][]{Ibragimov81book}. So we restrict ourselves to estimators which are functions of $S_n$. This, together with Theorem \ref{thm:invariance-est-lfp}, shows that there is a minimax estimator which is a function $S_n$ and which is invariant under the action of the orthogonal group $\mathbb{O}(d)$. Let $\hat\Sigma$ be such an estimator. Since $\hat\Sigma$ is an invariant estimator, it satisfies the following equality for any orthogonal matrix $V$
\[\hat{\Sigma}(VS_nV^T)=V\hat{\Sigma}(S_n)V^T.\] Setting $V = U^T$ in the above equation, we get $\hat{\Sigma}(S_n)=U\hat{\Sigma}(\Delta)U^T.$
Hence, $\hat\Sigma$ is completely determined by it's action on diagonal matrices. So, in the rest of the proof we try to understand $\hat{\Sigma}(\Delta)$. Again relying on invariance of $\hat\Sigma$ and setting $V = \Delta'U^T$ for some diagonal matrix $\Delta'$ with diagonal elements $\pm 1$, we get
\[\hat{\Sigma}(\Delta'\Delta\Delta')= \Delta'U^T\hat{\Sigma}(S_n)U\Delta'\stackrel{(a)}{=} \Delta'\hat{\Sigma}(\Delta)\Delta',\]
where $(a)$ follows from the fact that $\hat{\Sigma}(S_n)=U\hat{\Sigma}(\Delta)U^T.$
Since $\Delta'\Delta\Delta' = \Delta$, the above equation shows that $\Delta'\hat{\Sigma}(\Delta)\Delta'=\hat{\Sigma}(\Delta)$ for any diagonal matrix $\Delta'$ with diagonal elements $\pm 1$. This shows that $\hat{\Sigma}(\Delta)$ is a diagonal matrix. 
Next, we set $V = P_{\pi}U^T$, where $P_{\pi}$ is the permutation matrix corresponding to some permutation $\pi$. This gives us
\[\hat{\Sigma}(P_{\pi}\Delta P_{\pi}^T) = P_{\pi}\hat{\Sigma}(\Delta)P_{\pi}^T.\]
This shows that for any permutation $\pi$, $\hat{\Sigma}(\pi(\Delta)) = \pi(\hat{\Sigma}(\Delta)),$ where $\pi(\Delta)$ represents permutation of the diagonal elements of $\Delta$.
In the rest of the proof, we use the notation $\Delta_i$ to denote the $i^{th}$ diagonal entry of  $\Delta$ and $\hat{\Sigma}_i(\Delta)$ to denote the $i^{th}$ diagonal entry of $\hat{\Sigma}(\Delta)$. The above property of $\hat\Sigma$ shows that $\hat{\Sigma}_i(\Delta)$ doesn't depend on the ordering of the elements in $\{\Delta_{j}\}_{j\neq i}$. This follows by choosing any permutation $\pi$ which keeps the $i^{th}$ element fixed. Next, by considering the permutation which only exchanges positions $1$ and $i$, we get
\[
\hat\Sigma_i(\Delta_1,\dots \Delta_i, \dots \Delta_d) = \hat\Sigma_1(\Delta_i,\dots \Delta_1, \dots \Delta_d).
\]
Thus $\hat\Sigma_i$ can be expressed in terms of $\hat\Sigma_1$. Represent $\hat\Sigma_1$ by $\hat\Sigma_0$. Combining the above two properties, we have
\begin{align*}
    \hat{\Sigma}_i(\Delta) = \hat\Sigma_0(\Delta_i, \{\Delta_{j}\}_{j\neq i}),
\end{align*}
where $\{\Delta_{j}\}_{j\neq i}$ represents the independence of $\hat\Sigma_0$ on the ordering of elements $\{\Delta_{j}\}_{j\neq i}$.
Now, consider the function $\hat\Sigma_0(\Delta_1,\{\Delta_j\}_{j=2}^d)$. 
For any fixed $a$, and $\Delta_1=a$, $\hat{\Sigma}_0(a, \{\Delta_{j}\}_{j=2}^d)$ is a permutation invariant function. 
Using Proposition~\ref{prop:representer_thm_perm_invariant}, $\hat{\Sigma}_0(a, \{\Delta_{j}\}_{j=2}^d)$ can be written as
\begin{align*}
    \hat{\Sigma}_0(a, \{\Delta_{j}\}_{j=2}^d) =  f_{a}\left(\sum_{j=2}^d g_{a}(\Delta_{j})\right),
\end{align*}
for some functions $f_a, g_a$. We overload the notation and define $f_a(x) = f(a,x)$ and $g_a(x) = g(a,x)$. Using this, we can represent $\hat{\Sigma}_i(\Delta)$  as
\[
\hat{\Sigma}_i(\Delta) = f\left(\Delta_i, \sum_{j\neq i} g(\Delta_i,\Delta_{j})\right),
\] 
for some functions $f,g$. 
There is a small technicality which we ignored while using Proposition~\ref{prop:representer_thm_perm_invariant} on $\hat{\Sigma}_0$. Proposition~\ref{prop:representer_thm_perm_invariant} only holds for continuous functions. Since $\hat{\Sigma}_0$ is not guaranteed to be continuous, the proposition can't be used on this function.  However, this is not an issue because any measurable function is a limit of continuous functions. Since $\hat{\Sigma}_0$ is a measurable function, it can be approximated arbitrarily close in the form of $f_{a}\left(\sum_{j=2}^d g_{a}(\Delta_{j})\right)$. 

To conclude the proof of the proposition, we note that 
\[
\inf_{\hat{\Sigma} \in \mathcal{M}_{\mathcal{D},G}} \sup_{\lambda\in\Xi_G}  R(\hat{\Sigma}, \text{Diag}(\lambda))= \inf\limits_{\hat{\Sigma} \in \mathcal{M}_{f,g}}\sup\limits_{\lambda \in \Xi_G}R(\hat{\Sigma},\text{Diag}(\lambda)).
\]
This is because the minimax estimator can be approximated arbitrarily well using estimators of the form $\hat{\Sigma}_i(\Delta) = f\left(\Delta_i, \sum_{j\neq i} g(\Delta_i,\Delta_{j})\right)$ and the fact that the model class has absolutely continuous distributions.

\section{Entropy Estimation}
\subsection{Proof of Proposition~\ref{prop:entropy_minimax_estimator}}
\label{sec:entropy-proof}
First note that any estimator of entropy is a function of $\hat{P}_n$, which is a sufficient statistic for the problem. This, together with Theorem~\ref{thm:invariance-est-lfp}, shows that there is a minimax estimator which is a function of $\hat{P}_n$ and which is invariant under the action of permutation group. Let $\hat{f}:\mathbb{R}^d\to\mathbb{R}$ be such an estimator. Since $\hat{f}$ is invariant, it satisfies the following property for any permutation $\pi$
\[
\hat{f}(\pi(\hat{P}_n)) = \hat{f}(\hat{P}_n).
\]
If $\hat{f}(\hat{P}_n)$ is continuous, then Proposition~\ref{prop:representer_thm_perm_invariant} shows that it can written as $g\left(\sum_{j=1}^d h(\hat{p}_j)\right)$, for some functions $h:\mathbb{R}\to\mathbb{R}^{d+1}, g:\mathbb{R}^{d+1}\to\mathbb{R}$. Even if it is not continuous, since it is a measurable function, it is a limit of continuous functions. So $\hat{f}$ can be approximated arbitrarily close in the form of $g\left(\sum_{j=1}^d h(\hat{p}_j)\right)$. This also implies the statistical game in Equation~\eqref{eqn:entropy_simplified} can reduced to the following problem
\[
\inf_{\hat{f} \in \mathcal{M}_{\mathcal{D},G}} \sup_{P\in\mathcal{P}_G}  R(\hat{f}, P) = \inf_{\hat{f} \in \mathcal{M}_{g,h}} \sup_{P\in\mathcal{P}_G}  R(\hat{f}, P).
\]

\section{Experiments}
\label{sec:exps_appendix}
\subsection{Covariance Estimation}
In this section, we compare the performance of various estimators at randomly generated $\Sigma$'s. 
 We use beta distribution to randomly generate $\Sigma$'s with varying spectral decays and compute the average risks of all the estimators at these $\Sigma$'s. Figure~\ref{fig:cov_estimation_beta} presents the results from this experiment. It can be seen that our estimator has better average case performance than empirical and James Stein estimators.
\begin{figure}[tbh]
\centering
\includegraphics[width=0.45\textwidth]{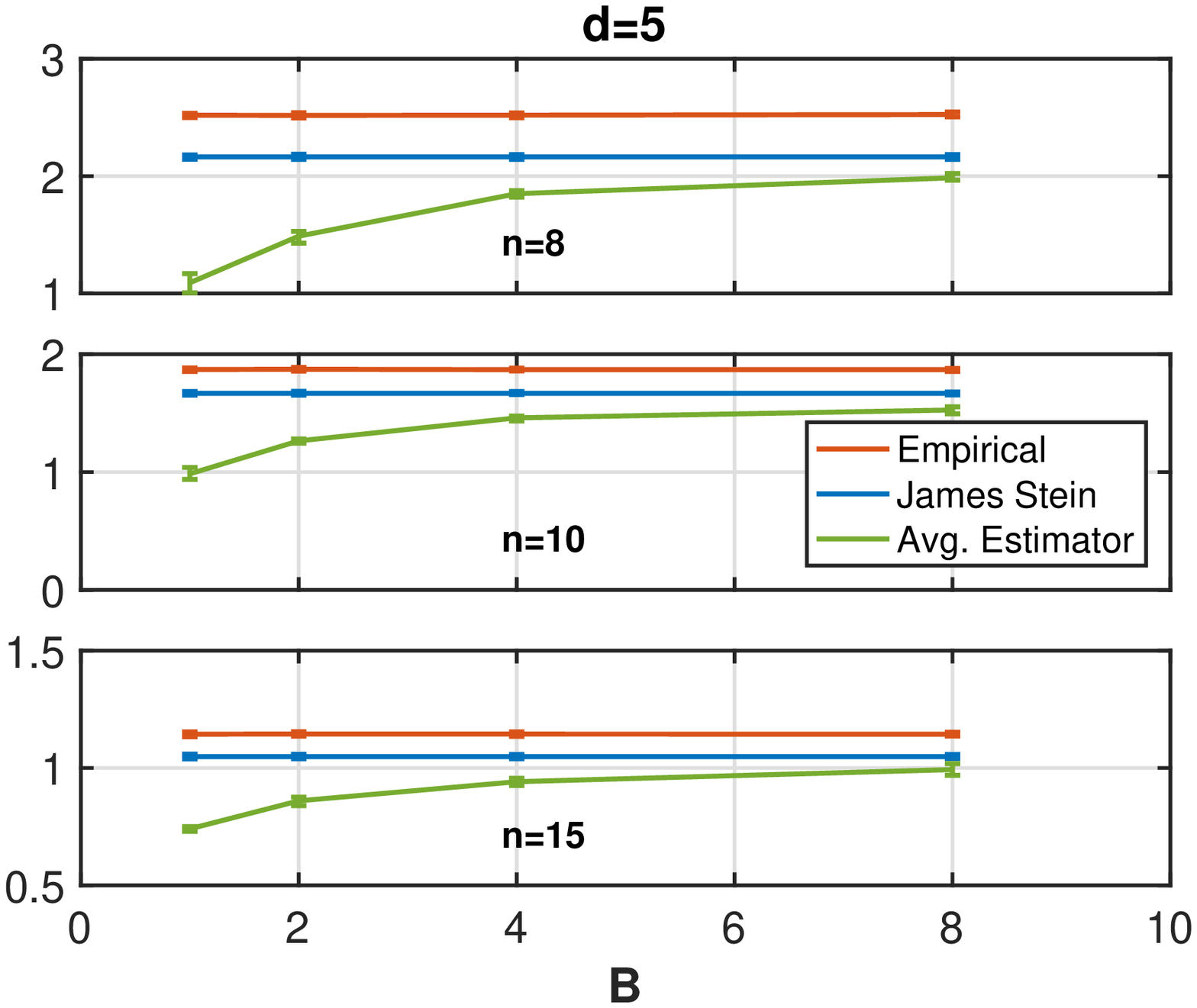}
\includegraphics[width=0.45\textwidth]{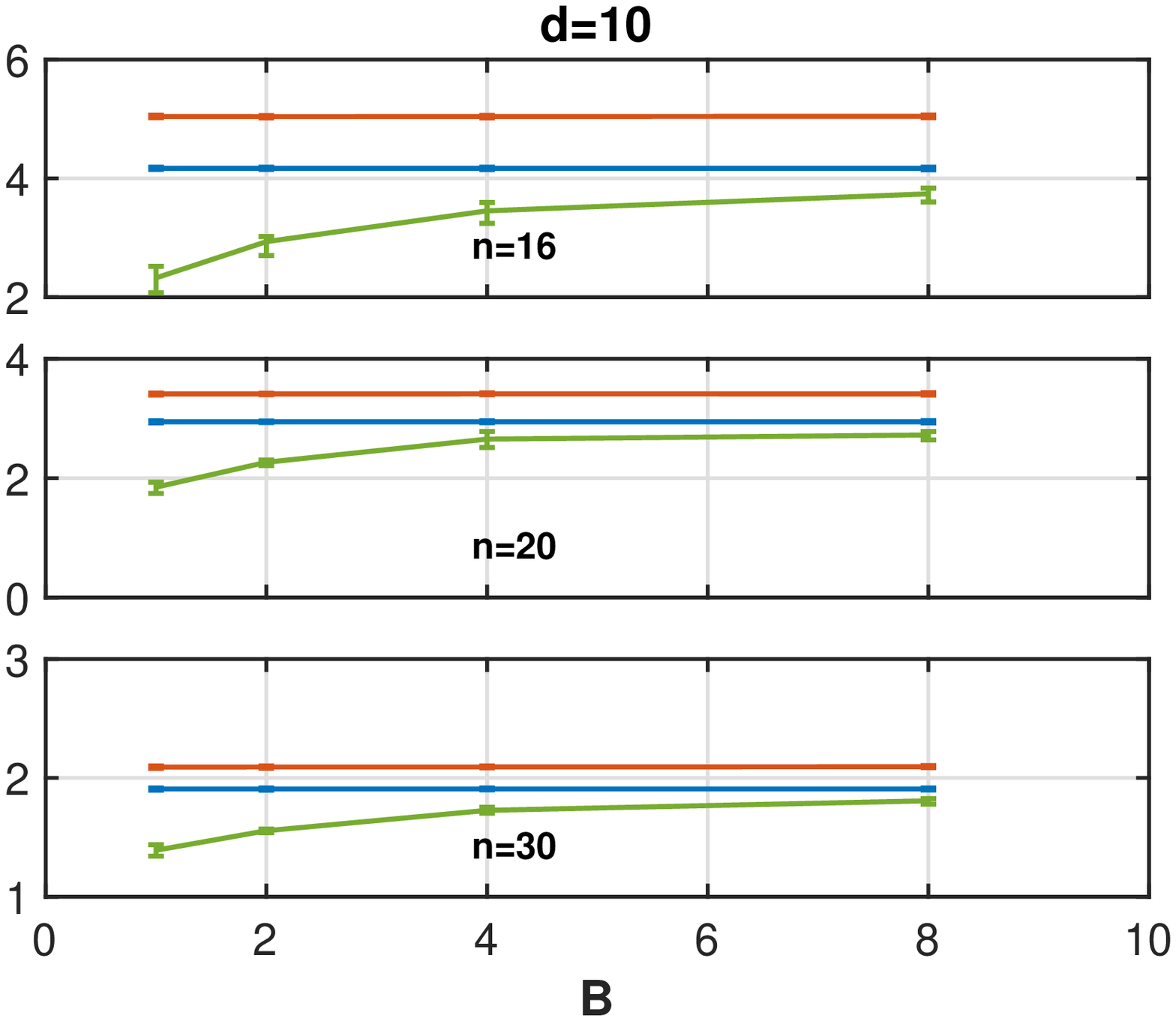}
\caption{\small{Risk of various estimators for covariance estimation evaluated at randomly generated $\Sigma$'s. We generated multiple $\Sigma$'s whose eigenvalues are randomly sampled from a Beta distribution with various parameters and averaged the risks of estimators at these $\Sigma$'s. Plots on the left correspond to $d=5$ and the plots on the right correspond to $d=10$.}}
\label{fig:cov_estimation_beta}
\end{figure}

\subsection{Entropy Estimation}
In this section, we compare the performance of various estimators at randomly generated $P$'s. We use beta distribution to randomly generate $P$'s and compute the average risks of all the estimators at these $P$'s. Figure~\ref{fig:entropy_estimation_beta}  presents the results from this experiment. 
\begin{figure}[ht]
\centering
\includegraphics[trim={3cm 8cm 3cm 8cm},clip,width=0.45\textwidth]{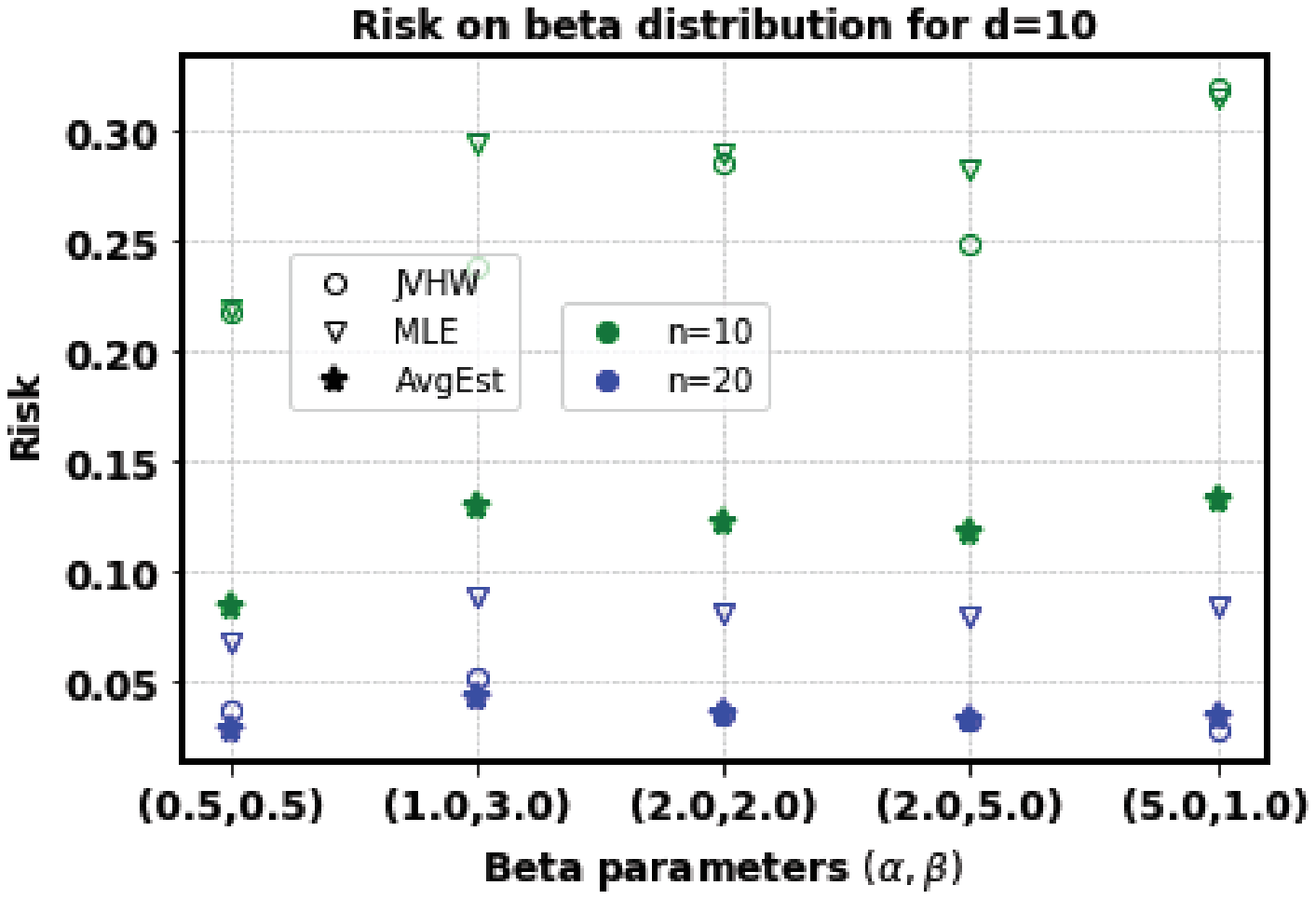}
\includegraphics[trim={3cm 8cm 3cm 8cm},clip,width=0.45\textwidth]{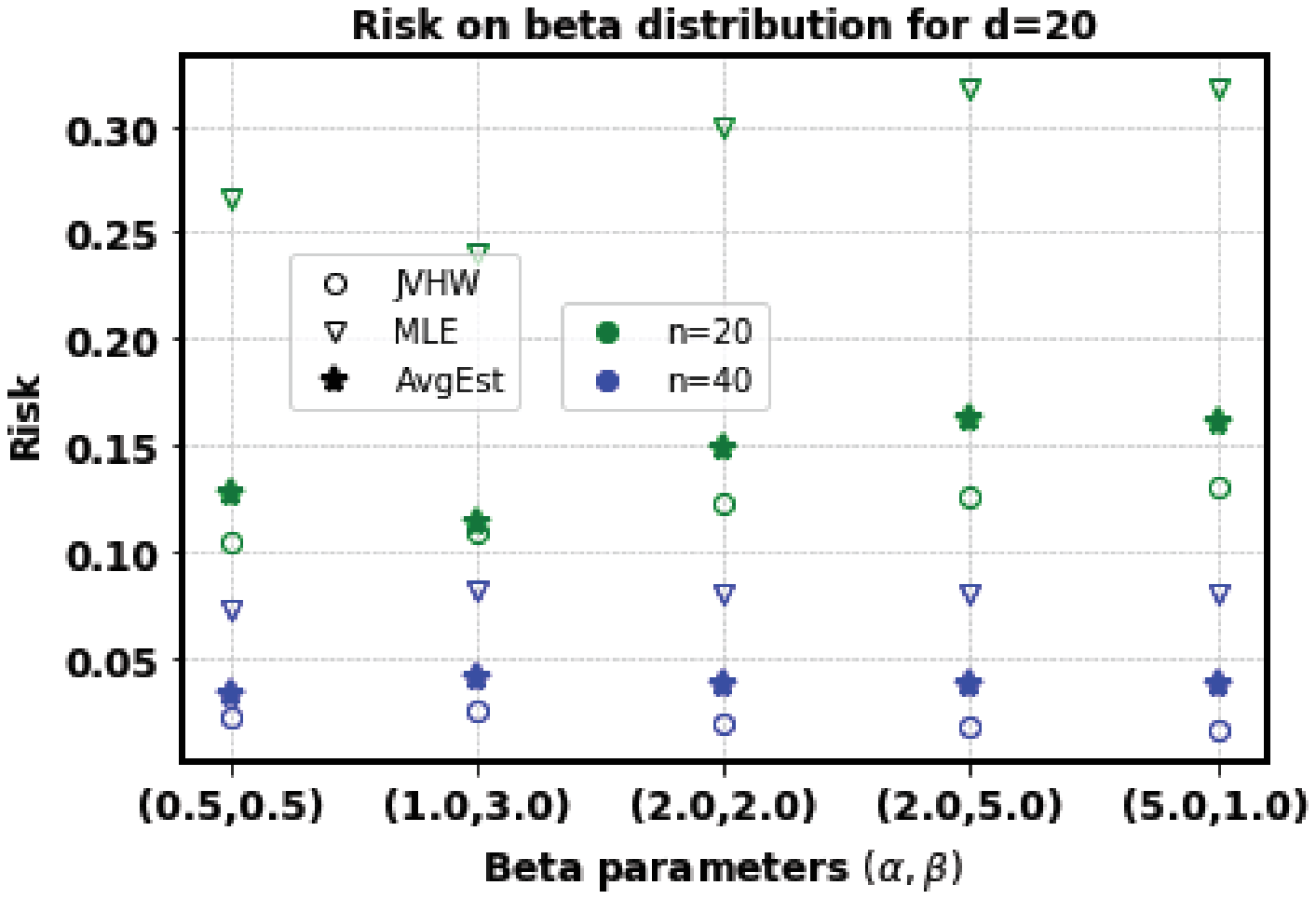}
\includegraphics[trim={3cm 8cm 3cm 8cm},clip,width=0.45\textwidth]{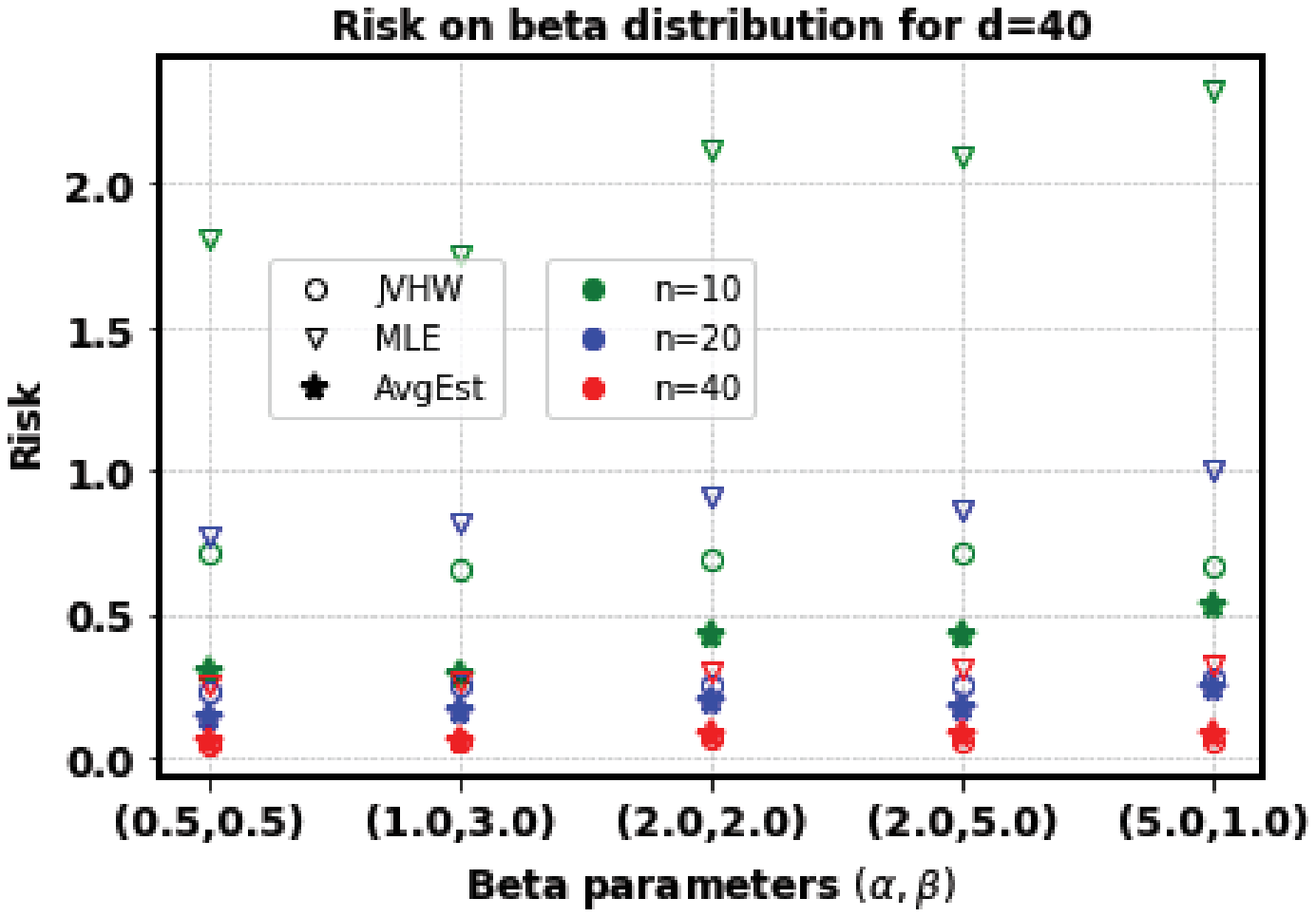}
\includegraphics[trim={3cm 8cm 3cm 8cm},clip,width=0.45\textwidth]{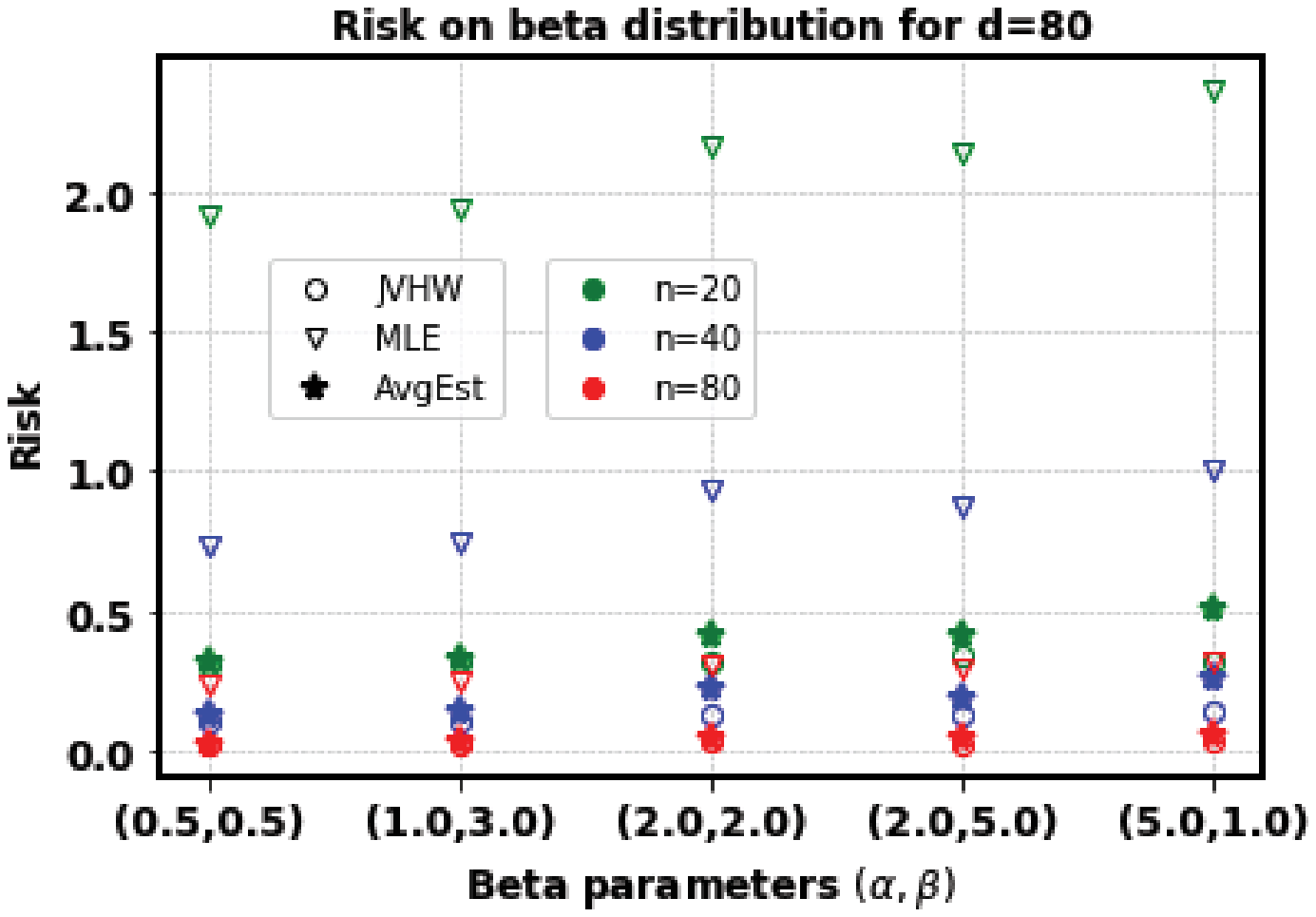}
\caption{\small{Risk of various estimators for entropy estimation evaluated at randomly generated distributions. We generated multiple $P$'s with $p_i$'s sampled from a Beta distribution and averaged the risks of estimators at these $P$'s. }}
\label{fig:entropy_estimation_beta}
\end{figure}

 \end{document}